\documentclass[twoside,11pt]{article}

\usepackage{jmlr2e}
\usepackage[T1]{fontenc}
\usepackage{verbatim}
\usepackage{amsmath}
\usepackage{kantlipsum}
\allowdisplaybreaks
\usepackage{amssymb}
\usepackage{stackrel}
\usepackage{verbatim}
\usepackage{url}
\usepackage{hyperref}
\usepackage{hyperref}
\hypersetup{
    colorlinks=true,
    linkcolor=black,
    filecolor=magenta,      
    urlcolor=cyan,
    citecolor=orange	
}
\usepackage{graphicx}
\usepackage{enumerate}%\usepackage{enumitem}
\newcommand\numberthis{\addtocounter{equation}{1}\tag{\theequation}}
\newcommand{\lp}{\left(}
\newcommand{\rp}{\right)}

\newcommand{\argmin}{\mathop{\mathrm{argmin}}}
\newcommand{\argmax}{\mathop{\mathrm{argmax}}}

\newcommand{\mbf}{\mathbf}
\newcommand{\mc}{\mathcal}

\newcommand{\mbb}{\mathbb}
\newcommand{\bds}{\boldsymbol}
\usepackage[dvipsnames]{xcolor}
\usepackage[normalem]{ulem}

\newcommand{\ads}[1]{\textcolor{magenta}{\textbf{ADS:} #1}}
\newcommand{\kn}[1]{\textcolor{blue}{\textbf{KN:} #1}}

\newcommand\redout{\bgroup\markoverwith{\textcolor{red}{\rule[.5ex]{2pt}{0.4pt}}}\ULon}

\renewcommand{\P}{\mbb{P}}
\newcommand{\E}{\mbb{E}}
\newcommand{\Var}{\mathop{\mathrm{Var}}\nolimits}

\DeclareMathOperator\arctanh{arctanh}

\newcommand{\ind}[1]{\mbf{1}_{#1}}

\newcommand{\Rad}{\mathrm{Rademacher}}
\newcommand\defeq{\stackrel{\triangle}{=}}
\newcommand{\SKL}{\bds{S}_{\text{KL}}}
\newcommand{\KL}{\bds{D}_{\text{KL}}}

\newcommand{\tpath}{\mathrm{path}}

\newcommand{\ssTV}{\mathrm{ssTV}}

\newcommand{\bx}{\mbf{x}}
\newcommand{\bX}{\mbf{X}}
\newcommand{\bnX}{\mbf{Y}} %noisy X
\newcommand{\bnx}{\mbf{y}}
\newcommand{\nX}{Y} %noisy X
\newcommand{\nx}{y}
\newcommand{\nmu}{\mu^{\dagger}} % noisy \mu
\newcommand{\nmue}{\hat{\mu}^{\dagger}} %estimate of noisy \mu
\newcommand{\np}{\p_{\dagger}}
\newcommand{\T}{\mathrm{T}}%original tree
\newcommand{\F}{\mathrm{F}}
\newcommand{\p}{\mathrm{p}}% pmf
\newcommand{\TCL}{\T^\mathrm{CL}}%CL tree
\newcommand{\TCLn}{\T^\text{CL}_{\dagger}}%CL tree noisy setting
\newcommand{\N}{N} %noise RV
\newcommand{\n}{n}%number of samples
\newcommand{\nn}{n_{\dagger}}%number of samples noisy
\newcommand{\neps}{\epsilon_{\dagger}}
\newcommand{\eps}{\epsilon}

\newcommand{\ngam}{\gamma_{\dagger}}
\newcommand{\ntau}{\tau^{\dagger}}
\newcommand{\nP}{P^{\dagger}}
\newcommand{\nPe}{\hat{P}^{\dagger}}
\newcommand{\nZ}{Z_{\dagger}}
\newcommand{\Lnorm}{\mc{L}^{(2)}}
\newcommand{\RIP}{\Pi{}_{\TCL}} 
\newcommand{\RIPn}{\Pi{}_{\TCLn}}
\newcommand{\isingtree}{\mc{P}_{\T}(\alpha,\beta)}

\newcommand{\CP}{\mc{CP}_{\T}}
\newcommand{\Vast}{\bBigg@{5}}

\newcommand{\Err}{\mathrm{E}}
\newcommand{\Ecorr}{\mathrm{E}^{\text{corr}}\left(\epsilon\right)}

\newcommand{\Estrong}{\mathrm{E}^{\text{strong}}\left(\epsilon\right)}
\newcommand{\Ecorrn}{\mathrm{E}_{\dagger}^{\text{corr}}\left(\neps\right)}
\newcommand{\Ecascn}{\mathrm{E}_{\dagger}^{\text{cascade}}\left(\ngam\right)}
\newcommand{\Estrongn}{\mathrm{E}_{\dagger}^{\text{strong}}\left(\neps\right)}
\newcommand{\Einter}{\mathrm{E}_{\dagger} (\neps,\ngam)}

\newcommand{\BSC}{\mathrm{BSC}(q)^p}
\newcommand{\Q}{\mathrm{Q}}
\newcommand{\Nw}{N_{w}}
\newcommand{\Nwb}{N_{\bar{w}}}
\newcommand{\Nu}{N_{u}}
\newcommand{\Nub}{N_{\bar{u}}}
\newcommand{\Zsample}{Z_{e,u,\bar{u}}^{(i)}}
\newcommand{\Z}{Z_{e,u,\bar{u}}}
\newcommand{\Ysample}{M_{e,u,\bar{u}}^{(i)}}
\newcommand{\Y}{M_{e,u,\bar{u}}}
\newcommand{\Zf}{Z_{f,u,\bar{u}}}
\newcommand{\Zfsample}{Z_{f,u,\bar{u}}^{(i)}}
\newcommand{\Yf}{M_{f,u,\bar{u}}}
\newcommand{\Yfsample}{M_{f,u,\bar{u}}^{(i)}}
\newcommand{\DKL}{D_{\text{KL}}}
\newcommand{\efq}{\hat{f}^{\mu_k}_{\mu_{k-1}}(q)}
\newcommand{\UBxi}{S(\beta,q)}

\newcommand{\hd}{\mathrm{d}_{\mathrm{H}}}

\usepackage{algorithm}% http://ctan.org/pkg/algorithm
\usepackage[noend]{algpseudocode}% http://ctan.org/pkg/algorithmicx
\usepackage{float}
\usepackage{bbm}
\usepackage{xcolor}
\usepackage{dsfont}

\usepackage{enumitem}
\setitemize{noitemsep}

\ShortHeadings{Predictive Learning on Hidden Tree-Structured Ising Models}{Nikolakakis, Kalogerias and Sarwate}
\newtheorem{assumption}{Assumption}
\newtheorem{defn}{Definition}

\begin{document}

\title{Predictive Learning on Hidden Tree-Structured Ising Models%---Predictive Learning on Markov Trees through Noisy Observations
}

%---Predictive Tree Structure Learning \edit{on Hidden Sign-Valued Trees}{from Noisy Observations?}}

\author{\name Konstantinos E. Nikolakakis \email k.nikolakakis@rutgers.edu \\
       \addr Department of Electrical \& Computer Engineering\\
       Rutgers, The State University of New Jersey \\
	   94 Brett Road, Piscataway, NJ 08854, USA
       \AND
       \name Dionysios S. Kalogerias \email dionysis@msu.edu \\
       \addr  Department of Electrical \& Computer Engineering\\
       Michigan State University\\
       428 S. Shaw Lane, MI 48824, USA
       %98 Charlton St, Princeton, NJ 08540, USA
       \AND 
       \name Anand D. Sarwate \email anand.sarwate@rutgers.edu \\
       \addr Department of Electrical \& Computer Engineering\\
       Rutgers, The State University of New Jersey \\
	   94 Brett Road, Piscataway, NJ 08854, USA}

%%%%%%%%%%%%%%%%%%%%%%%%%%%%%%%%%%%%%%%%%%%%%%%%%%%%%%%%%%%%%%%%%%%%%%%%%%%%%%%%%%%%%%
%\editor{Bert Huang}
\thispagestyle{empty}
%%%%%%%%%%%%%%%%%%%%%%%%%%%%%%%%%%%%%%%%%%%%%%%%%%%%%%%%%%%%%%%%%%%%%%%%%%%%%%%%%%%%%%

\maketitle

\begin{abstract}%
We provide high-probability sample complexity guarantees for exact \textit{structure recovery} and accurate \textit{predictive learning} using noise-corrupted samples from an acyclic (tree-shaped) graphical model. The hidden variables follow a tree-structured Ising model distribution, whereas the observable variables are generated by a binary symmetric channel taking the hidden variables as its input (flipping each bit independently with some constant probability $q\in [0,1/2)$). In the absence of noise, predictive learning on Ising models was recently studied by~\cite{bresler2020learning}; this paper quantifies how noise in the hidden model impacts the tasks of structure recovery and marginal distribution estimation by proving upper and lower bounds on the sample complexity. Our results generalize state-of-the-art bounds reported in prior work, and they exactly recover the noiseless case ($q=0$). In fact, for any tree with $p$ vertices and probability of incorrect recovery $\delta>0$, the sufficient number of samples remains logarithmic as in the noiseless case, i.e., $\mc{O}(\log(p/\delta))$, while the dependence on $q$ is $\mc{O}\big( 1/(1-2q)^{4} \big)$, for both aforementioned tasks. We also present a new equivalent of Isserlis' Theorem for sign-valued tree-structured distributions, yielding a new low-complexity algorithm for higher-order moment estimation. 
\end{abstract}

%\tableofcontents{}

\section{Introduction}
Graphical models are a useful tool for modeling high-dimensional structured data. 
%In particular, Markov random fields (MRFs) are undirected graphical models in which variables, represented by nodes in a graph, satisfy conditional independence properties (Markov properties). 
The graph captures structural dependencies: its edge set corresponds to (often physical) interactions between variables. There is a long and deep literature on graphical models (see~\citet{koller2009probabilistic} for a comprehensive introduction), 
%Several topics on graphical models have been studied for more than 20 years. While there are numerous applications of graphical models in the literature we can indicatively refer a few of them due to space limitation: 
and they have found wide applications in areas such as image processing and vision~\citep{schwing2015fully,Li_2016_CVPR,lin2016efficient,liu2017deep,morningstar2017deep,wu2017coupled}, artificial intelligence more broadly~\citep{wainwright2003tree,wang2017learning}, signal processing~\citep{kim2013single,wisdom2016deep}, and gene regulatory networks~\citep{zuo2017incorporating,banf2017enhancing}, to name a few.

%, image reconstruction/recognition combined with deep learning implementations
%The book of Koller and Friedman~\citet{koller2009probabilistic} gives a comprehensive introduction to graphical models. 
An \textit{undirected graphical model}, or \textit{Markov random field} (MRF) in particular, is defined in terms of a hypergraph $\mc{G} = (\mc{V},\mc{E})$, that models the Markov properties of a joint distribution on $p \triangleq |\mc{V}|$ node variables $  (X_1, X_2, \ldots, X_p)\triangleq \mbf{X}$.  A \textit{tree-structured graphical model} is one in which $\mc{G}$ is a tree. We denote the tree-structured model as $\T=(\mc{V},\mc{E})$. In this paper, we consider binary models on $2p$ variables $(\mbf{X}, \bnX )$, where the joint distribution $\p (\cdot)$ of $\mbf{X}$ is a tree-structured Ising model distribution on $\{-1,+1\}^{p}$ and $\bnX = (Y_1, Y_2, \ldots, Y_p)$  is a noisy version of $\mbf{X}$, such that $Y_i = N_i X_i$ and $\{N_i\}$ are independent and identically distributed (i.i.d.) Rademacher noise with $\P(N_i=-1)=1- \P(N_i=+1)= q$, for all $i\in\mc{V}$. We refer to $\bX$ as the \textit{hidden layer} and $\bnX$ as the \textit{observed layer}. Under this setting, our objective is to recover the underlying tree structure and accurately estimate the distribution of the hidden layer $\mbf{X}$ (with high probability) using only the noisy observations $\bnX$. This is non-trivial because $\bnX$ does \textit{not} itself follow any tree structure; this is similar to more traditional problems in nonlinear filtering, where a Markov process of known distribution (and thus, of known structure) is observed through noisy measurements~\citep{arulampalam2002tutorial,jazwinski2007stochastic,van2009observability,douc2011consistency,kalogerias2016grid}. The sample complexity of the noiseless  version of our model was recently studied by~\citet{bresler2020learning}, where the well-known Chow-Liu algorithm~\citep{chow1968approximating} is employed for tree reconstruction.  Like them, we also analyze the Chow-Liu algorithm.
%The goal of this paper is to characterize the effect of the noise on \textit{hidden structure estimation}.

\subsection{Applications and Motivating Examples}

Models for joint distributions characterized by pairwise variable interactions have found many applications, with the Ising model being a popular model for binary variables. Our work is primarily motivated by examples of Ising models \textit{corrupted by noise}.
%Models with binary random variables frequently appear in different applications. The interactions of the variables are pairwise or they assumed to be pairwise as an approximation by the analyst. 
In many cases, the underlying graph-structured process cannot be observed directly; instead, 
%For many of those applications noise exists which acts on the data and allows us to 
only a noisy version of the process is available. Examples abound in  physics, computer science, biology, medicine, psychology, social sciences, and finance. %We proceed by presenting specific applications of our hidden model (Section~\ref{hidden_model}). 
%For the applications below each variable is a binary random vector $\bX$ which is assumed to follow an Ising model distribution, since the interaction of each variable with its neighbors is considered. The noise is typically introduced while one observes a realization of $\bX$ and gives rise to the hidden model.
Some applications motivating this work include the following: 

\textbf{1)} \textit{Statistical mechanics of population, social and pedestrian dynamics} (see related work by~\cite{matsuda1992statistical,castellano2009statistical}): The Ising model can be used to represent the statistical properties of the spreading of a feeling, behavior or the change of an emotional state among individuals in a crowd, where each individual interacts with his neighbors. 

\textbf{2)} \textit{Epidemic dynamics and  epidemiological models} by~\cite{barnett2013information,erten2017criticality}: Disease spread can be modeled through the Ising model, where each individual is susceptible (spin down) or ineffective (spin up).

\textbf{3)} \textit{Neoplastic transitions} and related applications in biology (\cite{torquato2011toward}): Each cell interacts with neighboring cells. Different cases are studied in the literature, for instance, healthy versus cancerous cells, malignant versus benign cells, where both can be modeled as spin up and spin down observations. The probability of diagnostic error is not zero which gives rise to the hidden model that we consider.

\textbf{4)} \textit{Differential Privacy}, originally proposed by~\cite{dwork2006our,dwork2006calibrating}: In computer science, differential privacy is used to guarantee privacy for individuals. A hidden model describes data gathered using a \emph{locally differentially private mechanism}~~\citep{Warner65:response,KLNRS08} such as randomized response.

\textbf{5)} \textit{Trading} and related applications in economics (see related work by~\cite{zhou2007self,takaishi2015multiple}): The Ising model has been considered in the literature to model increasing (spin up) or decreasing (spin down) price trends in a market.

\subsection{Structure Learning for Undirected Graphical Models and Related Work}
For a detailed review of methods for structure learning involving undirected and directed graphical models, see the relevant article by~\citet{drton2017structure}. In general, learning the structure of a graphical model from samples can be intractable~\citep{karger2001learning,hojsgaard2012graphical}. For general graphs, neighborhood selection methods~\citep{jalali2011learning,bresler2015efficiently,ray2015improved} estimate the conditional distribution for each vertex in order to learn the neighborhood of each node and therefore the full structure.  These approaches may use greedy search or $\ell_{1}$ regularization. For Gaussian or Ising models, $\ell_{1}$-regularization~\citep{ravikumar2010high}, the GLasso~\citep{yuan2007model,banerjee2008model}, or coordinate descent approaches~\citep{friedman2008sparse} have been proposed, focusing on estimating the non-zero entries of the precision (or interaction) matrix. Model selection can also be performed using score matching methods~\citep{hyvarinen2005estimation,hyvarinen2007some,nandy2015high,lin2016estimation}, or Bayesian information criterion methods~\citep{foygel2010extended,gao2012tuning,barber2015high}. Other works address non-Gaussian models such as elliptical distributions, $t$-distribution models or latent Gaussian data~\citep{finegold2011robust,vogel2011elliptical,vogel2014robust,bilodeau2014graphical}, or even mixed data~\citep{fan2017high}. 

For tree- or forest-structured models, exact inference and the structure learning problem are significantly simpler: the Chow-Liu algorithm provides an estimate of the tree or forest structure of the underlying graph~\citep{chow1968approximating,wainwright2008graphical,edwards2010selecting,tan2011learning,liu2011forest,daskalakis2018testing,bresler2020learning}. Furthermore, marginal distributions and maximum values are simpler to compute using a variety of algorithms (sum-product, max-product, message passing, variational inference)  \citep{pearl1988morgan,lauritzen1996graphical,wainwright2003tree,wainwright2008graphical}). 

%A \emph{hidden model} (c.f. hidden Markov models) refers to a graphical model in which the %observed variables are dependent on underlying hidden variables. In our case, the hidden model is %a tree-shaped Ising model.

\textit{The noiseless counterpart} of the model considered in this paper was studied recently by~\citet{bresler2020learning}; in this paper, we extend their results to the hidden case, where samples from a tree-structured Ising model are passed through a binary symmetric channel with crossover probability $q\in[0,1/2)$. Of course, in the special case of a linear graph, our model reduces to a hidden Markov model.
%The proof strategy is similar but requires new analysis.
%In our hidden model, samples from a tree-structured Ising model are passed through a  binary symmetric channel with crossover probability $q\in[0,1/2]$. 
Latent variable models are often considered in the literature when some variables of the graph are deterministically unobserved~\citep{chandrasekaran2010latent,anandkumar2013learning,ma2013alternating,anandkumar2014tensor}. 
Our model is most similar to that studied by Chaganty et al.~\citep{chaganty2014estimating}, in which a hidden model is considered with a discrete exponential distribution and Gaussian noise. They solve the parameter estimation problem by using moment matching and pseudo-likelihood methods; the structure can be recovered indirectly using the estimated parameters.

\noindent\textbf{Connection with Phylogenetic Estimation.} In phylogenetic estimation problems the goal is to learn the structure of tree given only observations form the leaves~~\citep{steel1997few}. The sample complexity of phylogenetic reconstruction algorithms grows exponentially with respect to the depth of the tree~\citep{steel1997few}, however if we are interested in reconstructing only parts of the tree which are ``close'' to the leaves then the depth of tree does not affect the sample complexity~\citep{daskalakis2009phylogenies}. The hidden structure learning problem that we consider in this paper is a special case of phylogeny estimation problem with constant depth; there is exactly one noisy observable for each hidden node of the tree. In contrast with phylogenetic estimation approaches, Chow-Liu algorithm is simple and computationally more efficient, while the sample complexity is of the same order\footnote{while considering the depth fixed} with the well-known phylogenetic reconstruction methods, to name a few ``Dyadic Closure'' method by~\cite{steel1997few}, the ``Contractor-Extender'' and ``Cherry-picking'' algorithms by~\cite{daskalakis2006optimal,daskalakis2009phylogenies,daskalakis2013alignment}. On the other hand, the approach of distribution estimation by matching the structure and the correlations~\citep{bresler2020learning} has not been considered in the phylogenetic estimation literature. Based on the above discussion, the following interesting question naturally rises: How well can we estimate the distribution of a hidden tree structured model while having access only to the leaves of the tree? The latter remains open problem for future work.

\subsection{Statement of Contributions} 

We are interested in answering the following general question: \textit{How does noise affect the sample complexity of the structure and predictive learning procedure?}  That is, given \textit{only} noisy observations, our goal is to learn the tree structure of the hidden layer in a well-defined and meaningful sense. The MLE-structure from tree-structured (noiseless) data is the output of the Chow-Liu algorithm~\citep{chow1968approximating}. However, the MLE-structure from noisy data is not consistent with the hidden structure in general because the graphical model of the observables is a complete graph. Further, the (latent) MLE of the actual interaction parameters $\theta$ of the hidden layer is intractable. In Sections \ref{Hidden_structure_estimation} and \ref{MLE} we explain the importance of Chow-Liu algorithm in our setting, we show why the classical MLE approach fails, and we discuss the connection between the output of the Chow-Liu algorithm and an alternative, projection-based MLE approach.

The estimated structure is an essential statistic for estimating the underlying distribution of the hidden layer, allowing for predictive learning. Specifically, based on the structure estimate, we are also interested in appropriately approximating the tree-structured distribution under study, which can then be used for accurate predictions.
%\footnote{The notion of ``inference in order to make prediction'' was introduced by %\citet{bresler2018learning}, where the accuracy of the estimation is measured by the 11small set %Total Variation'' distance. We refer to this as  ``Predictive Learning'' for short.}  
We also consider the problem of hidden layer higher-order moment estimation of tree-structured Ising models and, in particular, how such estimation can be efficiently performed, on the basis of noisy observations.

A summary of the main contributions of this paper is as follows: 
\begin{itemize}
\item A lower bound on the sufficient number of samples needed to recover the exact hidden structure with high probability, by using the Chow-Liu algorithm. We also show an upper bound on the necessary number of samples for any algorithm to estimate the hidden structure. The proof of the lower bound follows the general structure of Lemmata 8.1-8.4 by~\cite{bresler2020learning}, however we need to extend the necessary events and prove new concentration bounds for the noisy setting. 
%These events for exact structure recovery depends only on noisy observations, while the structure of noisy variables does not factorize according to any tree. 
Although the graphical model of the observables is a complete graph we show that the Chow-Liu algorithm (with input a finite number of noisy samples) returns the exact tree of the hidden layer with high probability and we characterize its sample complexity. The proof of the upper bound uses the same construction of the approach in Section 7.1 by~\cite{bresler2020learning} but requires the combination of Fano's inequality and a strong data processing inequality (SDPI) by~\cite{polyanskiy2017strong}. Specifically, we show that SDPI's can be a useful tool to derive minimax bounds when closed form expressions or upper bounds of the KL-divergence are hard to be found. The later is of independent interest and it can be applied to other machine learning problems that involve noisy observations.

\item Determination of the sufficient and necessary number of samples for accurate predictive learning. We analyze the sample complexity of learning distribution estimates, which can accurately provide predictions on the hidden tree. The estimates are computed using the noisy data. Predictive learning under noisy samples is challenging because structural properties such as the independence of random variables $X_i X_j$ and correlation estimates $\hat{\E}[X_iX_j]$ for $(i,j)\in\mc{E}$ do not hold for the noisy observable $\bnX$. To overcome this we evaluate the required conditional distributions of the dependent variables, construct a martingale difference sequence, and prove a high probability bound of the event that involves these variables by applying a concentration bound for supermartingales (generalized Bennet's inequality~\citep{fan2012hoeffding}). We refer the reader to Section \ref{necessary events discussion} for a detailed discussion about the technical contributions and a sketch of proof of the main result.

\item A closed-form expression \textit{and} a computationally efficient estimator for higher-order moment estimation in tree-structured Ising models. This result corresponds to an equivalent statement of Isserlis' theorem for sign-valued tree models. Given pair-wise correlations and the tree (or estimates of both, from noisy or noiseless data) we provide an algorithm that runs on the tree and returns the expression of high-order moments. The proof involves the existence and identification of (minimum length) disjoint paths among any set of pairs of nodes. The proposed algorithm (Algorithm \ref{alg:matching_pairs}) identifies these paths that yield the expression of the moments. The results may be of independent interest for a computational efficient exact or approximated higher-moment evaluation.

\end{itemize}
Our main results Theorem \ref{thm:sufficient} and Theorem \ref{thm:Main_result} provide the amount of finite samples needed for exact structure recovery and accurate predictive learning with high probability. Although we are interested in the finite sample complexity bounds, our results are also asymptotically optimal. That is, for any fixed (constant) $q\in[0,1/2)$ the order of the upper bound (necessary number of samples) matches the corresponding (lower) minimax bound. The sample complexity bounds that we provide are the extended form of state of the art (noiseless setting) bounds by~\cite{bresler2020learning}. By setting $q=0$, our bounds reduce to the noiseless setting bounds. Further, the explicit version of our results (see Section \ref{Main Results}) are continuous functions of the cross-over probability $q$.

\begin{table}[t]
\centering
\begin{tabular}{r|l}
\textbf{Symbol} & \textbf{Meaning} \\
\hline 
$p$ & number of variables nodes in the tree \\
$\p (\bx)$ & $\exp\big( \sum_{(i,j)\in\mc{E}} \theta_{i,j}x_i x_j \big)/Z(\theta)$, $\bx\in\{-1,+1\}^p$, $Z(\theta)\!:$ partition function\\
$\alpha$ & minimum $|\theta_{ij}|$ in the Ising model, $\min_{i,j\in\mc{V}} |\theta_{ij}|$ \\
$\beta$ & maximum $|\theta_{ij}|$ in the Ising model, $\max_{i,j\in\mc{V}} |\theta_{ij}|$ \\
$\T$ & Original tree of the model\\
$\isingtree$ & set of tree-structured Ising models with $\alpha \leq |\theta_{ij}| \leq \beta$ \\
$n$ & number of samples \\
$q$ & crossover probability of the BSC, $q\in[0,1/2)$\\
$c_q$ & $1-2q$\\
$\p(\cdot)$ & distribution of the hidden node variables, $\bX\sim \p(\cdot)\in \isingtree$ \\
$\p_{\dagger}(\cdot)$ & distribution of the observable node variables $\bnX\sim \p_{\dagger}(\cdot)$\\
$\mathds{1}_{\bds{A}}$ & indicator function of the set $\bds{A}$\\
$\KL$ & KL divergence\\
$\SKL$ & symmetric KL divergence\\
$I(X,Y)$ & mutual information of X,Y\\
$d_{\text{TV}}$ & total variation distance \\
$\Lnorm (P,Q)$ & $\sup_{i,j\in\mc{V}} d_{\text{TV}}\left(P_{ij},Q_{ij}\right)$, and $P_{ij},Q_{ij}$ the pairwise marginals of $P,Q$ \\
$\bX^{1:n}$ & $ n$ independent observations of $\bX$\\
$\bnX^{1:n}$ & $ n$ independent observations of $\bnX$\\
$\TCL$ & Chow-Liu-estimated structure from noiseless data $\bX^{1:n}$\\
$\TCLn$ & Chow-Liu-estimate of the hidden tree structure $\T$ from noisy data $\bnX^{1:n}$ \\
$\tpath_{\T} (w,\tilde{w})$ &the set of edges which connects the nodes $w,\tilde{w}\in\mc{V}_{\T}$ \\ 
$\hat{\mu}_{i,j}$ & $\frac{1}{n} \sum^{\n}_{k=1}  X_i^{(k)}\! X_{\!j}^{(k)}$ 
\\$\nmue_{i,j}$ & $\frac{1}{n}\sum^{n}_{k=1} \nX_i^{(k)} \nX_{\!j}^{(k)}$ \\
%$\Pi{}_{\T}(P)$ & Projection of $P$ to the set of tree structured distribt Reverse-I projection \eqref{eq:RIP}\\
$\RIPn (\hat{\p}_{\dagger} )$ & estimator of the distribution $\p(\cdot)$ from noisy data $\bnX^{1:n}$\\
$\eta$ & maximum error on the distribution estimation: $\Lnorm (P,\hat{P})\leq \eta$\\
$\delta$ & maximum probability of error, the notation depends on the task\\ 
 & \qquad in structure estimation: $\P (\TCLn \neq \T)\leq \delta$\\
 & \qquad in predictive learning: $	\P \lp \Lnorm \lp \p (\cdot),\RIPn (\hat{p}_{\dagger} ) \rp \leq \eta \rp \geq 1-\delta.$
\end{tabular}
\caption{Notation/Definitions. \label{table:notation}}
\end{table}

\subsection{Notation}
 Boldface indicates a vector or tuple and calligraphic face for sets and trees. The sets of even and odd natural numbers are $2\mbb{N}$ and $2\mbb{N}+1$ respectively. For an integer $n$, define $[n] \triangleq \{1,2,\ldots n \}$. 
%We will use $p$ for the size of the vertex set $\mc{V}$. 
The indicator function of a set $A$ is $\bds{1}_A$. For a graph $\mc{G} = (\mc{V}, \mc{E})$, $\mc{V} = [p]$ indexes the set of variables $\{X_1, X_2, \ldots, X_p\}$, for any pair of vertices $i,j \in \mc{V}$ the correlation $\mu_{ij} \defeq \E\left[X_{i}X_{j}\right]$ and for any edge $e = (i,j) \in \mc{E}$ it is $\mu_e\triangleq\E[X_i X_j]$. For two nodes $w,\tilde{w}$ of a tree, the term $\tpath(w,\tilde{w})$ denotes the set of edges in the unique path with endpoints $w$ and $\tilde{w}$. 
%The probability mass function of $\bX$ is denoted as $\p (\cdot)$. 
 Further, BSC$(q)^p$ denotes a binary symmetric channel with crossover probability $q$ and block-length $p$. The $\BSC$ is a conditional distribution from $\{-1,1\}^p \to \{-1,1\}^p$ that acts componentwise independently on $\bX$ to generate $\bnX$, such that $X_i = N_i \nX_i$ and $\mbf{N}$ is a vector of i.i.d.~Rademacher variables equal to $+1$ with probability $1 - q$.
%where $q$ is the probability crossover probability $q$, which acts componentwise independently on $\bX$ and generates $\bnX$. 
We use the symbol $\dagger$ to indicate the corresponding quantity for the observable (noisy) layer. For instance, $\np(\cdot)$ is the probability mass function of $\bnX$ and $\nmu_{i,j}\triangleq \E [\nX_i \nX_j]$ corresponds to the correlation of variables $\nX_i ,\nX_j$. For our readers' convenience, we summarize the notation in Table \ref{table:notation}.

%Similarly, $\nn$ and $n$ denote the number of samples of $\bnX$ and $\bX$ respectively. The estimated structure from corrupted by noise observations is denoted as $\TCLn$, while $\TCL$ is the corresponding estimate in the noiseless setting.
% hard to understand these last two until they are introduced -- moving them later.

\begin{figure}[!ht]
%\begin{subfigure}{.5\textwidth}
  \centering
  \includegraphics[width=0.8\linewidth]{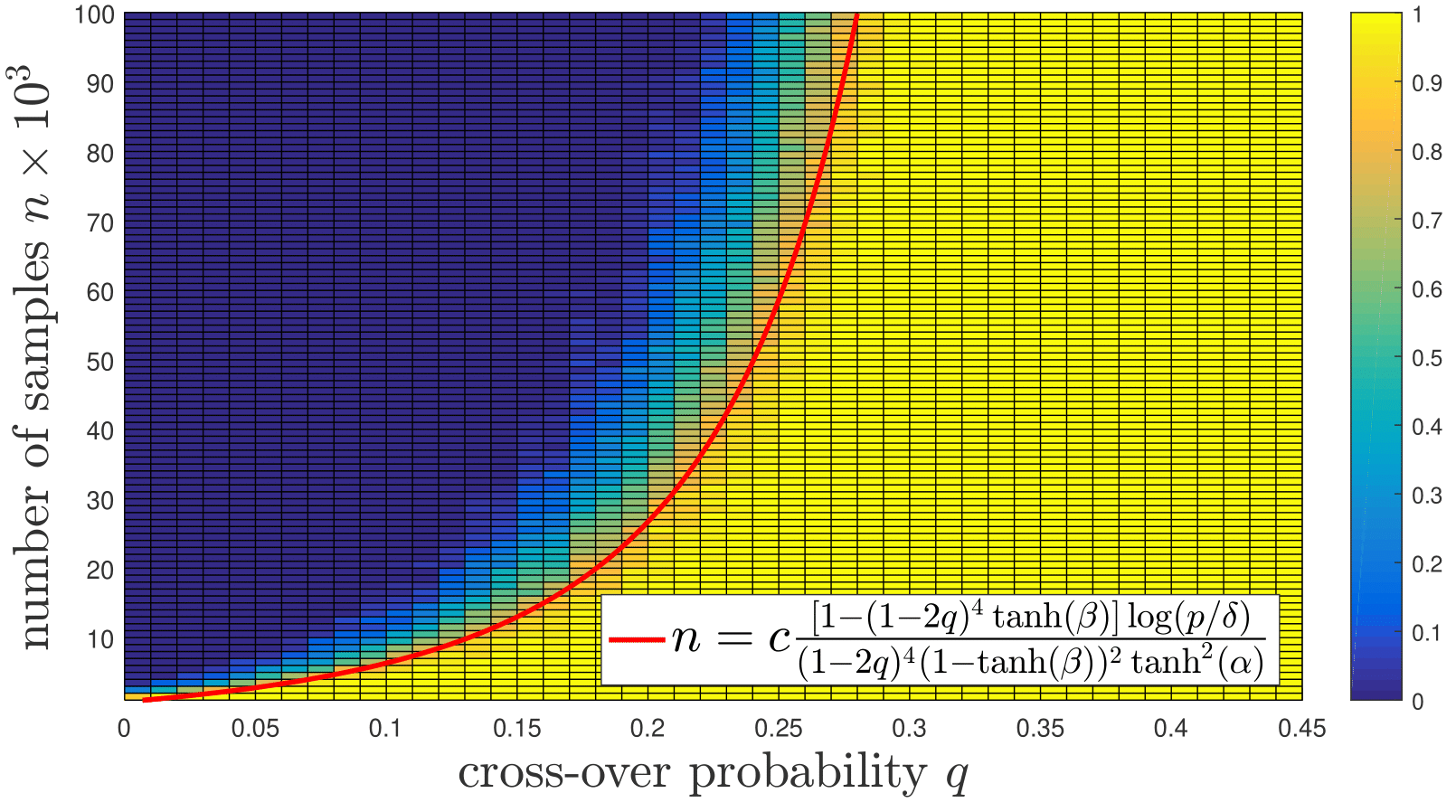}
\caption{The simulation corresponds to structure learning. Comparison of the experimental results (heat-map) and the theoretical bound of Theorem \ref{thm:sufficient}, the bound that yields \eqref{eq:sf_1}. The colored regions denote different values of the estimated probability of error $\delta$ (at least one edge has been missed). The value of $\delta$ varies between $0$ and $1$ while the parameters $\alpha=0.2,\beta = 1.1 ,p = 100$ are fixed. The red line shows the bound from Theorem \ref{thm:sufficient} (the explicit form of Theorem  \ref{theorem: structure: simple}). The code of the experiment is available at \url{https://github.com/KonstantinosNikolakakis/Structure-Learning}. \label{fig:sfig1}}
%\label{fig:fig}
\end{figure}
\subsection{Summary of the Results}

In this section, we present a summary of the main results of our work up to constant factors $C,C'>0$. We refer the reader to Table \ref{table:notation} for the definition of the model parameters. We provide the explicit statements of the results, and we specify the constants in Section \ref{Main Results}. Recall that, the random vector $\bnX\in \{-1,+1\}^p$ is the output of the binary symmetric channel $\BSC$ with input the random vector $\bX\sim p(\cdot)\in\isingtree$.

\subsubsection{Structure Learning}
The first results provides the sufficient number of samples for exact structure recovery.
\begin{theorem}[Sample Complexity for Structure Learning.]\label{theorem: structure: simple}The Chow-Liu algorithm with input $n$ noisy samples $\bnX^{1:n}$ exactly estimates the hidden tree structure $\TCLn\equiv\T$ with probability at least $1-\delta\in (0,1)$, as long as \begin{align}\label{eq:sf_1}
    n>C\frac{e^{2\beta(1+\mathds{1}_{q\neq0})}}{(1-2q)^4\tanh^2 (\alpha)} \log (p /\delta).
\end{align} 
\end{theorem} The order with respect to $\beta$ is $\mc{O}( e^{ 4\beta})$ for all $q>0$. The bound in \eqref{eq:sf_1} exactly reduces to the noiseless case~\cite[Theorem 3.2]{bresler2020learning}. Additionally, the explicit form of the result, Theorem \ref{thm:sufficient}, shows that the bound is also a continuous function of $q\in[0,1/2)$. The next proposition gives the necessary number of samples for exact structure recovery.  %We can further compress the inequality \eqref{eq:sf_1} to the following sufficient condition that provides a shorter form of the structure learning result for $q>0$,\begin{align}
  %  n>C\frac{e^{2\beta(1+\mathds{1}_{q\neq0})}}{(1-2q)^4\tanh^2 (\alpha)} \log (p /\delta).
%\end{align}%\begin{align}
    %n>C\frac{e^{4\beta}}{\tanh^2 (\alpha)}\left[  (1-2q)^{-4}- \tanh (\beta) \right] \log (p /\delta).
%\end{align}  

\begin{proposition}\label{prop_structure}
No algorithm can recover the structure with probability great than $1/2$ if \begin{align}\label{eq:nec_structure}
    n<C'\frac{e^{2\beta}[1-(4q(1-q))^p]^{-1}}{\alpha\tanh(\alpha)}\log \lp p\rp.
\end{align}
\end{proposition} Note that the terms $(1-2q)^{-4}$ and $[1-(4q(1-q))^p]^{-1}$ introduce a gap between the sample complexity of \eqref{eq:sf_1} and \eqref{eq:nec_structure}. However, the sample complexity of Theorem \ref{theorem: structure: simple} is indeed accurate. To illustrate this experimentally, we show that the theoretical and experimental bounds exactly match, see Figure \ref{fig:sfig1}. The latter indicates that the Chow-Liu algorithm requires exactly the number of samples that our theoretical result suggest (see Figure \ref{fig:sfig1}). On the other hand, Proposition \ref{prop_structure} provides the necessary number of samples, for any algorithm. %A construction of a more efficient (in terms of the sample complexity) algorithm remains out of the scope of the paper, while Chow-Liu is computationally the least expensive. 
Finally, we conjecture that the bound of Proposition 1.2 is tight only under the low temperature regime $|\theta_{i,j}|\to\infty$ for all $i,j\in\mc{E}$. The derivation of generalized tighter forms of the bound in \eqref{eq:nec_structure} is challenging and left for future work.%a different algorithm with less sample complexity than the Chow-Liu might exist (for $q> 0$). We believe that such algorithms can only be efficient under specific regimes, (for instance the low temperature regime).

\subsubsection{Predictive Learning}\label{Predictive_Learning_Intro}

To learn the tree-shaped distribution $\p(\cdot)\in \isingtree$ of $\bX$ from $n$ noisy samples $\bnX^{1:n}$, we first estimate the correlations $\hat{\mu}^{\dagger}_{i,j}$ for all $i,j\in\mc{V}$. We then estimate the tree structure $\TCLn$ by running the Chow-Liu algorithm with input the candidate edge weights $\hat{\mu}^{\dagger}_{i,j}$ and finally evaluate the estimator of $\p(\cdot)$ (by matching correlations) as follows\footnote{The distribution in \eqref{eq:estimator} is a function of $\bx$, however we suppress the notation for consistency with prior work and for sake of space. } \begin{align}
     \RIPn (\hat{\p}_{\dagger})\triangleq\frac{1}{2} \prod_{\lp i,j\rp\in \mc{E}_{\TCLn}} \frac{1+x_{i}x_{j}\frac{\nmue_{i,j}}{(1-2q)^2}}{2},\quad \bx\in \{-1,+1\}^p.\label{eq:estimator}
\end{align}
Note that one restriction of our approach is that the distribution estimator requires the value $q$ to be known. The same restriction appears in other structure learning from noisy data approaches~\citep{goel2019learning}. However, in our setting $q$ is required only for the predictive learning, while the Chow-Liu algorithm and the structure estimation does not require $q$ to be known. Under the assumption that $q$ is unknown, one can first learn its value  through an independent procedure~\cite[Section 5]{goel2019learning}. The accuracy of the estimated distribution in \eqref{eq:estimator} is measured by the small-set Total Variation (ssTV), that captures the estimation error on the $k^\text{th}$-order marginals~\citep{georgii2011gibbs,rebeschini2015can,bresler2020learning}. Let $P_{\mc{S}},Q_{\mc{S}}$ denote the marginals of $P,Q$ on a set $\mc{S}\subset \mc{V}$, and $|\mc{S}|=k$. Then the $k^{\text{th}}$ order ssTV of $P$ and $Q$ is defined as \begin{align}
    \label{eq:SSTV_DEF_intro}
    \mc{L}^{(k)}\left(P,Q\right) & \triangleq \sup_{\mc{S}:\left|\mc{S}\right|=k} d_{\text{TV}}\left(P_{\mc{S}},Q_{\mc{S}}\right).
    \end{align}The next results provides the necessary number of samples for accurate distribution estimation by guaranteeing that the $\Lnorm$ is less than a small positive number $\eta$ with high probability. We provide guarantees on higher-order marginals ($k>2$) in Section \ref{higher order moments - Isserlis'}.

  %A key component of the bound is the following function \begin{align}
  %  \Gamma(\beta,q)\triangleq \lp \frac{1-(1-2q)^2}{1-(1-2q)^4\tanh^2 (\beta)}\rp^2,\quad  \beta>0 \text{ and } q\in [0,1/2).
%\end{align} Note that $\Gamma(\beta,q)\in [0,1]$ for all $\beta>0$ and $q\in [0,1/2)$, and $\Gamma(\beta,0)=0$ for all $\beta>0$.
\begin{figure}[!t]
%\begin{subfigure}{.5\textwidth}
  \centering
  \includegraphics[width=0.6\linewidth]{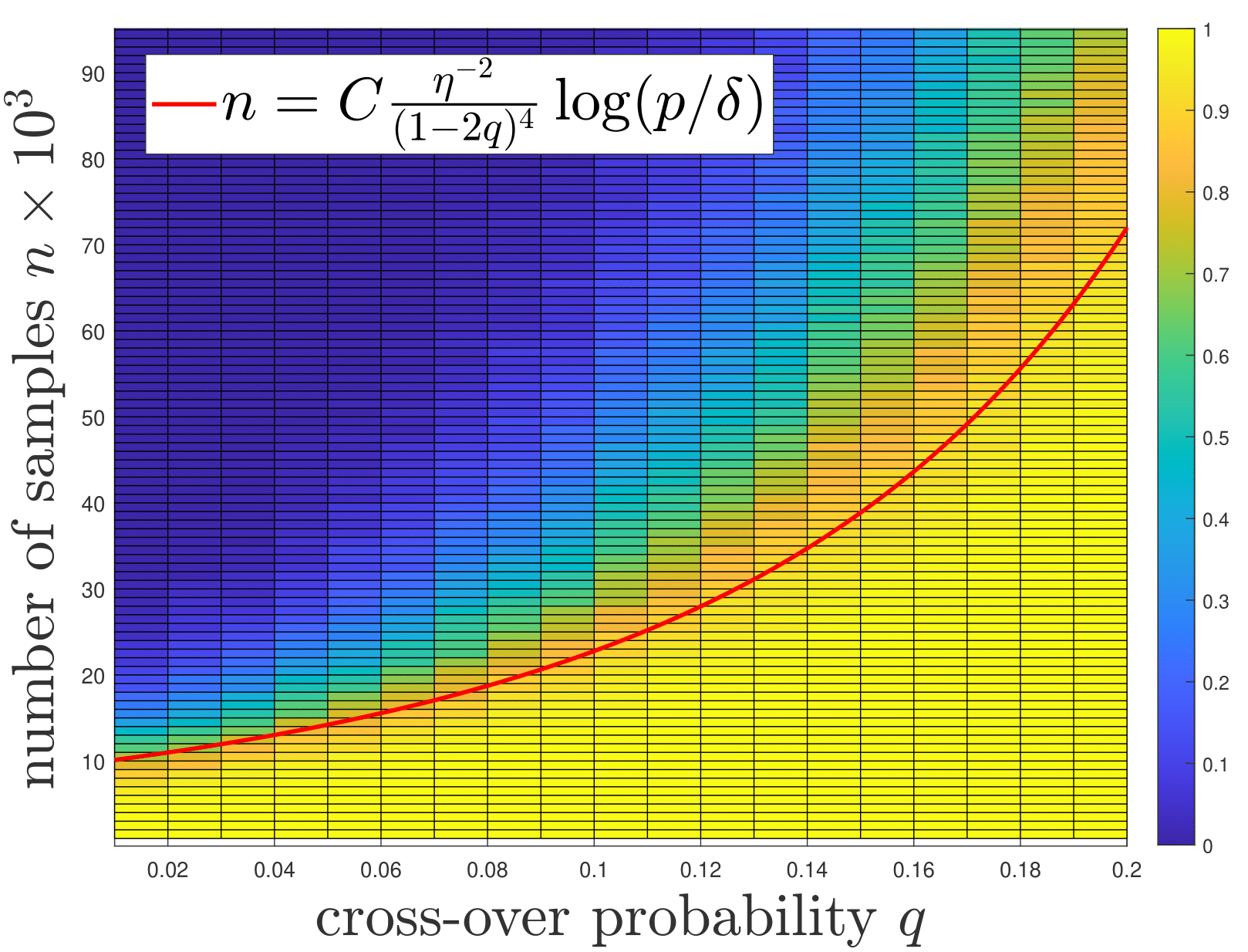}
  %\caption{1b}
%\end{subfigure}
\caption{The simulation corresponds to predictive learning. Comparison of the experimental results (heat-map) and the theoretical bound of Theorem \ref{thm:Main_result:short}. The colored regions denote different values of the estimated probability of error $\delta$ (ssTV to be greater than a fixed number $\eta$). The value of $\delta$ varies between $0$ and $1$ while the parameters $\eta=0.03,\beta=1.1,p=31$ are fixed. The code of the experiment is available at \url{https://github.com/KonstantinosNikolakakis/Predictive-Learning}.}\label{fig:sfig2}
%\label{fig:fig}
\end{figure}

\begin{theorem}[Sample Complexity for Predictive Learning]\label{thm:Main_result:short}  Fix $\delta\in (0,1)$. Choose $\eta>0$ (independent of $\delta$). If
\begin{align}
    n\geq C\max \left\{\frac{1}{ \eta^2(1-2q)^4}, \frac{ e^{2\beta(1+\mathds{1}_{q\neq 0})}}{(1-2q)^4} , \frac{e^{4\beta}}{\eta^2}\mathds{1}_{q\neq 0} \right\}\log\lp\frac{p}{\delta}\rp\label{eq:short_main_result}
\end{align} then \begin{align}
	\P \lp \Lnorm \lp \p (\cdot),\RIPn (\hat{p}_{\dagger} ) \rp \leq \eta \rp \geq 1-\delta.
	\end{align} 
\end{theorem}
Note that the dependence on $\beta$ is $\mc{O}(e^{4\beta})$ for accurate distribution learning from noisy data (similarly to the structure learning task, Theorem \ref{theorem: structure: simple}). The bound in \eqref{eq:short_main_result} \textit{exactly} reduces to the noiseless setting bound by~\cite[Theorem 3.3]{bresler2020learning}. Theorem \ref{thm:Main_result:short} is a short version of the main result of the paper. The explicit statement, Theorem \ref{thm:Main_result}, shows that the bound is also continuous at $q\to 0$.

%\begin{theorem}[Sample Complexity for Predictive Learning]\label{thm:Main_result:short}
%Fix $\delta\in (0,1)$ and $\eta\in (0, 1/20]$. Given $n$ noisy samples $\bnX^{1:n}$ if
%	\begin{align}
%	n> C\frac{\max\{1, e^{4\beta-2} \}}{\eta^2(1-2q)^4}\log (p/\delta).
%	\end{align} 
%then 
%	\begin{align}
%	\P \lp \Lnorm \lp \p (\cdot),\RIPn (\hat{p}_{\dagger} ) \rp \leq \eta \rp \geq 1-\delta.
%	\end{align}
%\end{theorem} 

Conversely, the following proposition gives an upper bound on the necessary number of samples for accurate marginal distributions' estimation under the assumption $\beta>\alpha$.
\begin{proposition}\label{proposition_predictive} Fix $\eta>0$ such that $\eta\leq (\tanh(\beta)-\tanh(\alpha))/2$. Then
no algorithm can accurately estimate the distribution of the hidden variables (ssTV less than $\eta>0$) with probability greater than $1/2$ if \begin{align}
    n<C'\eta^{-2}[1-(4q(1-q))^p]^{-1}\log (p).\label{eq:PL_nec_sect_1}
\end{align}
\end{proposition}

\begin{table}[]
\centering
\begin{tabular}{|p{2.2cm}||p{4.2cm}|p{7.5cm}|}
\hline
 \multicolumn{3}{| c |}{ Sufficient Number of Samples}
 \\\hline Task/Setting &  Noiseless (prior work) &  Noisy   \\ \hline
 \vspace{+1pt}Structure Learning&     
 \vspace{+3pt}$C\frac{e^{2\beta}}{\tanh^2 (\alpha)} \log (p /\delta)$  
 & \vspace{+1pt} $C\frac{e^{2\beta(1+\mathds{1}_{q\neq0})}}{(1-2q)^4\tanh^2 (\alpha)} \log (p /\delta)$\vspace{+1pt} \\ \hline
 \vspace{+1pt}Predictive Learning& \vspace{+4pt}$ C\max\{\eta^{-2}, e^{2\beta}\}\log (p/\delta)$  & \vspace{+1pt}$  C\max \left\{\frac{\eta^{-2}}{ (1-2q)^4}, \frac{ e^{2\beta(1+\mathds{1}_{q\neq 0})}}{(1-2q)^4} , \frac{e^{4\beta}\mathds{1}_{q\neq 0}}{\eta^2} \right\}\log\lp p/\delta\rp$   \\ \hline   
\end{tabular}
\caption{Sufficient number of samples for accurate structure and predictive learning.  \label{table:Results}}
\end{table} 

\begin{table}[]
\centering
\begin{tabular}{|p{2.2cm}||p{4.2cm}|p{7.5cm}|}
\hline 
 \multicolumn{3}{| c |}{ Necessary Number of Samples}\\
\hline Task/Setting
 &  Noiseless (prior work)&  Noisy   \\ \hline
 \vspace{+1pt}Structure learning& \vspace{+3pt}$C'\frac{e^{2\beta}}{\alpha\tanh(\alpha)}\log \lp p\rp$ &   \vspace{+1pt}$C'\frac{e^{2\beta}[1-(4q(1-q))^p]^{-1}}{\alpha\tanh(\alpha)}\log \lp p\rp$ \\ \hline
 \vspace{+1pt}Predictive learning&\vspace{+3pt} $ C'\eta^{-2}\log (p)$  &\vspace{+1pt} $ C'\eta^{-2}[1-(4q(1-q))^p]^{-1}\log (p)$   \\ \hline  
\end{tabular}
\caption{Necessary number of samples for structure and predictive learning. \label{table:Results1}}
\end{table}

%Table \ref{table:Results} and Table \ref{table:Results1} present a summary of the state of the art bounds for the predictive learning as well. In Table \ref{table:Results} the structure learning bounds characterize the sample complexity of the Chow-Liu algorithm. The predictive learning bounds characterize the required samples for accurate estimate $\RIPn (\hat{p}_{\dagger} )$ of the model distribution.  Conversely, Table \ref{table:Results1} provides the upper bounds on the number of samples for which any algorithm fails to recover the structure and to estimate accurately the distribution with probability greater than $1/2$. 
A quick comparison of \eqref{eq:short_main_result} and \eqref{eq:PL_nec_sect_1} shows that there is a gap between the sufficient and necessary number of samples. Our experiments (Figure \ref{fig:sfig2}) confirm the accuracy of our theoretical results. For instance the bound of Theorem \ref{thm:Main_result:short} exactly matches the experimental curve. For further discussion related to the gap between the upper and lower bounds see Section \ref{MLE}. Further, we conjecture that bound in \eqref{eq:PL_nec_sect_1} is tight only under the low temperature regime, similarly to the Proposition \ref{prop_structure}.  The derivation of tighter characterization of the necessary number of samples Propositions \ref{prop_structure} and \ref{proposition_predictive} remains an problem for future work. Additional plots of the experiments are provided in Section \ref{Simulations}. Finally, Table \ref{table:Results} and \ref{table:Results1} summarize the state-of-the-art bounds of the noiseless setting by~\citet{bresler2020learning} and the extended version under the noisy setting that we study in this paper.

To summarize, the following holds for both structure and predictive learning: the dependence on the parameter $\beta$ is of the order $\mc{O}(e^{2\beta})$ for $q= 0$ and becomes $\mc{O}(e^{4\beta})$ for positive values of $q$. Further, the bounds are continuous functions of $q$, as our results suggest (for the continuity see the explicit form of the results Theorem \ref{thm:sufficient} and Theorem \ref{thm:Main_result}.) Similarly to the noiseless case, the following statement holds when noise exists as well: Under the high temperature regime ($\alpha$ close to zero), structure learning requires much more data than the predictive learning task, because of the $\tanh^2(\alpha)$ in the denominator of the bound in \eqref{eq:sf_1}. On the contrary, the required number of samples for predictive learning \eqref{eq:short_main_result} does not depend on $\alpha$. Specifically, exact structure recovery is not necessary for learning the distribution efficiently, that is, \textit{weak} edges' identification failure does not affect the predictive learning task. We refer the reader to Section \ref{necessary events discussion} for the definition of \textit{weak/strong} edges and additional explanation. Finally, for $q>0$ an extra term that involves both $\beta$ and $\eta$ appears in the bound of Theorem \ref{thm:Main_result:short}, while for values of $q$ close to zero and $q=0$ vanishes. 

The pairwise correlations of end-point vertices ($E[X_i X_j]$: $(i,j)\in\mc{E}$) are sufficient statistics, and as expected, the accuracy of pairwise marginals corresponds to accuracy of higher order marginals and accurate estimation of higher order moments. In Sections \ref{higher order moments - Isserlis'} and \ref{higher_order_moments} we provide a method for evaluating higher order moments (and marginals) from noisy observations. Our approach is based on an equivalent of Isserlis' theorem for tree-structured Ising models that is also of independent interest.

\section{Preliminaries and Problem Statement}
In this section, we introduce our model of \textit{hidden sign-valued Markov random fields on trees}.

\subsection{Undirected Graphical Models}

We consider \textit{sign-valued graphical models} where the joint distribution $\p (\cdot)$ has support $\{-1,+1\}^p$. Let $\bX=\lp X_1,X_2,\ldots,X_p \rp \in \{-1,+1\}^p$ be a collection of sign-valued (binary) random variables. Then, $\mathds{1}_{X_i =x_i} \equiv (1+x_i X_i)/2$, and the distribution of $\bX$ is	
    \begin{align*}
	\p(\bx)&=\E\left[ \prod^{p}_{i=1} \mathds{1}_{X_i =x_i}\right] = \frac{1}{2^{p}} \left[ 1 + \sum_{k \in [p]}  \sum_{\mc{S} \subset \mc{V}: |\mc{S}| = k} \E\left[ \prod_{s \in \mc{S}} X_s \right] \prod_{s \in \mc{S}} x_s  \right],\quad\bx\in\{-1,+1\}^{p}. \numberthis \label{eq:general}
	\end{align*} 
In this paper we assume that the marginal distributions of the $X_i$ are uniform, that is,
	\begin{align}\label{eq:uniform}
    	\P \left( X_{i} = \pm 1 \right)=\frac{1}{2},\quad \forall i\in \mc{V}.
        %\P \left( X_{i} = + 1 \right)=\P \left( X_{i} = - 1 \right)=\frac{1}{2}\hspace{+0.3cm} \forall i\in \mc{V}.
	% \E[X_i] = 0, \quad \text{for all } i\in\mc{V}. 
    \end{align} 
Thus, $\E[X_i] = 0$, for all $i \in \mc{V}$. A distribution is Markov with respect to a hypergraph $\mc{G}=(\mc{V},\mc{E})$ if for every node $i$ in the set $\mc{V}$ it is true that $\P \lp X_i \vert\ \bx_{\mc{V}\setminus \{i\}} \rp=\P \lp X_i \vert\ \bx_{\mc{N}(i)} \rp$,  
where $\mc{N}(i)$ is the set of neighbors of $i$ in $\mc{G}$. One subclass of distributions for which the Markov property holds is the \textit{Ising model}, in which the random variables $X_i$ are sign-valued and the hypergraph is a simple undirected graph, indicating that variables have only pairwise and unary interactions. The joint distribution for the \textit{Ising model} with zero external field is given by
	\begin{align}
    \p(\bx) &= \frac{1}{Z(\theta)} \exp\left\{ \sum_{\left(s,t\right)\in\mc{E}}
    \theta_{st} x_{s} x_{t} \right\} ,\quad\bx\in\{-1,1\}^{p}.
    \label{eq:Ising_model_zero_external}
	\end{align}
$\{ \theta_{st}: (s,t) \in \mc{E} \}$ are parameters of the model representing the interaction strength of the variables and $Z(\cdot) \in (0,\infty)$ is the \textit{partition function}. 
%Notice that the Ising model distribution of \eqref{eq:Ising_model_zero_external} is a sub-collection of the family in \eqref{eq:general}: the model in \eqref{eq:Ising_model_zero_external} allows only pairwise interactions between the nodes. 
%%  already stated above, redundant
These interactions are expressed through potential functions $\exp( \theta_{st}x_s x_t )$ that ensure that the Markov property holds with respect to the graph $G=\lp \mc{V},\mc{E} \rp$. %and edge parameters $\alpha \leq |\theta_{st}|\leq \beta,$ $\forall (s,t)\in \mc{E}$ and $\alpha,\beta\in\mbb{R}^{+}$. 
%For the rest of the paper, we use the notation $\mu_{ij} \defeq \E\left[X_{i}X_{j}\right]$, for any $i,j\in\mc{V}$ and $\mu_e\triangleq\E[X_i X_j]$, where $e\equiv (i,j)$ is an element in the set of edges $\mc{E}$. \footnote{It is true that $\E\left[X_{i}X_{j}\right]=\partial\ln Z\left(\theta\right)/\partial\theta_{ij}$ for any $(i,j)\in\mc{E}$.} The notation has been chosen to be consistent with prior work by~\citet{bresler2018learning}. \ads{We already define correlation earlier in the notation section -- maybe remove it there and keep this one?}
Next, we discuss the properties of distributions of the form of \eqref{eq:general}, \textit{which are Markov with respect to a tree}.

\subsection{Sign-Valued Markov Fields on Trees}

From prior work by~\citet{lauritzen1996graphical}, it is known that any distribution $\p(\cdot)$ that is Markov with respect to a tree (or forest) $\T=(\mc{V},\mc{E})$ factorizes as
	\begin{align}\label{eq:tree_shaped}
	\p(\bx)
	&= \prod_{i\in V} \p\left(x_{i}\right) 
	\prod_{ (i,j) \in \mc{E}} \frac{\p(x_{i},x_{j})}{\p(x_{i}) \p(x_{j})},\hspace{+0.3cm} \bx\in\{-1,+1\}^p,
	\end{align} and we call $\p (\cdot)$ as tree (forest) structured distribution, to indicate the factorization property.
If the distribution $\p(\cdot)$ has the form of \eqref{eq:general} with $\P (X_i=\pm 1)=1/2$, for all $i\in\mc{V}$, and is Markov with respect to a tree $\T$, then
	\begin{align}\label{eq:tree}
	\p(\bx) = \frac{1}{2}\prod_{\left(i,j\right)\in \mc{E}} 
		\frac{1+x_{i}x_{j}\E\left[X_{i}X_{j}\right]}{2}
	\end{align}
and
	\begin{align}\label{eq:CDP_1}
	\E\left[X_{i}X_{j}\right]
	=\prod_{e\in\tpath\left(i,j\right)} \mu_{e}, \quad \text{for all } i,j\in\mc{V}.
	\end{align}
(see Appendix A, Lemma~\ref{Tree_structured_Model_Lemma}). Additionally, let us state the definition of the so-called \textit{Correlation (coefficient) Decay Property (CDP)}, that will be of central importance in our analysis.
\begin{defn}
The CDP holds if and only if $|\E[X_i X_k]|\geq |\E[X_\ell X_m]|$ for all tuples $\{i,k,\ell,m\}$ $\subset \mc{V}$ such that $\tpath(i,k) \subset \tpath(\ell,m)$.
\end{defn}

The CDP is a well known attribute of acyclic Markov fields (see, e.g., \cite{tan2010learning}, \cite{bresler2020learning}). Further, it is true that the products $X_i X_j$ for all $(i,j)\in\mc{E}$ are independent and the CDP holds for every $\p(\cdot)$ of the form of \eqref{eq:general}, that factorizes with respect to a tree (see Lemma~\ref{independent_products}, Appendix A). This is a consequence of property \eqref{eq:CDP_1} and the inequality $|\mu_e|\leq 1$, for all $e\in\mc{E}$. We can interpret the CDP as a type of data processing inequality (see~\citet{cover2012elements}). The connection is clear through the relationship between the mutual information $I(X_i,X_j)$ and the correlations $\E[X_i X_j]$, namely,
	\begin{align}\label{eq:MI}
	I\left(X_{i},X_{j}\right)
	&=\frac{1}{2} \log_{2} \lp \lp 1-\E\left[X_{i}X_{j}\right]\rp^{1-\E\left[X_{i}X_{j}\right]} \lp 1+\E\left[X_{i}X_{j}\right]\rp^{1+\E\left[X_{i}X_{j}\right]}\rp,
	\end{align} 
for any pair of nodes $i,j\in\mc{V}$. This expression shows that the mutual information is a symmetric function of $\E\left[X_{i}X_{j}\right]$ and increasing with respect to $\left|\E\left[X_{i}X_{j}\right]\right|$ (see also Lemma~\ref{MI_function_of_corr}, Appendix A).

 \textbf{Tree-structured Ising models:} Despite its simple form, the Ising model has numerous useful properties. In particular, \eqref{eq:tree}, \eqref{eq:CDP_1} hold for any tree-structured Ising model with uniform marginal distributions and $\theta_r = 0$ for all $r \in \mc{V}$. Furthermore,  
\begin{align}
	\E[X_i X_j] &= \tanh\theta_{ij}, 
		\quad \forall (i,j)\in \mc{E}_{\T},
	\label{eq:Tree_Structured_Ising_Model_tanh}
	\end{align} the latter implies that
\begin{align}
	\p(\bx) &= \frac{1}{2} \prod_{\left(i,j\right)\in \mc{E}_{\T} } \frac{1+x_{i}x_{j}\tanh\theta_{ij}}{2},\quad \bx\in \{-1,1\}^{p},\quad \alpha\leq|\theta_{ij}|\leq\beta,\numberthis \label{eq:Tree_Structured_Ising_Model}\\
    \E[X_i X_j] &= \prod_{e\in\tpath\left(i,j\right)} \mu_{e}=
		\prod_{e\in\tpath\left(i,j\right)} \tanh\left(\theta_{e}\right), \quad \forall i,j\in \mc{V}.
		\label{eq:Covariates_prod}
	\end{align}
A short argument showing \eqref{eq:Tree_Structured_Ising_Model_tanh} and \eqref{eq:Tree_Structured_Ising_Model} is included in Appendix A, Lemma~\ref{Tree_structured_Model_Lemma2}.
	%\begin{align}
	%\E[X_i X_j] \overset{\eqref{eq:CDP_1}}{=} \prod_{e\in\tpath\left(i,j\right)} \mu_{e}  	
	%	&\overset{\eqref{eq:Tree_Structured_Ising_Model_tanh}}{=}
	%	\prod_{e\in\tpath\left(i,j\right)} \tanh\left(\theta_{e}\right) \quad \text{ for all } i,j\in \mc{V}.
	%	\label{eq:Covariates_prod}
	%\end{align}
%    \begin{comment}
%Under the assumption that the graph is a tree, these properties become even simpler and will allow us to characterize the structure estimation problem. For a tree-structure Ising model Lemma \ref{Tree_structured_Model_Lemma} and Lemma \ref{Tree_structured_Model_Lemma2} show that the pairwise correlation and \eqref{eq:Ising_model_zero_external} are given by:
%	\begin{align}
%	\mu_{st} &= \tanh\theta_{st} \quad \forall (s,t)\in \mc{E} \\
%	\p(\bx) &= \frac{1}{2} \prod_{\left(s,t\right)\in \mc{E} } \frac{1+x_{s}x_{t}\tanh\theta_{st}}{2},\quad \bx\in \{-1,1\}^{p} %\text{ and } \tanh\theta_{st}=\E \left[X_{s}X_{t}\right] \quad \forall (s,t)\in \mc{E}
%\label{eq:Tree_Structured_Ising_Model_moved}
%\end{align}
%and for any pair $(i,j)$ not necessarily in $\mc{E}$,
%	\begin{align}
%	\mu_{ij} = \prod_{e\in\tpath\left(i,j\right)} \mu_{e}  	
%		&\overset{\eqref{eq:Tree_Structured_Ising_Model}}{=}
%		\prod_{e\in\tpath\left(i,j\right)}
%		\tanh\left(\theta_{e}\right)
%		\label{eq:Covariates_prod}
%	\end{align}
%    \end{comment}
For the rest of the paper, we assume a tree-structured Ising model for the hidden variable $\bX$, that is, the distribution of $\bX$ has the form of \eqref{eq:tree}. We also impose a reasonable compactness assumption on the respective interaction parameters, as follows.

\begin{assumption}\label{assume:isingtree}
There exist $\alpha$ and $\beta$ such that for the distribution $\p(\cdot)$,  $0<\alpha \leq |\theta_{st}| \leq \beta < \infty$ for all $(s,t)\in\mc{E}$. 
\end{assumption}

For a fixed tree structure $\T$, and for future reference, we hereafter let $\isingtree$ be the class of Ising models satisfying Assumption \ref{assume:isingtree}.

\subsection{Hidden Sign-Valued Tree-Structured Models}\label{hidden model}

The problem considered in this paper is that of learning a tree-structured model from corrupted observations. Because we have no access to the original samples $\bX^{1:n}$, we obtain the noisy observations $\bnX^{1:n}$. To formalize this, consider a hidden Markov random field whose hidden layer $\bX$ is an Ising model with respect to a tree, i.e., $\bX \sim \p (\cdot)\in \isingtree$, as defined in \eqref{eq:Tree_Structured_Ising_Model}. The observed variables $\bnX$ are formed by setting $\nX _{r} = \N_{r} X_{r}$ for all $r \in \mc{V}$, where $\{ N_r \}$ are i.i.d. $\Rad(q)$ random variables. Let $\np(\cdot)$ be the distribution of the observed variables $\bnX$. We can think of $\bnX$ as the result of passing $\bX$ through a binary symmetric channel $\BSC$. We have the following expressions
    \begin{align}\label{eq:noise_mean}
    \E[N_{r}]&=1-2q\triangleq c_q,  \quad \forall r\in\mc{V}, \text{ and } q\in [ 0,1/2),  \\
    \nmu_{r,s} &\defeq \E\left[\nX_{r}\nX_{s}\right]=\E\left[\N_{r}X_{r} \N_{s}X_{s}\right]=\left(1-2q\right)^{2}\E\left[X_{r}X_{s}\right],\quad \forall r,s\in \mc{V}. \label{eq:BSC_corr}
    \end{align} 
The distribution $\np(\cdot)$ of $\bnX$ also has support $\{-1,+1\}^p$, and so the joint distribution satisfies the general form \eqref{eq:general}. Since the marginal distribution of each $\nX_r$ is also uniform, $\E[\nX_r]=0$ for all $r\in\mc{V}$, \eqref{eq:general} and \eqref{eq:noise_mean} yield 
\begin{align}
\np(\bnx ) & \hspace{-1pt} =\E\hspace{-1pt}\left[\prod^{p}_{i=1} \ind{\nX_{i} = y_{i}}\right] 
\hspace{-3pt} = \hspace{-1pt}\frac{1}{2^{p}} \hspace{-2pt} \left[ 1 \hspace{-1pt}+\hspace{-1pt} \sum_{k \in [p]\cap 2\mbb{N}} c_q^k  \sum_{\mc{S} \subset \mc{V}: |\mc{S}| = k} \E\hspace{-1pt}\left[ \prod_{s \in \mc{S}} X_s \right] \hspace{-1pt}\prod_{s \in \mc{S}} y_s  \right]\hspace{-1pt}\hspace{-1pt}, \,\, \bnx\hspace{-1pt}\in\hspace{-1pt} \{-1,1\}^p.\hspace{-1pt}
 \label{eq:prob2}
 \end{align} 
The moments of the hidden variables $\E\left[ \prod_{s \in \mc{S}} X_s \right]$ in \eqref{eq:prob2} can be expressed as products of the pairwise correlations $\E[X_s X_t]$, for any $(s,t)\in \mc{E}_{\T}$ (Section \ref{higher order moments - Isserlis'}, Theorem \ref{thm:binary:isserlis}). 
From \eqref{eq:prob2} it is clear that the distribution $\np(\cdot)$ of $\bnX$ does not factorize with respect to any tree, that is, $\np(\cdot) \notin \isingtree$ in general.\footnote{Lemma~\ref{G_graph} shows the structure preserving property for the observable layer holds for the special case of single-edge forests.} %The next subsections present the contribution and the most important questions which are handled in this paper.

%%%
%%% 
%%% 
\subsection{Hidden Structure Estimation}\label{Hidden_structure_estimation}

\begin{algorithm}[t]%[H]
\caption{$\mathsf{Chow-Liu}$ \label{alg:Chow-Liu}}
\begin{algorithmic}[1]
\Require $\mc{D}=\left\{ \bds{\nx}^{(1)},\bds{\nx}^{(2)},\ldots,\bds{\nx}^{(n)} \right\}\in \{-1,1\}^{p\times n}$, where $\bds{\nx}(k)$ is the $k^{\text{th}}$ observation of $\bnX$
\State Compute $ \nmue_{i,j} \gets \frac{1}{n}\sum^{n}_{k=1} \nx_i^{(k)} \nx_j^{(k)}$, for all $i,j \in \mc{V}$
%\State \textit{For} $\boldsymbol{ \mathsf{Chow-Liu-CC\!}}$ \textit{:}
%\State \hspace{12.5pt} \textbf{if} $\nmue_{i,j}>(1-2q)^2$ \textbf{then}
%\State \hspace{25pt} $\nmue_{i,j}\gets (1-2q)^2 $
%\State \hspace{12.5pt} \textbf{else if} $\nmue_{i,j}<-(1-2q)^2
%$ \textbf{then}
%\State \hspace{25pt} $\nmue_{i,j}\gets -(1-2q)^2 $
	   % \If{$\nmue_{i,j}>(1-2q)^2 \tanh(\beta)$}
	    %	\State $\nmue_{i,j}\gets (1-2q)^2 \tanh(\beta)$%\hspace{2cm}\rlap{\smash{$\left.\begin{array}{@{}c@{}}\\{}\\{}\\{}\\{}\\{}\end{array}\color{red}\right\}%
          %\color{red}\begin{tabular}{l}Clipping the correlation estimates.\\Chow-Liu-CC\end{tabular}$}}
    %	\ElsIf{$\nmue_{i,j}<-(1-2q)^2 \tanh(\beta)$}
	 %   	\State $\nmue_{i,j}\gets -(1-2q)^2 \tanh(\beta)$
	  %  \EndIf
%\State $\TCLn \gets$ MaximumSpanningTree$\lp \cup_{i\neq j} \left\{\left|\nmue_{i,j}\right|\right\} \rp$
\State \Return $\TCLn \gets$ MaximumSpanningTree$\lp \cup_{i\neq j} \left\{\left|\nmue_{i,j}\right|\right\} \rp$
\end{algorithmic}
\end{algorithm}

We are interested in characterizing the \emph{sample complexity} of structure recovery: given data generated from $\p (\cdot)\in \isingtree$ for an unknown tree $\T$, what is the minimum number $\nn$ of samples $\{\bnx^{(i)}, i\in[\nn] \}$ from $\np(\cdot)$ needed to recover the (unweighted) edge set of $\T$ with high probability? In particular, we would like to quantify how $\nn$ depends on the crossover probability $q$. Intuitively, noise makes ``weak'' edges to appear ``weaker'', and the sample complexity is expected to be an increasing function of $q$. Because the distribution $\np(\cdot)$ of the observable variables does not factorize according to any tree, this problem does not follow directly from the noiseless case. Although the classical MLE is the standard approach for the noiseless case, for the noisy setting the MLE estimation of parameters $\theta$ of the hidden model is intractable, due to the summation over the support of $\bX$. Additionally, the MLE structure estimate from noisy data is not in general consistent with the hidden structure as we explain in Section \ref{MLE}. However, for the model that we consider in this paper, the projected-MLE estimate of the observables onto the space of tree-structured models gives a consistent structure estimate. Additionally, that structure estimate is identical to the output of Chow-Liu algorithm (Algorithm \ref{alg:Chow-Liu}) from noisy data. We refer the reader to Section \ref{MLE} for the discussion about the MLE and the connection with the noisy Chow-Liu algorithm. %In this work, we use and analyze the sample complexity of the classical Chow-Liu algorithm (Algorithm \ref{alg:Chow-Liu}).

In this work, we use and analyze the sample complexity of the classical Chow-Liu algorithm (Algorithm \ref{alg:Chow-Liu}) for the following reasons: We show that given finite number of noisy data as input, the Chow-Liu algorithm recovers the original tree $\T$ with high probability. Further the sample complexity is asymptotically optimal for fixed $q<1/2$ (see Tables \ref{table:Results} and \ref{table:Results1}), and its order remains $\mc{O}(\log p)$ in the high dimensional regime. The algorithm is computationally efficient in comparison to other optimization techniques and it does not require the value $q$ to be known. Additionally, Algorithm \ref{alg:Chow-Liu} solves the projected-MLE problem that we discuss in Section \ref{MLE}. The above reasons and our finite sample complexity bound Theorems \ref{theorem: structure: simple} and \ref{thm:sufficient} suggest that Algorithm \ref{alg:Chow-Liu} is an excellent approach for tree-structure learning from noisy data.

\subsection{Evaluating the Accuracy of the Estimated Distribution}

In addition to recovering the graph structure, we are interested in the ``goodness of fit'' of the estimated distribution. Let $P_{\mc{S}},Q_{\mc{S}}$ be the marginal distributions of $P,Q$ on the set $\mc{S}\subset \mc{V}$,  let $d_{\text{TV}}$ denote the total variation distance, and fix $k=2$. We measure the error of distribution estimator through the ``small set Total Variation'' (or $\ssTV$) distance as defined by~\citet{bresler2020learning}
    \begin{align}
    \label{eq:SSTV_DEF}
    \mc{L}^{(k)}\left(P,Q\right) & \triangleq \sup_{\mc{S}:\left|\mc{S}\right|=k} d_{\text{TV}}\left(P_{\mc{S}},Q_{\mc{S}}\right).
    \end{align} If $Q$ is an estimate of $P$, the norm $\mc{L}^{(k)}$ guarantees predictive accuracy because~\citep[Section 3, page 720]{bresler2020learning}
    \begin{align}
    \label{eq:ssTV in order to make pred}
    \E_{X_{\mc{S}}}\Big[\left|P\left(X_{i}=+1|X_{\mc{S}}\right)-Q\left(X_{i}=+1|X_{\mc{S}}\right)\right|\Big]\leq2\mc{L}^{(\left|\mc{S}\right|+1)}\left(P,Q\right).
    \end{align}
%The proof of inequality \eqref{eq:ssTV in order to make pred}is given by~\citet[Section 3.2]{bresler2018learning}. 
The estimated (from noisy data) distribution of the hidden variables in \eqref{eq:estimator} is a simple extension of the noiseless estimate. In fact the estimated distribution factorizes according to the estimated from noisy data tree structure, that is the output of Algorithm \ref{alg:Chow-Liu}. Further, the pairwise correlations are normalized by the constant $(1-2q)$. As a result, the estimator is consistent because if $n\to\infty$ then $\TCLn\to \T$, $\nmue_{i,j}/(1-2q)\to \mu_{i,j}$, and as a consequence the estimate $\RIPn (\hat{\p}_{\dagger})$ convergence to the original distribution $p(\cdot)$ of $\bX$. Our main result gives a lower bound on the number of samples needed to guarantee accurate estimation (in the sense of small $\ssTV$), with high probability.

\subsection{Maximum Likelihood Estimate}\label{MLE}

A natural first place to start in estimation is the maximum-likelihood estimate (MLE). We explain why this is problematic and show a method (the projected-MLE) which turns out to be equivalent to the Chow-Liu algorithm. This motivates why we study the Chow-Liu algorithm in the first place.
%We start by explaining major problems that arise when considering the standard MLE approach for our setting, then we proceed by showing how to overcome these issues and we show the projected-MLE on trees technique which is equivalent to the Chow-Liu algorithm. 
%The following discussion gives further insight on why we consider the Chow-Liu algorithm for the problem of learning hidden trees. 
To begin, the distribution of the observables parametrized over the interaction parameters $\mathbf{\theta}$ of the hidden layer is \
    \begin{align}
    \np(\bnx) &= \sum_{\bx\in\{-1,+1\}^p}\frac{1}{Z(\theta)} \exp\left\{ \sum_{\left(s,t\right)\in\mc{E}_\mathbf{G}}
    \theta_{st} x_{s} x_{t} \right\}\p (\bnx|\bx) ,\quad\bnx\in\{-1,1\}^{p}.
    \label{eq:Ising_model_y}
	\end{align} It is known that above expression is intractable in closed form and it can be evaluated only through approximations. Secondly, the log-likelihood of $\bnX$ can be written as \begin{align}
	   \log \np(\bnx) &= \log   \sum_{\bx\in\{-1,+1\}^p} \p (\bnx|\bx) \prod_{i\in V} \p\left(x_{i}\right) 
	\prod_{ (i,j) \in \mc{E}} \frac{\p(x_{i},x_{j})}{\p(x_{i}) \p(x_{j})} ,\quad\bnx\in\{-1,1\}^{p},
	\end{align} 
	and the logarithm of the summation cannot be expressed as summation of logarithms. Therefore we see the classical MLE structure estimation approach is not applicable for hidden models. Specifically, the structure of the observable layer is a complete graph and not a tree (there is no conditional independence between $Y$'s). The maximum likelihood structure estimate with respect to the parameters $\mathbf{\theta}'$ of the observables in general will return a complete graph. Specifically, let $\mathbf{G}=(\mathbf{V},\mathbf{E}_\mathbf{G})$ be the graph (which is complete) of the observable layer, then the distribution $\np(\cdot)$ is an Ising-Model distribution and it can be written as 	\begin{align}\label{eq:Ising_model_zero_external_noisy}
    \np(\bnx) &= \frac{1}{Z'(\theta')} \exp\left\{ \sum_{\left(s,t\right)\in\mc{E}_\mathbf{G}}
    \theta'_{st} y_{s} y_{t} \right\} ,\quad\bnx\in\{-1,1\}^{p}.
	\end{align} Since all the edges exist in the edge set, none of the values $\theta'_{st}$ is zero. As a consequence, even asymptotically ($n\to \infty$) the maximum likelihood that estimates the parameters $\theta'_{st}$ gives a complete graph. Recall that we want to recover the structure of the hidden layer which is a tree. Thus, the maximum likelihood structure estimate directly applied on \eqref{eq:Ising_model_zero_external_noisy} is not consistent, because of the different hidden and observables' structure.
	
	To overcome the inconsistency that is introduced by the noise, we can project the distribution $\np(\bnx)$ to a set of tree-structured distributions and then find the maximum likelihood structure estimate. We denote the projection of $\np(\bnx)$ onto the space of trees as $\np^\mc{T} (\bnx)$ and we call the MLE with respect to $\np^\mc{T} (\bnx)$ as projected-MLE (PMLE). Then the following questions are natural: Is the PMLE always \textit{consistent} with respect to structure of the hidden layer? Is the PMLE\textit{ asymptotically optimal} ($n\to \infty$)? Is the PMLE \textit{optimal for finite values of $n$}? (by optimal we mean that the sample complexity bound matches the minimax bound). We continue by answering the questions above. First we present the structural consistency and then we continue by discussing the asymptotic optimality and optimality for finite $n$. 
	
	 Although, the PMLE is not in general consistent with structure of the hidden layer (see also related work by~\cite{nikolakakis2020info}), for the setting of the BSC channel with i.i.d noise we do have $\hat{\T}_{\text{PMLE}}\to \T$ when $n\to \infty$. In fact, the projected distribution as $\np^\mc{T} (\bnx)$ is given by \begin{align}
	 \label{eq:projected_p}
    \np^\mc{T} (\bnx) \triangleq \argmin_{Q(\cdot)\in\isingtree} \DKL(\np(\bnx) ||\Q (\bnx)).
\end{align} The proof of the claim follows by a standard argument (see also Lemma 1 and Lemma 2 by~\cite[Supplemetary material, Appendix A]{bresler2020learning}) and it gives \begin{align}
     \DKL(\np(\bnx) ||\np^\mc{T}) =1 -H(\np(\bnx)) + \sum_{(i,j)\in\mc{E}} H_{\mathrm{B}}\lp \frac{1+(1-2q)^2\mu_{i,j}}{2}\rp.
\end{align} As a consequence the projected-MLE $\hat{\T}_{\text{PMLE}}$ is  \begin{align}
  \hat{\T}_{\text{PMLE}}   =\argmin_{\T\in\mc{T}} \sum_{(i,j)\in\mc{E}_{\T}} H_{\mathrm{B}}\lp \frac{1+\nmue_{i,j}}{2}\rp\equiv \TCLn,
\end{align} 
%\begin{align}
%    \TCL =\argmin_{\T\in\mc{T}} \sum_{(i,j)\in\mc{E}_{\T}} H_{\mathrm{B}}\lp \frac{1+\hat{\mu}_{i,j}}{2}\rp.
%\end{align} 

\noindent and the following  \begin{align}
   \argmin_{\T\in\mc{T}} \sum_{(i,j)\in\mc{E}_{\T}} H_{\mathrm{B}}\lp \frac{1+(1-2q)^2 \mu_{i,j}}{2}\rp \equiv \argmin_{\T\in\mc{T}} \sum_{(i,j)\in\mc{E}_{\T}} H_{\mathrm{B}}\lp \frac{1+\mu_{i,j}}{2}\rp
\end{align} gives that $\hat{\T}_{\text{PMLE}}\equiv \TCLn\to \T$ (almost surely) when $n \to \infty$. Although, the above discussion of the consistency for $n\to \infty$ shows the connection with MLE, our results for instance Theorem 3.1 shows that the Chow-Liu algorithm returns the original tree for finite $n$ with probability $1-\delta$.

     Additionally, the PMLE is asymptotically optimal, however for finite $n$ it may be not optimal. For our structure/predictive learning problem our bounds are asymptotically optimal (up to constants). That is, for fixed $q$ the upper and lower bounds match as $n\to\infty.$ Nevertheless for finite $n$ the PMLE is not optimal in general. It is known that under the presence of noise the MLE approach may be non-robust and sub-optimal and extra steps should be considered including pre-processing, statistical learning of the noise by using pilot samples, and detecting and rejecting bad samples (for further information see also \citet[page 62]{zoubir2012robust} and~\citet{nikolakakis2020info}). 
     %As an example of known case for MSE estimator is the James Stein's estimator~\citep{james1992estimation}, because of the Stein phenomenon the JS estimator outperforms the MLE estimator for finite $n$, while for $n\to \infty$ the performance of both is identical.
     % ADS: I TOOK THE ABOVE OUT BECAUSE WE ARE TALKING ABOUT DISCRETE DISTRIBUTIONS.
     The reason that we consider Chow-Liu algorithm in our work is that it is computationally efficient, while its sample complexity remains logarithmic with respect to $p$ even when noise exists. The latter makes the Chow-Liu algorithm useful in practice when only noisy observations are available. Finally, to give further insight about the gap between the upper and lower bounds we present an example in Section \ref{GAP_discussion} (Appendix), for which perfect denoising is possible for $p\to\infty$ before running the Chow-Liu algorithm. As consequence, for $p\to\infty $ the bounds in Propositions \ref{prop_structure} and \ref{proposition_predictive} reduce to the noiseless case as they should. This example is a marginal case (since perfect denoising is not possible in general) and it affects our converse results which are universal and owe to include corner cases.

\section{Main Results}\label{Main Results}

 The main question asked by this paper is as follows: \textit{what is the impact of noise on the sample complexity of learning a tree-structured graphical model in order to make predictions}? This corresponds to sampling variables $\bnX$ generated by sampling $\bX$ from the model \eqref{eq:Ising_model_zero_external} and randomly flipping each sign independently with probability $q$. We use the Chow-Liu algorithm to estimate the hidden structure using the noise-corrupted samples. 
 %Hidden (binary) model structures arise in a variety of applications. A typical example is classification where a subset of the data is misclassified. In such a case, corrupted data are observed, however, we are still able to retrieve the underlying structure by considering the appropriate number of samples.  ADS: removed since it should be in the intro.
We first find upper (Theorem \ref{thm:sufficient}) and lower bounds (Theorem \ref{thm:necessary}) on the sample complexity for exact hidden structure recovery using the Chow-Liu algorithm on noisy observations.

Secondly, we use the structure statistic to derive an accurate estimate of the hidden layer's probability distribution. The distribution estimate is computed to be accurate under the $\ssTV$ utility measure, that was introduced by~\citet{bresler2020learning}. Furthermore, the estimator of the distribution factorizes with respect to the structure estimate, while the $\ssTV$ metric ensures that the estimated distribution is a trustworthy predictor.  Theorem \ref{thm:Main_result} and Theorem \ref{thm:fano's inequality theorem} give the sufficient and necessary sample complexity for accurate distribution estimation from noisy samples. These theorems generalize the results for the noiseless case ($q=0$) by~\citet{bresler2020learning} and lead to interesting connections between structure learning on hidden models and data processing inequalities~\citep{raginsky2016strong,polyanskiy2017strong}. 

The third part of the results includes Theorem \ref{thm:binary:isserlis}, which gives an equivalent of Isserlis' theorem by providing closed form expressions for higher order moments of sign-valued Markov fields on trees. Based on Theorem \ref{thm:binary:isserlis} we  propose a low complexity algorithm to estimate any higher order moment of the hidden variables given the estimated tree structure and estimates of the pairwise correlations (both evaluated from observations corrupted by noise). 

Finally, Theorem \ref{thm:SKLTheorem} gives the sufficient number of samples for distribution estimation, when the symmetric KL divergence is considered as utility measure. These give rise to extensions of testing algorithms~\cite{daskalakis2018testing} under a hidden model setting.

\subsection{Tree Structure Learning from Noisy Observations}\label{1st_part_results}

Our goal is to learn the tree structure $\T$ of an Ising model
with parameters $|\theta_{st}|\in[\alpha,\beta]$, when the nodes $X_{i}$ are hidden variables and we observe $\nX_{i}\triangleq \N_{i}X_{i}$, $i\in\mc{V}$, where $N_{i}\sim\Rad(q)$ are i.i.d, for all $i\in\mc{V}$ and for all $q\in [ 0, 1/2 )$. We derive the estimated structure $\TCLn$ by applying the Chow-Liu algorithm (Algorithm \ref{alg:Chow-Liu})~\citep{chow1968approximating}. 

 Instead of mutual information estimates, our Chow-Liu algorithm (Algorithm \ref{alg:Chow-Liu}) requires correlation estimates; these are sufficient statistics because of \eqref{eq:MI}. Further, it can consistently recover the hidden structure through noisy observations. The latter is true because of the \textit{order preserving} property of the mutual information. That is, the stochastic mapping $\bX\xrightarrow[]{\BSC}\bnX$ allows structure recovery of $\bX$ by observing $\bnX$, because for any tuple $X_i,X_j,X_{i'},X_{j'}$ such that $I\lp X_i;X_j\rp\leq I\lp X_{i'},X_{j'}\rp$, it is true that $I\lp Y_i;Y_j\rp\leq I\lp Y_{i'},Y_{j'}\rp$. The proof directly comes from \eqref{eq:MI} and \eqref{eq:BSC_corr}. In addition, the monotonicity of mutual information with respect to the absolute values of correlations allows us to apply the Chow-Liu algorithm directly on the estimated correlations $\nmue_{i,j}\triangleq1/\nn \sum^{\nn}_{k=1} \lp\nX_i\rp^{(k)} (\nX_j )^{(k)}$. Notice that because of \eqref{eq:BSC_corr}, $\nmue_{i,j}$ can be used as an alternative of $\hat{\mu}_{i,j}$. The algorithm returns the maximum spanning tree $\TCLn$. Further discussion about the Chow-Liu algorithm is given in Section \ref{The Chow-Liu algorithm}. The following theorem provides the sufficient number of samples for exact structure recovery through noisy observations.

\begin{theorem}[Sufficient number of samples for structure learning] \label{thm:sufficient}
%If the weakest edge satisfies the following property 
%\begin{align*}
%\tanh\alpha & \geq \ntau\triangleq \frac{4\neps\sqrt{1-\left(1-2q\right)^{4}\tanh\beta}}{\left(1-%2q\right)^{2}\left(1-\tanh\beta\right)},\text{ where }\neps=\sqrt{ \frac{2\log\left(2p^{2}/\delta\right)}{\nn}},
%\end{align*} or equivalently if 
Let $\bnX$ be the output of a $\mathrm{BSC}(q)^p$, with input variable $\bX\sim\p(\cdot)\in\isingtree$. Fix a number $\delta\in (0,1)$. If the number of samples $\nn$ of $\bnX$ satisfies the inequality
\begin{align}\label{eq:sufficient_number_of_samples_noise_thm}
\nn\geq & \frac{32\left[1-\left(1-2q\right)^{4}\tanh\beta\right]}{\left(1-2q\right)^{4}\left(1-\tanh\beta\right)^{2}\tanh^{2}\alpha}\log\frac{2p^{2}}{\delta},
\end{align} then Algorithm \ref{alg:Chow-Liu} returns $\TCLn=\T$ with probability at least $1-\delta$.
\end{theorem} 

Theorem \ref{thm:sufficient} characterizes the finite-sample performance of the Chow-Liu estimator and by taking $n \to \infty$ we can see that Algorithm \ref{alg:Chow-Liu} is consistent in the noisy setting. As a consequence of \eqref{eq:sufficient_number_of_samples_noise_thm} and the inequality $1-\tanh(\beta)\geq e^{-2\beta}$, if the number of samples satisfies the following bound \begin{align}\label{eq:sf_inter}
   n>C\frac{e^{2\beta}}{\tanh^2 (\alpha)}\left[\mathds{1}_{q=0} +  e^{2\beta} \lp  (1-2q)^{-4}- \tanh (\beta) \rp\mathds{1}_{q\neq 0}  \right] \log (p /\delta),
\end{align} then the structure is exactly recovered with probability at least $1-\delta$. The latter gives the statement of Theorem \ref{theorem: structure: simple}.
%%%%%%%%%%%%%%%%%%%%%%%%%%%%%%%%%%%%%%%%%%%%%%%%%%%%%%%%%%%%%%%%%%%%%%%%%%%%%%%%%%%%%%%%%%%%%%%%%%%%%%%%%%%%%%%
%\kn{Alternative ways to express the bound above: It is true that $\frac{1}{1-\tanh (\beta)}<e^{2\beta}$ which gives}\begin{align}\color{blue}
%\nn\geq & \color{blue}\frac{32e^{2\beta}\left[\left(1-2q\right)^{-4}-\tanh\beta\right]}{\left(1-\tanh\b%eta\right)\tanh^{2}\alpha}\log\frac{2p^{2}}{\delta}\quad \text{ or } \nn\geq  %\color{blue}\frac{32e^{4\beta}\left[\left(1-2q\right)^{-4}-\tanh\beta\right]}{\tanh^{2}\alpha}\log\frac%{2p^{2}}{\delta}
%\end{align}
%%%%%%%%%%%%%%%%%%%%%%%%%%%%%%%%%%%%%%%%%%%%%%%%%%%%%%%%%%%%%%%%%%%%%%%%%%%%%%%%%%%%%%%%%%%%%%%%%%%%%%%%%%%%%%%

 Complementary to Theorem \ref{thm:sufficient}, our next result characterizes the necessary number of samples required for exact structure recovery. Specifically, we prove a lower bound on the sample complexity that characterizes the necessary number of samples for any estimator $\psi$.  

\begin{theorem}[Necessary number of samples for structure learning]\label{thm:necessary}
Let $\bnX$ be the output of a $\mathrm{BSC}(q)^p$, with input variable $\bX\sim\p(\cdot)\in\isingtree$. If the given number of samples of $\bnX$ satisfies the inequality \begin{align}\label{eq:suff_n_complexity}
\nn< \frac{[1-(4q(1-q))^p]^{-1}}{16\alpha\tanh(\alpha)}e^{2\beta}\log \lp p\rp, 
\end{align} then for any estimator $\psi$, it is true that\begin{align}
\inf_{\psi} \sup_{\substack{\T\in\mc{T} \\ \p(\cdot)\in \isingtree}} \P \lp \psi \lp \bnX_{1:\nn} \rp \neq \T
\rp>\frac{1}{2}.
\end{align}
\end{theorem}

It can be shown that the right hand-side of \eqref{eq:sufficient_number_of_samples_noise_thm} is greater than the right-hand side of \eqref{eq:suff_n_complexity} for any $q$ in $[0,1/2)$ (and for all possible values of $p,\beta,\alpha$), by simply comparing the two terms. Theorems \ref{thm:sufficient} and \ref{thm:necessary} reduce to the noiseless setting by setting $q=0$ (\citet{bresler2020learning}). The sample complexity is increasing with respect to $q$, and structure learning is always feasible as long as $q\neq 1/2$. Let $n$ denote the required samples under a noiseless setting assumption, then for a fixed probability of exact recovery, we always need $\nn\geq n$ because \begin{align}
\frac{\left[1- (1-2q)^4 \tanh (\beta)\right]}{\left[(1-2q)^4 (1-\tanh(\beta))\right]}\geq 1, \quad \forall q\in \Big[0,\frac{1}{2}\Big) \text{ and } \beta\in \mbb{R}.
\end{align}Furthermore,
    \begin{align}
    \label{eq:SDPI BSC}
    \frac{1}{1-(4q(1-q))^p}\geq 1,
        \quad \forall q\in [0,1/2)
            \text{ and } p\in\mbb{N},
    \end{align} 
the latter shows that the sample complexity in a hidden model is greater than the noiseless case ($q=0$), for any measurable estimator (Theorem \ref{thm:necessary}). When $q$ approaches $1/2$, the sample complexity approaches infinity, $\nn\to \infty$, and the structure learning is impossible. Theorem \ref{thm:necessary} extends Theorem 3.1 by~\citet{bresler2020learning} to our hidden model. Our results combines Bresler's and Karzand's method and a strong data processing inequality (SDPI) by~\citet[Evaluation of the BSC]{polyanskiy2017strong}. Upper bounds on the symmetric KL divergence for the output distribution $\np(\cdot)$ can not be found in a closed form. However, by using the SDPI, we manage to capture the dependence of the bound on the parameters $\alpha,\beta,q$ and derive a non-trivial result. When $p\to\infty$, the bound becomes trivial since $\lim_{p\to\infty} 1/\left[1-(4q(1-q))^p\right]\to 1$, giving the classical data processing inequality (contraction of KL divergence for finite alphabets,~\citep{raginsky2016strong,polyanskiy2017strong}). While direct application of the SDPI is simple and provides an upper bound which is almost insensitive to $q$ (for sufficiently large $p$), it introduces a gap between the lower and upper bounds. Nevertheless, it is important because it indicates a possible non-optimal performance of the classical Chow-Liu algorithm (under a hidden model). We conjecture that the sample complexity bounds in (Theorem \ref{thm:necessary} and Theorem \ref{thm:fano's inequality theorem}) are tight only under the low temperature regime $|\theta_{i,j}|=\beta\to\infty$ for all $i,j\in\mc{E}$, while in general ($\theta_{i,j}\in [\alpha,\beta]$) the inequalities hold but they are not tight. For further explanation related to the gap between the upper and lower bounds see Section \ref{MLE}. %Next we show that both the insensitivity with respect to $q$ and the gap between the two bounds are absolutely necessary, since Theorem \ref{thm:necessary} provides the necessary number of samples for any algorithm. In the follow discussion we show that it is also tight for a larger class of algorithms, under at least one regime of tree structured models. 
The latter is a consequence of the SDPI, which is tight for the repetition code~\citep[Evaluation for the BSC, page 12]{polyanskiy2017strong}. We performed extensive simulations (c.f. Figures \ref{fig:sfig1}, \ref{fig:sfig2}) that suggests that our bound does indeed accurately characterize the performance of Chow-Liu. These simulations choose $p=100$, but our evidence shows that the dependence on $q$ is not affected for larger ($p = 200$) or smaller ($p = 50$) values of $q$. We believe that the term $1/[(1-(4q(1-q))^p]$ does not characterize the Chow-Liu algorithm, but possibly a more complicated algorithm.

\subsection{Predictive Learning from Noisy Observations}

In addition to recovering the structure of the hidden Ising model, we are interested in estimating the distribution $\p(\cdot)  \in \isingtree$ itself. If the $\Lnorm$ distance between the estimator and the true distribution is sufficiently small, then the estimated distribution is appropriate for predictive learning because of \eqref{eq:ssTV in order to make pred}. For consistency, this distribution should factorize according to the structure estimate $\TCLn$ and for the predictive learning part, the estimate $\TCLn$ is considered the output of the Chow-Liu algorithm (see Algorithm \ref{alg:Chow-Liu}). We continue by defining the distribution estimator of $\p(\cdot)$ as
\begin{align}
\label{eq:distribution_estimator}
\RIPn (\hat{\p}_{\dagger} ) \triangleq \frac{1}{2} \prod_{\lp i,j\rp\in \mc{E}_{\TCLn}} \frac{1+x_{i}x_{j}\frac{\nmue_{i,j}}{(1-2q)^2}}{2}.
\end{align}
The estimator \eqref{eq:distribution_estimator} can be defined for any $q\in[0,1/2)$. For $q=0$ it reduces to that in the noiseless case, since $\TCLn\equiv \TCL$, $\nmue_{i,j}\equiv \hat{\mu}_{i,j}$, and thus $\RIPn  \big( \nPe \big)\equiv \RIP \big(\hat{P}\big)$. It is also closely related to the reverse information projection onto the tree-structured Ising models~\citep[supplementary material, Appendix A]{bresler2020learning}, in the sense that 
    \begin{align}\label{eq:RIP}
    \Pi{}_{\T} (P) =
    \argmin_{Q\in\mc{P}_{\T}(\alpha,\beta)}
    \KL \lp P||Q \rp, \quad P\in \mc{P}_{\T}(\alpha,\beta).
\end{align} To compute $\RIPn (\hat{p}_{\dagger} )$, two sufficient statistics are required: the structure $\TCLn$ and the set of second order moments~\citep{chow1968approximating,bresler2020learning}, under the assumption that $q$ is known. The next result provides a sufficient condition on the number of samples to guarantee that the $\Lnorm$ distance between the true distribution and the estimated distribution is small with probability at least $1-\delta$. 

% \begin{comment}
% To evaluate the accuracy of the estimator $\RIPn (\hat{p}_{\dagger} )$ we consider the small set TV metric (ssTV) which was introduced in prior work by~\citet{bresler2018learning} as\kn{duplicate}\begin{align}\label{eq:SSTV_DEF}
% \mc{L}^{(k)}\left(P,Q\right) & \triangleq \sup_{\mc{S}:\left|\mc{S}\right|=k} d_{\text{TV}}\left(P_{\mc{S}},Q_{\mc{S}}\right),
% \end{align} where $P_{\mc{S}},Q_{\mc{S}}$ are the marginal distributions of $P,Q$ on the set  $\mc{S}\subset \mc{V}$, $d_{\text{TV}}$ is the total variation and $k=2$. %Further details about the estimator $\RIPn (\hat{p}_{\dagger} )$ and the ssTV matric are given in Section \ref{measures_of_perfomance}.  
% The norm $\mc{L}^{(k)}$ guarantees accurate estimates in order to make predictions under the sense that
% \begin{align}\label{eq:ssTV in order to make pred}
% \E_{X_{\mc{S}}}\Big|P\left(X_{i}=+1|X_{\mc{S}}\right)-Q\left(X_{i}=+1|X_{\mc{S}}\right)\Big|\leq2\mc{L}^{(\left|\mc{S}\right|+1)}\left(P,Q\right).
% \end{align}The proof of inequality \eqref{eq:ssTV in order to make pred} is given by~\citet[Section 3.2]{bresler2018learning}.
% \end{comment}

Note that the dependence on $\beta$ changes from $e^{2\beta}$ to $e^{4\beta}$ when the data are noisy $q>0$, while for $q=0$ our bound exactly recovers the noiseless case~\citep{bresler2020learning}. A key component of the bound is the following function \begin{align}
    \Gamma(\beta,q)\triangleq \lp \frac{1-(1-2q)^2}{1-(1-2q)^4\tanh^2 (\beta)}\rp^2,\quad  \beta>0 \text{ and } q\in [0,1/2).
\end{align} Further, notice that $\Gamma(\beta,q)\in [0,1]$ for all $\beta>0$ and $q\in [0,1/2)$, and $\Gamma(\beta,0)=0$ for all $\beta>0$. Additionally, we define the functions\begin{align}
\hspace{-1cm} K(\beta,q)&\triangleq \frac{10(1-\tanh^2 (\beta))}{9+(1-2q)^2-\tanh^2 (\beta) (1-2q)^2(9(1-2q)^2+1)},\label{eq:K()_def1}%\geq   e^{-2\beta\mathds{1}_{q= 0}} \triangleq K 
\end{align} and
\begin{align}
    B (\beta,q)\triangleq \max\left\{\frac{1}{K(\beta,q)},\left(1+ 2e^{\beta}\sqrt{2\left(1-q\right)q\tanh\beta}\right)^2\right\}.%\label{eq:B_def1}
\end{align} The latter constitute additional components of the main rusult that follows.
\begin{theorem} \label{thm:Main_result} Fix $\delta\in (0,1)$ and choose $\eta>0$. If

\begin{align}
    n\geq \max \left\{\frac{512}{ \eta^2(1-2q)^4}, \frac{1152 e^{2\beta}B(\beta,q)}{(1-2q)^4} , \frac{48 e^{4\beta}}{\eta^2}\Gamma(\beta,q) \right\}\log\lp\frac{6p^3}{\delta}\rp\label{eq:main_result_bound}
\end{align} then \begin{align}
	\P \lp \Lnorm \lp \p (\cdot),\RIPn (\hat{p}_{\dagger} ) \rp \leq \eta \rp \geq 1-\delta.
	\end{align} 
\end{theorem} \eqref{eq:main_result_bound} and the inequalities $\Gamma (\beta,q)\leq \mathds{1}_{q\neq 0}$, $B(\beta,q)\leq \lp1+3\sqrt{q}\rp^2e^{2\beta\mathds{1}_{q\neq 0}}$ give Theorem \ref{thm:Main_result:short}. We provide the proof of Theorem \ref{thm:Main_result:short} and Theorem \ref{thm:Main_result} in Section \ref{Section:Predictive_Learning_Proof} (Appendix). As we mentioned in Section \ref{Predictive_Learning_Intro}, the sample complexity for accurate predictive learning does not depend on $\alpha$, that is, even in the high temperature regime $\alpha\to 0$ (and in contrast with the structure learning), the number of required samples does not increase.

 Conversely, the following result provides the necessary number of samples for small $\Lnorm$ distance by a minimax bound, that characterizes any possible estimator $\psi$. In other words, it provides the necessary number of samples required for accurate distribution estimation, appropriate for predictive learning (small $\mc{L}^{(2)}(\cdot)$).

\begin{theorem}[Necessary number of samples for inference]\label{thm:fano's inequality theorem}
Fix a number $\delta\in (0,1)$. Choose $\eta>0$ such that $\tanh(\alpha)+2\eta < \tanh(\beta)$. If the given number of samples satisfies the inequality   \begin{align}
\nn<\frac{1-\left[ \tanh(\alpha) +2\eta\right]^2}{16\eta^2[1-(4q(1-q))^p]}\log p,
\end{align} then for any algorithm $\psi$, it is true that
\begin{align*}
\inf_{\psi} \sup_{\substack{\T\in\mc{T} \\ \p(\cdot) \in \isingtree}} \P \lp \mc{L}^2 \lp \p (\cdot), \psi \lp \bnX_{1:n} \rp\rp >\eta  \rp>\frac{1}{2}.
\end{align*}
\end{theorem}

 Theorems \ref{thm:Main_result} and \ref{thm:fano's inequality theorem} reduce to the noiseless setting for $q=0$, that has been studied earlier by \citet{bresler2020learning}. Similarly to our structure learning results, presented previously (Theorem \ref{thm:sufficient}, Theorem \ref{thm:necessary}), when $q\to 1/2 $ we have $\nn\to\infty$, the latter indicates that the learning task becomes impossible for $q=1/2$.

\begin{remark}
Theorem \ref{thm:fano's inequality theorem} requires the assumption $\alpha<\beta$. The special case $\alpha=\beta$ can be derived by applying the same proof technique of Theorem \ref{thm:fano's inequality theorem} combined with Theorem 3 by~\cite[supplementary material]{bresler2020learning} and the SDPI by~\cite{polyanskiy2017strong}.
\end{remark}

\noindent Further details and proof sketches of Theorems \ref{thm:Main_result} and \ref{thm:fano's inequality theorem} are provided in Section \ref{necessary events discussion}.

%%%
%%% ESTIMATING HIGHER ORDER MOMENTS
%%%
\subsection{Estimating Higher Order Moments of Signed-Valued Trees}
\label{higher order moments - Isserlis'}

A collection of moments is sufficient to represent completely any probability mass function. For many distributions, the first and second order moments are sufficient statistics; this is true, for instance, for the \textit{Gaussian distribution} or the \textit{Ising model} with unitary and pairwise interactions. Even further, in the Gaussian case, the well-known Isserlis' Theorem (\citet{isserlis1918formula}) gives a closed form expression for all moments of every order.  As part of this work, we derive the corresponding moment expressions, \textit{for any tree-structured Ising model}. To derive the expression of higher order moments, we first prove a key property of tree structures: for any tree structure $\T=(\mc{V},\mc{E})$ and a even-sized set of nodes $\mc{V'}\subset \mc{V}$, we can partition $\mc{V'}$ into $|\mc{V'}|/2$ pairs of nodes, such that the path along any pair is disjoint with the path of any other pair (see Appendix A,  Lemma~\ref{disjoint_paths}). We denote as $\mc{C}_{\T}(\mc{V'})$ the set of distinct $|\mc{V'}|/2$ pairs of nodes in $\mc{V'}$, such that $ \tpath(u,u')\cap\tpath(w,w')=\emptyset$, for all $\{u,u'\},\{w,w'\}\in \mc{C}_{\T}(\mc{V'} \}$. Let $\mc{CP}_{\T}(\mc{V'})$ be the set of all edges in all mutually edge-disjoint paths with endpoints the pairs of nodes in $\mc{V'}$, that is,
	\begin{align}\label{eq:CPT}
	\mc{CP}_{\T}(\mc{V'}) \triangleq \bigcup_{\left\{w,w'\right\}\in \mc{C}_{\T}(\mc{V'})} \tpath_{\T}(w,w').
\end{align} 
\begin{algorithm}[t]%[H]
\caption{Matching Pairs \label{alg:matching_pairs} }
\begin{algorithmic}[1]
\Require Tree structure $\T=(\mc{V},\mc{E})$, any set $\mc{V'\subset \mc{V}}:|\mc{V'}|\in 2\mbb{N}$  
\State $\mc{CP}_{\T}\gets \emptyset$ 
\For{$i \in \mc{V}$}
	\If {$i\in \mc{V'}$}
        \State $p(i)\gets 1$ %\gets i
    \Else
    	 \State $p(i)\gets 0$ %\gets \emptyset
    \EndIf     
\EndFor
\For{$k\in [d]$} \algorithmiccomment{$d$ is the depth of the tree}
    \State $\text{Store all nodes at level $k$ to } L(k)$ %$\text{to contain all vertices in } \mc{V} \text{ with depth } k \text{ in the tree.}$
\EndFor
\For{$k\in [d]$}
	\For{$i\in L(d+1-k)$} \algorithmiccomment{Visit each of the nodes at level $d+1-k$}
        \If {$p(i) = 1$} %\neq \emptyset
            \State $\mc{V'}\gets \mc{V'}\setminus\{i\}$
        	\State $\mc{CP}_{\T}\gets \mc{CP}_{\T}\cup (i,\text{ancestor}(i))$
            \If {$p(\text{ancestor(i)}) = 1$} %\neq \emptyset
           	 \State $\mc{V'}\gets \mc{V'}\setminus\{\text{ancestor}(i)\}$
           	 \State $p(\text{ancestor}(i))\gets 0$ %\emptyset
            \Else
       		 \State $p(\text{ancestor}(i))\gets 1$ %\gets p(i)
          \EndIf
        \EndIf
        \If {$\mc{V'}\equiv\emptyset$}
        	\State \Return $\mc{CP}_{\T}$           
        \EndIf
    \EndFor
\EndFor
%\State \Return The collection of pairs of vertices $\mc{C}_{\T}(\mc{V'})$
\end{algorithmic}
\end{algorithm}For any tree $\T$, the set $\mc{CP}_{\T}(\mc{V'})$ can be computed via the Matching Pairs algorithm, Algorithm \ref{alg:matching_pairs}. By using the notation above, we can now present the equivalent of Isserlis' Theorem. The closed form expression of moments is given by the next theorem.
\begin{theorem} 
	\label{thm:binary:isserlis}
For any distribution of the form of \eqref{eq:tree_shaped}, which factorizes according to a tree $\T$ and has support $\{-1,+1\}^p$, it is true that \begin{align}\label{eq:binary:isserlis}
	\E\left[X_{i_1}X_{i_2}\ldots X_{i_k}\right]
    	&= 
	\begin{cases}
		0 & k \ \textrm{odd} 
	\\
    		\prod_{e\in \mc{CP}_{\T}(i_1,i_2,\ldots,i_k)}\mu_e
			& k \ \textrm{even} .
	\end{cases}
	\end{align}
\end{theorem}

 Theorem \ref{thm:binary:isserlis} is an equivalent of Isserlis' theorem for tree-structured sign-valued distributions. Equation \eqref{eq:binary:isserlis} is used later to define an estimator of higher order moments that requires two sufficient statistics: the estimated structure $\TCLn$ and the correlation estimates $\nmue_e$, for any $e\in \TCLn$. Together with the parameter $q$, the higher order moments completely characterize the distribution of the noisy variables of the hidden model \eqref{eq:prob2}. We provide the proof of Theorem \ref{thm:binary:isserlis} in Appendix \ref{App.A}.

A similar expression to \eqref{eq:binary:isserlis} has been introduced in prior work. Specifically, Algorithm \ref{alg:matching_pairs} solves the problem of finding the optimal matching, see Definition 1 by~\cite[supplementary material]{bresler2020learning}.  The evaluation of higher order moments requires an explicit expression or a way to compute the set $\mc{CP}_{\T}$. For a given tree $\T=(\mc{E},\mc{V})$ and a set $((i_1,i_2,\ldots,i_k)\subset \mc{V}$, there is a unique set $\mc{CP}_{\T}(i_1,i_2,\ldots,i_k)$ (see Appendix \ref{App.A}, proof of Theorem \ref{thm:binary:isserlis}). Given a set of edges $\mc{E}$, we show that the set $\mc{CP}_{\T}$ can be evaluated by running a matching pair algorithm. For that purpose, we provide Algorithm \ref{alg:matching_pairs}  (with complexity $(\mc{O}(\mc{E}))$) and we prove its consistency (See Appendix, Lemma \ref{disjoint_paths}). The latter yields to an explicit expression of higher order moments; the Theorem \ref{thm:binary:isserlis}. Furthermore, it provides a concrete higher order moments estimator, that is based on the estimated structure $\TCL$ (or $\TCLn$) and the set of estimated correlations $\{\hat{\mu}_e:e\in \mc{CP}_{\TCL}\}$.
%in Theorem \ref{thm:binary:isserlis} consists of \begin{itemize}
 %   \item a proof for the existence and uniqueness of the set $\mc{CP}_{\T}$ (See Appendix, Lemma \ref{disjoint_paths})
%    \item we provide a low complexity algorithm (Algorithm \ref{alg:matching_pairs}) which computes the set $\mc{CP}_{\T}$ and we show consistency of the algorithm

 \textbf{High Order Moments Estimator:} A higher order moment is the expected value of the product of the hidden tree-structured Ising model variables $\{ X_{i}: i \in \mc{V}' \}$ where $\mc{V}' \subset \mc{V}$. Theorem \ref{thm:binary:isserlis} gives the closed form solution for such moments. We have the following estimator for higher order moments using only noisy observations and known $q$. In particular, we have
\begin{align}
\hat{\E}\left[X_{i_1}X_{i_2}\ldots X_{i_k}\right]
    	&\equiv 0,
		\quad k\in 2\mbb{N}+1,
	\\
\hat{\E}\left[X_{i_1}X_{i_2}\ldots X_{i_k}\right] 
    	&\triangleq \prod_{e\in \mc{CP}_{\TCLn}(i_1,i_2,\ldots,i_k)}\frac{\nmue_e}{(1-2q)^2},
		\quad  k\in 2\mbb{N}\label{eq:second_order_monets_estimate}.
\end{align} 
The estimated structure and pairwise correlations are sufficient statistics: given those, \eqref{eq:second_order_monets_estimate} suggests a computationally efficient estimator for higher order moments. First we run the classical Chow-Liu algorithm to estimate the tree structure $\TCLn$, and then we run Algorithm \ref{alg:matching_pairs} with input the estimate $\TCLn$ to evaluate the set $\mc{CP}_{\TCLn}$. Thus, by estimating $\TCLn$, $\mc{CP}_{\TCLn}$ and $\nmue_e$ for any $e\in \TCLn$, we can in turn estimate any higher order moment through \eqref{eq:second_order_monets_estimate}. Considering the absolute estimation error, we have
%$\left|\hat{\E}\left[X_{i_1}X_{i_2}\ldots X_{i_k}\right]-\E\left[X_{i_1}X_{i_2}\ldots X_{i_k}\right] \right|$, and we have
\begin{align}\label{eq:accuracy_HOM}
 \left|\hat{\E}\left[\prod_{s\in \mc{V}'} X_s \right]-\E\left[\prod_{s\in \mc{V}'} X_s\right] \right|\leq  2|\mc{V'}| \Lnorm \lp \p (\cdot),\RIPn  \big( \nPe \big) \rp.
\end{align} 

Theorem \ref{thm:Main_result} guarantees small $\ssTV$ and in combination with \eqref{eq:accuracy_HOM} gives an upper bound on the higher order moment estimate \eqref{eq:second_order_monets_estimate}. In Section \ref{higher_order_moments}, we provide further details and discussion about Theorem \ref{thm:binary:isserlis}, Algorithm \ref{alg:matching_pairs}, that computes the sets $\mc{CP}_{\T}(\mc{V'}),\mc{CP}_{\TCLn}(\mc{V'})$, and the bound on the error of estimation \eqref{eq:accuracy_HOM}. 

So far we have studied the consistency of the estimator with respect to the $\Lnorm$ metric. We are also interested in sample complexity bounds for $\phi$-divergences. While general divergences may be challenging, the most widely-used is the  KL-divergence, particularly in testing Ising models~\citep{daskalakis2018testing}. The next result gives a bound for the sufficient number of samples to guarantee a small \emph{symmetric} KL divergence $\SKL(P||Q)\triangleq \KL(P||Q)+\KL(Q||P)$ with high probability. For any Ising model distributions $P,Q$ of the form \eqref{eq:Ising_model_zero_external} with respective interaction parameters $\bds{\theta},\bds{\theta}'$, we have\begin{align}\label{eq:SKL}
 \SKL\left(\bds{\theta}||\bds{\theta}'\right) \triangleq \SKL(P||Q) & =\sum_{s,t\in\mc{E}} \left(\theta_{st}-\theta'_{st}\right)\left(\mu_{st}-\mu'_{st}\right).%\text{ where }\mc{E}=\mc{E}_{\boldsymbol{\theta}}\cup\mc{E}_{\boldsymbol{\theta}'}.
\end{align}

\begin{theorem}[Upper Bounds for the Symmetric KL Divergence]\label{thm:SKLTheorem}
%\begin{comment}
%$n>2\log\left(2p^{2}/\delta\right)\frac{\beta^{2}(p-1)^2}{\eta_{s}^{2}}$ samples of $\bX$, with probability at least $1-\delta$ we have $S\left(P||\RIP \lp \hat{P} \rp \right)\leq \eta_{s}$, where $P\in \isingtree$ and $\TCL$ is the Chow-Liu tree defined in \eqref{eq:CL_argmax_noiseless}). \\
%\end{comment} 
If the number of samples $\nn$ of $\bnX$ satisfies 
	\begin{align}\label{eq:sufficient:SKL}
	\nn \geq 4 \frac{\beta^{2}(p-1)^2}{(1-2q)^4\eta_{s}^{2}} \log\left( \frac{p^{2}}{\delta}\right),
	\end{align}
then for $\p (\cdot)\in \isingtree$ we have
	\begin{align}
	\P \lp \SKL\left(\p (\cdot)||\RIPn \big(\nPe \big) \right) \leq \eta_{s} \rp \ge 1-\delta,
	\end{align}
where $\TCLn$ is the Chow-Liu tree defined in \eqref{eq:CL_argmax_noisy} and the estimate $\RIPn\big(\nPe \big)$ is given by \eqref{eq:distribution_estimator}.
\end{theorem}

The asymptotic behavior of the bound in \eqref{eq:sufficient:SKL} was recently studied by~\citet{daskalakis2018testing}. In that work, a set of testing algorithms are proposed and analyzed under the assumption of an Ising model with respect to trees and arbitrary graphs. Theorem \ref{thm:SKLTheorem} gives rise to possible extensions of testing algorithms to the hidden model setting. We consider the latter as an interesting subject for future work.     

\begin{figure}[!t]
%\begin{subfigure}{.5\textwidth}
  \centering
  \includegraphics[width=.7\linewidth]{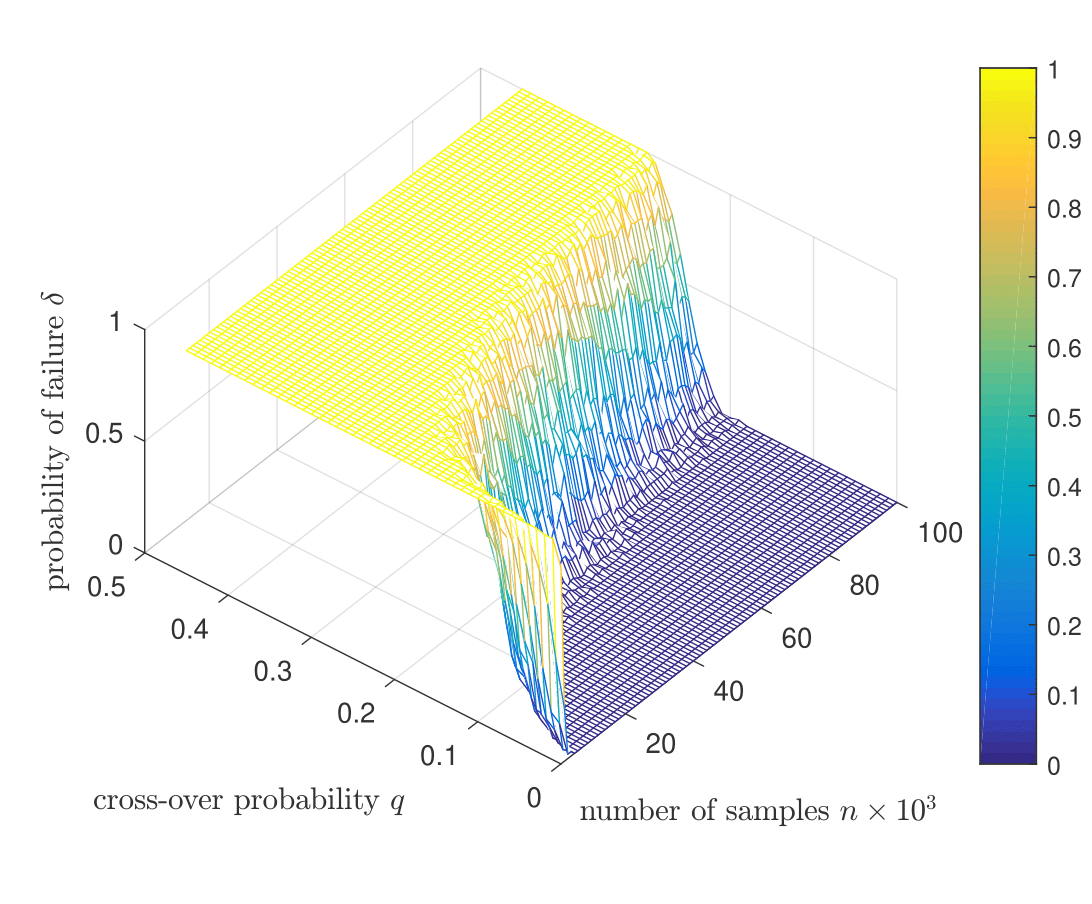}
  %\caption{Probability of incorrect structure recovery}
%\end{subfigure}%
\caption{Probability of incorrect structure recovery, The theoretical bound is given by Theorem \ref{thm:sufficient}. The top view of the figure is Figure \ref{fig:sfig1} and provides a clear comparison between the experimental and theoretical results.}  \label{fig:3D_1}
\end{figure}
\begin{figure}
%\begin{subfigure}{.5\textwidth}
  \centering
  \includegraphics[width=.7\linewidth]{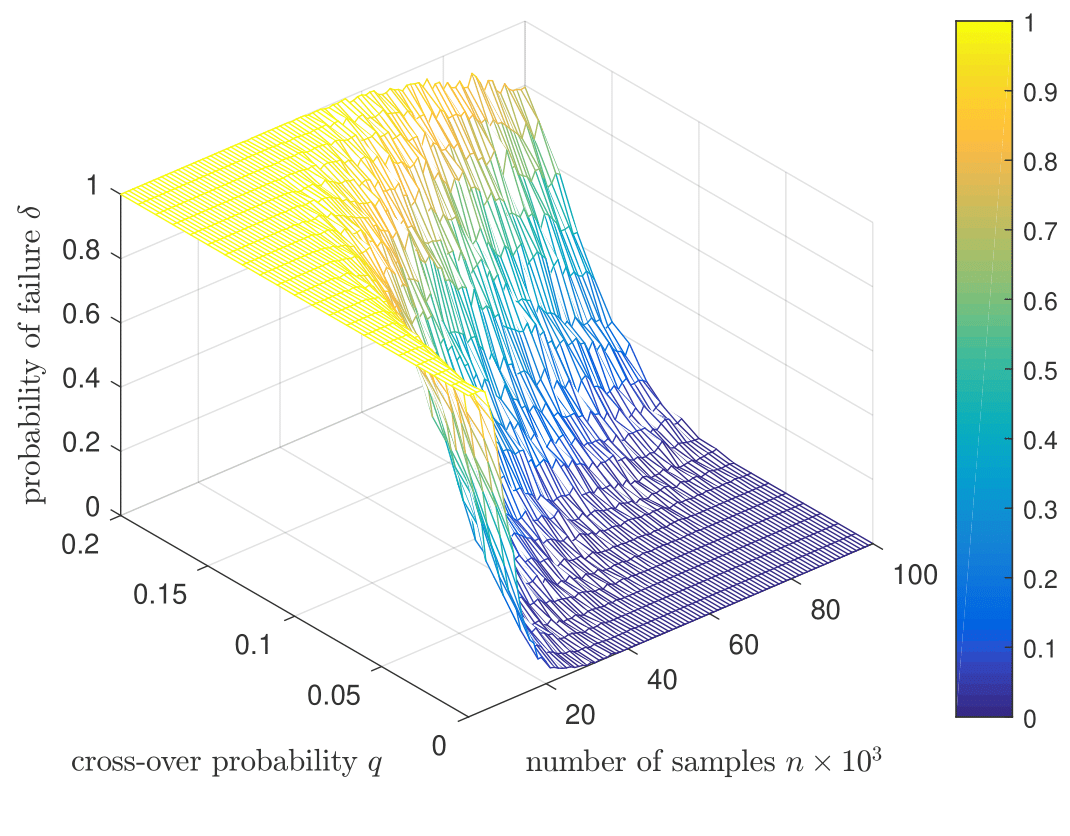}
  %\caption{Probability of incorrect structure recovery}
%\end{subfigure}%
\caption{Estimate of the probability of the ssTV to be greater than $\eta=0.03$. The theoretical bound is given by Theorem \ref{thm:Main_result}. The top view of the figure is Figure \ref{fig:sfig2} and provides a clear comparison between the experimental and theoretical results.}\label{fig:3D_2}
\end{figure}

\begin{figure}
%\begin{subfigure}{.5\textwidth}
  \centering
  \includegraphics[width=.7\linewidth]{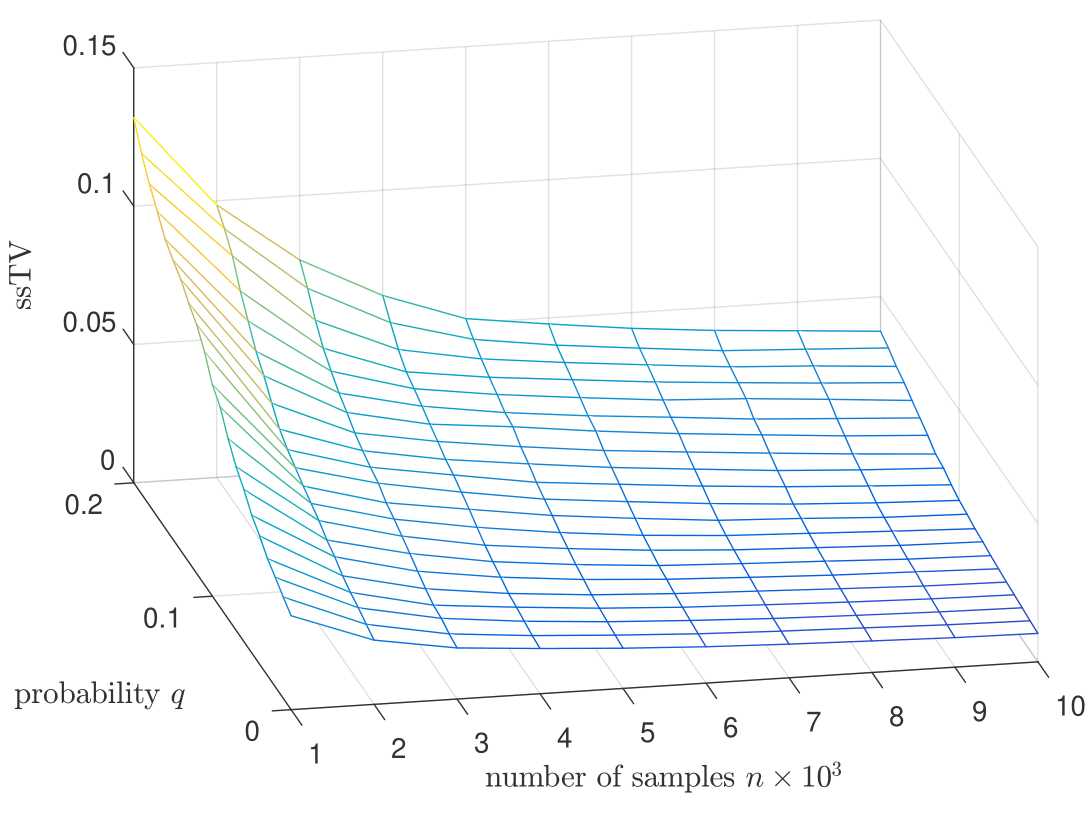}
  %\caption{Probability of incorrect structure recovery}
%\end{subfigure}%
\caption{Estimate of the distribution error metric ssTV as a function of $q$ and $n$.}\label{fig:3D_3}
\end{figure}

\subsection{Simulations}\label{Simulations}
We provide empirical results based on synthetic data to illustrate the probability of error $\delta$ as function of the cross-over probability $q$ and the number of samples $n$. For the simulations of this paper the original tree structure $\T$ is generated randomly where, starting from the root, we choose the parent of each new node uniformly at random among the nodes that are currently in the tree, in a sequential fashion. First, we estimate the probability of error $\P\big(\TCLn\neq \T \big)$ (named as $\delta$) of the structure learning problem, Figure \ref{fig:3D_1}. For the structure learning experiments, the number of nodes is $100$, $\beta=\arctanh (0.8)$, and $\alpha=\arctanh (0.2)$. Further, we considering $100$ Monte Carlo runs for averaging, and we plot the estimated probability of incorrect structure recovery while $q$ and $n$ vary. As a next step, we would like to see how well the theoretical bound of Theorem \ref{thm:sufficient} matches with the experimental results. To do this we plot the top view of Figure \ref{fig:3D_1} to get Figure \ref{fig:sfig1}. Quite remarkably, the theoretical and experimental bounds exactly match. The latter suggests that our theoretical bound that we derive, sample complexity of the Chow-Liu algorithm (Theorem \ref{thm:sufficient}), is indeed accurate. Second, we plot the probability of error for the predictive learning task, that is the probability of the ssTV to be greater than a positive number $\eta$ (Figure \ref{fig:3D_2}). For the simulation part, we restrict our attention to the case that the first of three terms in the maximization of \eqref{eq:main_result_bound} is the dominant. In fact, $\eta=0.03$, $p=31$, while $\alpha$ and $\beta$ are the same as the structure learning. Finally, Figure \ref{fig:3D_3} presents the ssTV itself for different values of $q$ and $n$. Finally, the top view of Figure \ref{fig:3D_2} is Figure \ref{fig:sfig2}, the latter suggest that the bound of our main result, Theorem \ref{thm:Main_result} is accurate.

Finally, we provide experimental results for the case of unknown $q$. Specifically, Figure \ref{Fig:est_q} illustrates the relationship between the average probability of error and the relative error $|\hat{q}-q|/q$ for the predictive learning task. We notice that the distribution can be approximated by using an estimate $\hat{q}$ of $q$ even for relative error $30\%$ or $60\%$.

\begin{figure}
\includegraphics[width=0.48\textwidth]{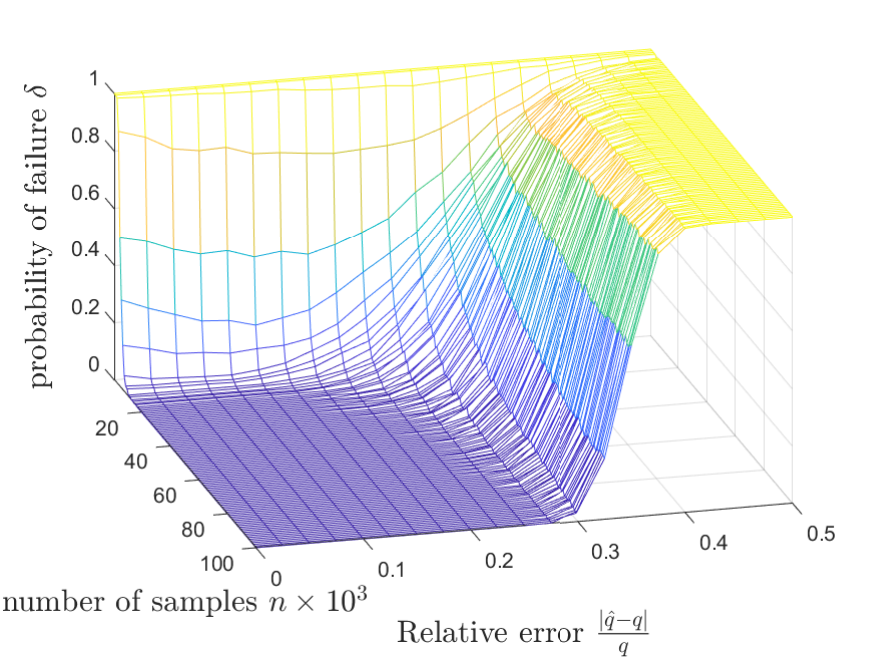}
\hspace*{\fill}
\includegraphics[width=0.48\textwidth]{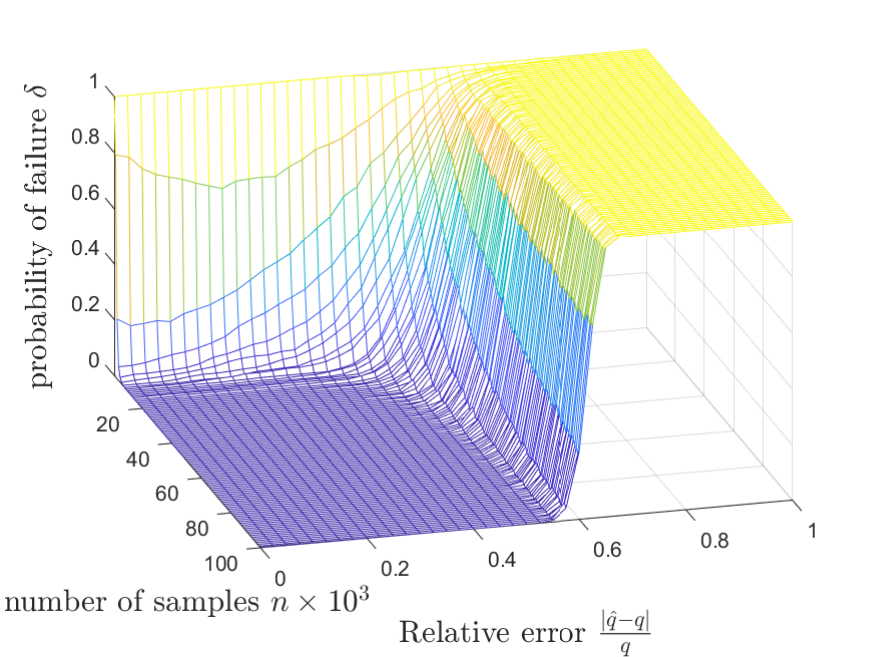}
\caption{The probability of error in the predictive learning task for different values of $n$ and estimation error of the parameter $q$. For both figures we consider $p=31$, $\alpha=0.2$, $\beta=1$, $q=0.1$ and averaging over $500$ independent runs. Left: $\hat{q}\in [0.05,0.15]$, $\eta=0.1$, Right: $\hat{q}\in [0,0.2]$, $\eta=0.12$.} \label{Fig:est_q}
\end{figure}

\section{Discussion }

In this section, we present sketches of proofs, we compare our results with prior work, we further elaborate on Algorithm \ref{alg:Chow-Liu}, Algorithm \ref{alg:matching_pairs} and the error of higher order moment estimates. First, we discuss the convergence of the estimate $\TCLn$ (Section \ref{The Chow-Liu algorithm: discussion}). In section \ref{Structure Learning under noisy observations and comparison with prior results}, we explain the connection between the hidden and noiseless settings on the tree structure learning problem. Later, in Section \ref{necessary events discussion}, we present the analysis and a sketch of proof for Theorem \ref{thm:Main_result}. Finally, in Section \ref{higher_order_moments}, we provide further details about Theorem \ref{thm:binary:isserlis}, discussion about the Matching Pairs algorithm (Algorithm \ref{alg:matching_pairs}) and the accuracy of the proposed higher order moments estimator \eqref{eq:second_order_monets_estimate}.

\begin{comment}
\ads{estimated correlations are not defined yet!!!}
The mutual information can be estimated through estimates of correlations 
\begin{align*}
\hat{I}\left(X_{i},X_{j}\right)=\frac{1}{2}\log\left(1-\hat{\E}^{2}\left[X_{i}X_{j}\right]\right)+\frac{1}{2}\hat{\E}\left[X_{i}X_{j}\right]\log\frac{1+\hat{\E}\left[X_{i}X_{j}\right]}{1-\hat{\E}\left[X_{i}X_{j}\right]}.
\end{align*}

The estimated mutual information is a symmetric function of $\hat{\E}\left[X_{i}X_{j}\right]$
and increasing with respect to $\left|\hat{\E}\left[X_{i}X_{j}\right]\right|$. 

\ads{move this to the algorithm section.} We can use Kruskal's algorithm with input the estimated correlations to derive the Chow-Liu tree at the output:
	\begin{align}\label{eq:CL_argmax_noiseless}
	\TCL= \argmax_{\T\in \mc{T}}  \sum_{(i,j)\in \in \mc{E}_{\T}} \left|\hat{\E}\left[X_{i}X_{j}\right]\right|
	\end{align}
where $\mc{T}$ is the family of all spanning trees on the set $\mc{V}$.  
\end{comment}

\subsection{Estimating the Tree Structure $\T$ \label{The Chow-Liu algorithm}\label{The Chow-Liu algorithm: discussion}}

In this work, the structure learning algorithm is based on the classical Chow-Liu algorithm, and is summarized in Algorithm \ref{alg:Chow-Liu}. We can express its output as
\begin{align}\label{eq:CL_argmax_noisy}
\TCLn=\argmax_{\T\in \mc{T}}  \sum_{(i,j)\in \in \mc{E}_{\T}} \left|\nmue_{i,j}\right|,
\end{align} (see also Section \ref{MLE}.). The difference between Algorithm \ref{alg:Chow-Liu} and the Chow-Liu algorithm of the noiseless scheme is the use of noisy observations as input, since we consider a hidden model, whereas~\cite{bresler2020learning} assume that observations directly from the tree-structured model are available. Further, \eqref{eq:CL_argmax_noisy} shows the consistency of the estimate $\TCLn$ for sufficiently large $n$. The tree structure estimator $\TCLn$ converges to $\T$ when $n\to\infty$, since \begin{align}\label{eq:asc_nmu}
\lim_{n\to \infty}\nmue_{i,j}\overset{\text{a.s.}}{=} c^2 \mu_{i,j}.
\end{align} From \eqref{eq:CL_argmax_noisy} and \eqref{eq:asc_nmu} we have (under an appropriate metric) \begin{align}
\lim_{n\to\infty} \TCLn \overset{\text{a.s.}}{=} \T.
\end{align} Asymptotically, both $\TCL$ and $\TCLn$ converge to $\T$, where $\TCL$ denotes the structure estimate from noiseless data ($q=0$). Most importantly, the Chow-Liu algorithm also returns the exact hidden structure with high probability given finite number of noisy samples. We provide the finite sample complexity bound in Theorem \ref{thm:sufficient}. For a fixed probability of exact structure recovery $1-\delta$, more samples are required in the hidden model setting, compared to the noiseless one. Additionally, the difference of the sample complexity between the noisy and noiseless setting comes from Theorem \ref{thm:sufficient} by comparing the bound for the values $q=0$ and $q\neq 0$.

\subsection{Hidden Structure Recovery and Comparison with Prior Results}\label{Structure Learning under noisy observations and comparison with prior results}
% \textbf{Comparison with Prior Results.} 
Theorem \ref{thm:sufficient} and Theorem \ref{thm:necessary} extend the noiseless setting \citep[Theorem 3.2, Theorem 3.1]{bresler2020learning} to our hidden model; the noiseless results correspond to $q = 0$. 
In particular, in the presence of noise, \textit{the dependence on $p$ remains strictly logarithmic, that is,} $\mc{O}(\log(p/\delta))$. 
%More specifically, Theorem \ref{thm:sufficient} extends Theorem 3.3  (Sufficient samples for structure learning) by Bresler and Karzand~\cite[Section 8.2]{bresler2016learning}, which is retrieved by setting zero cross over probability $q=0$ in Theorem \ref{thm:sufficient}. 
To make the connection between sufficient conditions more explicit, by setting $q=0$ in \eqref{eq:sufficient_number_of_samples_noise_thm} of Theorem \ref{thm:sufficient}, we retrieve the corresponding structure learning result by~\citet[Theorem 3.2]{bresler2020learning} \textit{exactly}: Fix a number $\delta\in (0,1)$. If the number of samples of $\bX$ satisfy the inequality \begin{align}\label{eq:suff_number_of_samples}
	n 
    &\geq 
    \frac{32}{\tanh^{2} \alpha\left(1-\tanh\beta\right)}
    	\log\lp \frac{2p^{2}}{\delta}\rp,
	\end{align}then the Chow-Liu algorithm returns $\TCL=\T$ with probability at least $1-\delta$. An equivalent condition of \eqref{eq:suff_number_of_samples} is 
 	\begin{align}\label{eq:weak edge noiseless}
	\tanh\alpha \geq\frac{4\eps}{\sqrt{1-\tanh\beta}}\triangleq \tau (\epsilon),
    \text{\ and\ }
    \eps\triangleq\sqrt{2/n\log\left(2p^{2}/\delta\right)},
	\end{align} 
the latter shows that the weight of weakest edge should satisfy the following inequality $\alpha>\arctanh\lp 4\eps/\sqrt{1-\tanh\beta}\rp$ (\citet{bresler2020learning}).  For the hidden model, the equivalent extended condition for the weakest edge is 
	\begin{align}\label{eq:weak edge noisy}
	\tanh\alpha 
    &\geq\frac{4\neps\sqrt{1-\left(1-2q\right)^{4}\tanh\beta}}{\left(1-2q\right)^{2}\left(1-\tanh\beta\right)}\triangleq \ntau (\neps),
    \text{\ and\ }
    \neps
    =
    \sqrt{ \frac{2\log\left(2p^{2}/\delta\right)}{\nn}},
	\end{align} (see Appendix C, Lemma \ref{lemma_Estrong}) Condition \eqref{eq:weak edge noiseless} is retrieved through \eqref{eq:weak edge noisy} for $q=0$.
Note that, for $q=1/2$, the mutual information of the hidden and observable variables is zero, thus structure recovery is impossible. 
%Thus, we need infinitely many samples to guarantee exact recovery; in other words, structure recovery is impossible. 

Theorem \ref{thm:necessary} provides the necessary number of samples bound for exact structure recovery given noisy observations. In fact, it generalizes Theorem 3.1 by~\citet{bresler2020learning} to the hidden setting. By fixing $q=0$, Theorem \ref{thm:necessary} recovers the noiseless case. Fix $\delta\in (0,1)$. If the number of samples of $\bX$ satisfies the inequality
	\begin{align}
	\n< \frac{1}{16}e^{2\beta}\left[\alpha\tanh(\alpha)\right]^{-1}\log \lp p\rp,
	\end{align} 
then for any algorithmic mapping (estimator) $\psi$, it is true that
	\begin{align}
	\inf_{\psi} \sup_{\substack{\T\in\mc{T} \\ P\in \isingtree}} \P \lp \psi \lp X_{1:n} \rp \neq \T \rp>\frac{1}{2}.
	\end{align}  
When there is no noise, $q=0$, we retrieve the noiseless result, while for any $q\in (0,1/2)$ the sample complexity increases since $[1-(4q(1-q))^p]^{-1}>1$ in \eqref{eq:suff_n_complexity} and for $q\to1/2$ the required number of samples $\nn \to \infty$, which makes structure learning impossible. The ratio between the noiseless and noisy necessary conditions indicates the gap between the hidden model and the original (noiseless) setting, which reads
	\begin{align}\label{eq:1/SDPI}
	\frac{\nn}{n}\leq[1-(4q(1-q))^p]^{-1} \leq \frac{1}{\eta_{\text{KL}}},
	\end{align} 
(see Appendix E.). The right hand-side of \eqref{eq:1/SDPI} is the strong data processing inequality for the binary symmetric channel, which was recently developed by~\citet[Equation (39)]{polyanskiy2017strong}. We continue by providing the main idea and the important steps of the proof of Theorem \ref{thm:Main_result}.

\subsection{Theorem \ref{thm:Main_result}: A Sketch of the Proof}
\label{necessary events discussion}

Recall that the indices $i,j\in\mc{V}$ of the quantities $\nmu_{i,j}$ and $\nmue_{i,j}$ are pair of nodes, and in fact they can be considered as one (pair) index. For sake of space we introduce the notation $\nmu_{e}$ and $\nmue_{e}$ for some $e\in\mc{E}_{\T}$, that is consistent with our previous definition and $e$ represents a pair of nodes. Theorem \ref{thm:Main_result} guarantees that the estimated pairwise marginal distributions are close to the the original distributions by ensuring that the $\Lnorm$ is small. In this section we provide a sketch of the proof of the Theorem and we mention the main differences between the hidden model and the noiseless case~\citep{bresler2020learning}. The intersection of three events is sufficient to guarantee that $\Lnorm$ is upper bound by $\eta>0$:
	\begin{align}
	\Ecorrn &\triangleq
		\left\{ \sup_{i,j\in\mc{V}}\left|\nmu_{i,j}-\nmue_{i,j}\right|\leq\neps\right\} 
		\label{eq:event_corr_nois-1},\\
	\Estrongn &\triangleq
		\left\{ \left\{ e\in\mc{E}_{\T}:\left|\tanh\theta_e\right|\geq\frac{\ntau (\neps)}{\left(1-2q\right)^{2}}\right\} \in\mc{E}_{\TCLn}\right\} 
		\label{eq:event_strong_nois-1},\\
	\Ecascn &\triangleq 
		\left\{ \left| \prod_{e\in\tpath_\T (i,j)}\frac{\hat{\mu}^{\dagger}_{e}}{(1-2q)^2} - \prod_{e\in\tpath_\T (i,j)}\frac{\mu^{\dagger}_{e}}{(1-2q)^2} \right|\leq\ngam: \,\, i,j\in\mc{V}\right\},
		\label{eq:event_cascade_nois-1}
\end{align}
where \eqref{eq:weak edge noisy} gives the definition of $\ntau (\neps)$. The three events are equivalent events of the noiseless case, but they are modified accordingly to guarantee accurate estimation based on noisy data. The event $\Ecorrn$ guarantees that the error of the correlation estimates is not greater than $\neps$. Under the event $\Estrongn$ all the strong edges are recovered by the Chow-Liu algorithm. Similarly to the noiseless setting, the event $\Estrongn$ requires the Chow-Liu algorithm to recover all the strong edges, while the weak edges (those that do not satisfy the inequality in \eqref{eq:event_strong_nois-1}) do not affect the accuracy of the predictive learning, even if the Chow-Liu algorithm fails to recover them. In contrast with structure learning, exact structure recovery is not necessary for the predicative learning task. In other words, even if $\alpha$ is extremely small, assume $\tanh(\alpha)\leq \ntau(\neps)/(1-2q)^2$ the required number of samples for accurate predictive learning will remain unaffected.

Under the event $\Ecascn$ the end-to-end error along paths is no greater than $\ngam$. In fact, each path between two nodes of the tree can be considered a sequence of segments with strong and weak edges. The end-to-end path error is determined by the strong edge segments of the path through the parameter $\ngam$ for the $\Ecascn$ event, while the effect of weak edges parameters is controlled by the quantity $\ntau (\neps)$ (for the segmentation of the tree and the detailed proof see \ref{Lemma end_to_end_error}). Our goal is to find sufficient conditions on the parameters $\neps$ and $\ngam$ that guarantee that the events $\Ecorrn, \Ecascn$ and $\Estrong$ occur with high probability. 

Recall that our goal is to guarantee that the quantity $\mc{L}^{(2)}(\p(\cdot),\RIPn (\hat{\p}_{\dagger} ))$ is smaller than a fixed number $\eta>0$ with probability at least $1-\delta$. To do this, we follow the technique of prior work by~\cite{bresler2020learning}, the triangle inequality gives 
\begin{align}\label{eq:triangle_sketch}
\mc{L}^{(2)}\left(\p(\cdot),\RIPn (\hat{\p}_{\dagger} )\right)\leq\mc{L}^{(2)}\left(\p (\cdot),\RIPn\left(\p(\cdot)\right)\right)+\mc{L}^{(2)}\left(\RIPn\left(\p(\cdot)\right),\RIPn (\hat{\p}_{\dagger} )\right),
\end{align} and we find the required number of samples such that each of the terms $\mc{L}^{(2)}(\p (\cdot),\RIPn\left(\p(\cdot)\right))$ and $\mc{L}^{(2)}(\RIPn\left(\p(\cdot)\right),\RIPn (\hat{\p}_{\dagger} ))$ is no greater than $\eta/2$ with probability at least $1-\delta$. As we show the probability of the event $\Ecascn$ (Lemma \ref{Lemma Ecascn}, Appendix) and the $\Lnorm$ (between the true and estimated distribution) can be bounded by a constant uniformly over the set of all trees and is not affected by long paths. To prove these properties of the hidden model is non-trivial and ensures that the estimation error from noisy observations does not increase exponentially along paths as someone might expect. Specifically, the first quantity at the right hand-side of inequality \eqref{eq:triangle_sketch} represents the loss due to graph estimation error, while the second term represents the loss due to parameter estimation error. Lemma \ref{Lemma end_to_end_error} (Appendix) shows that under the event 
\begin{align}
\Einter\triangleq\Ecorrn\cap\Ecascn \cap \Estrongn,    
\end{align} if \begin{align}\label{eq:condition1}
\ngam\leq \frac{\eta}{3}\text{ and } \neps\leq (1-2q)^2 e^{-\beta}\left[20\left(1+ 2e^{\beta}\sqrt{2\left(1-q\right)q\tanh\beta}\right)\right]^{-1}
\end{align}
then $\mc{L}^{(2)}(\RIPn\left(\p(\cdot)\right),\RIPn (\hat{\p}_{\dagger} ))\leq \eta/2$. Further Lemma \ref{Lemma F5} (Appendix) shows that if
\begin{align} \label{eq:condition2}
\neps\leq \min \left\{ \frac{\eta}{16}(1-2q)^2,\frac{(1-2q)^2 e^{-\beta}}{24\left(1+ 2e^{\beta}\sqrt{2\left(1-q\right)q\tanh\beta}\right)} \right\}
\end{align}
then $\mc{L}^{(2)}\left(\p (\cdot),\RIPn\left(\p(\cdot)\right)\right)\leq \frac{\eta}{2}$ under the event $\Ecorrn\cap \Estrongn$. Both conditions \eqref{eq:condition1} and \eqref{eq:condition2} should be satisfied, so it is necessary to have \begin{align}\label{eq:condition3}
    \ngam\leq \frac{\eta}{3} \text{ and } \neps\leq \min \left\{ \frac{\eta}{16}(1-2q)^2,\frac{(1-2q)^2 e^{-\beta}}{24\left(1+ 2e^{\beta}\sqrt{2\left(1-q\right)q\tanh\beta}\right)} \right\}.
\end{align}
To guarantee that the errors $\ngam$ and $\neps$ are sufficient small such that \eqref{eq:condition3} is satisfied, we need to make sure that the number of samples $n$ is sufficiently large. In fact, the upper bounds on the errors translate into lower bounds on the number of samples through the concentration bounds for the events. Specifically, Lemma \ref{Lem_Ecorrn} gives a sufficient sample size to ensure that the event $\Ecorrn$ occurs with probability at least $1-\delta$, Lemma \ref{lemma_Estrong} gives the concentration bound for the event $\Estrongn$ and Lemma \ref{Lemma Ecascn} gives the concentration bound of the event $\Ecascn$. Lemma \ref{Lem_Ecorrn}, Lemma \ref{lemma_Estrong} and Lemma \ref{Lemma Ecascn} together with \eqref{eq:condition3} give the final bound of the sample complexity (see the proof \ref{Main_result_proof_Appendix}) \begin{align}\label{eq:sketch_result1}
    n\geq& \max \Bigg\{\frac{512}{ \eta^2(1-2q)^4}, \frac{1152 e^{2\beta}B(\beta,q)}{(1-2q)^4}, \frac{48 e^{4\beta}}{\eta^2}\Gamma(\beta,q) \Bigg\}\log\lp\frac{6p^3}{\delta}\rp
\end{align} and its simplified but looser bound  \begin{align}\label{eq:sketch_result2}
    n\geq& \max \Bigg\{\frac{512}{ \eta^2(1-2q)^4}, \frac{1152 \lp1+3\sqrt{q}\rp^2e^{2\beta(1+\mathds{1}_{q\neq 0})}}{(1-2q)^4}, \frac{48 e^{4\beta}}{\eta^2}\mathds{1}_{q\neq 0}  \Bigg\}\log\lp\frac{6p^3}{\delta}\rp,
\end{align} that provides the condition of  Theorem \ref{thm:Main_result:short}. Although the general structure of our argument follows that of the noiseless case, the presence of noise introduces several technical challenges whose solution may be of independent interest. In the sequel, we highlight the most important aspects of our approach that do not appear in the noiseless case.

The proof of Theorem 3.3 is significantly different and includes additional steps and techniques compared with the approach by~\cite{bresler2020learning}. Specifically, Lemma \ref{Lemma Cond_Prob} is new and it is necessary for the hidden model and we use it later to prove (Lemma \ref{Lemma Ecascn}, Appendix). Lemma \ref{lemma:Edge_corr_estimation_error} is an non-trivial extension of the accurate estimation of edges' correlation. Although the resulting expression seems complicated is important for the proof of Lemma \ref{Lemma Ecascn}. In fact Lemma \ref{Lemma Ecascn}, the proof of the concentration bound for the event $\Ecascn$, is significantly more complicated and longer than the noiseless model (see Appendix E by~\cite{bresler2020learning} for comparison). To show this result we have to consider a martingale difference sequence and evaluate upper bounds for the conditional variance and bias of that sequence. The bias is crucial for the final result because it introduces an extra term in the final bound that does not exist in the noiseless case. It is interesting that this term does not involve any parameter related to the noise and shows how the result is affected by the structural inconsistency between the hidden and the observable layer. As a consequence, the expression of the bound \eqref{eq:suff_bound_cascade_red} in Lemma \ref{Lemma Ecascn} involves two inequalities to guarantee the high-probability bound. The first inequality which introduces restrictions on the parameter $\Delta$ (see inequality \ref{eq:suff_bound_cascade_red}) is an attribute of the noisy case. We continue by briefly explaining one of the main technical aspects of the proof.

%Finally, we highlight parts of our proof that are necessary for the hidden model and do not appear in the noiseless case. These aspects can also be of independent interest. 
To begin with, consider a path of length $d\geq 2$ in the original tree $\T$, $X_1-X_2-\cdot\cdot\cdot-X_{d+1}$ and we denote the edge $(k,k+1)$ as $e_k$, for some $k\in [d]$. Recall that $Y^{(i)}_k$ denotes the $i^\text{th}$ sample of $Y_k$ and $k\in[d+1]$. We would like derive a concentration bound of the probability of the event $\Ecascn$ (Lemma \ref{Lemma Ecascn}, Appendix). To do this, first we have to consider for all $\ell\in [n]$ and $k\in \{2,\ldots,d\}$ the random variables \begin{align}
    Z^{(\ell)}_k\triangleq  \lp\frac{ \lp X_k N_k X_{k+1} N_{k+1}\rp^{(\ell)}}{(1-2q)^2}-\frac{\nmu_{e_k}}{(1-2q)^2}\rp\prod^{k-1}_{j=1}\frac{\nmue_{e_j}}{(1-2q)^2}\prod^{d}_{j=k+1}\frac{\nmu_{e_j}}{(1-2q)^2}.
\end{align}
Define the martingale difference sequence (MDS) $\{\xi_{k}^{(i)}\}$ by setting $\xi_{k}^{(0)}\triangleq 0$, $\xi_{k}^{(1)}\triangleq Z^{(1)}_k-\E\left[Z^{(1)}_k | \nmue_{e_{k-1}},\ldots,\nmue_{e_1}\right]$, $\xi^{(i)}_{k}\triangleq Z^{(i)}_k -\E\left[Z^{(i)}_k |Z^{(i-1)}_k,\ldots, Z^{(1)}_k,\nmue_{e_{k-1}},\ldots,\nmue_{e_{1}} \right]$. Let $\mc{F}^{k}_{i-1}$ be the $\sigma$-algebra generated by $Z^{(i-1)}_k,\ldots, Z^{(1)}_k, \nmue_{e_{k-1}},\ldots,\nmue_{e_{1}}$. Then the pair $(\xi^{(i)}_{k},\mc{F}^{k}_i)_{i=1,\ldots,n}$ is an MDS. In contrast with the noiseless case, \emph{the conditional means are not zero}, which makes the problem significantly harder. To proceed, we apply a concentration bound for supermartingales (generalized Bennett's inequality) by~\cite{fan2012hoeffding}. 

Secondly we have to evaluate the following expression \begin{align}
    &\nonumber \P\lp  Y^{(\ell)}_{k} Y^{(\ell)}_{k+1} =\pm 1 \Big|\nmue_{e_{k-1}},\ldots,\nmue_{e_{1}} \rp \\&\qquad\qquad\qquad\qquad\qquad\qquad =\frac{1\pm \nmu_{e_{k}}}{2}\frac{1-\nmu_{e_{k-1}}\nmue_{e_{k-1}}}{1-(\nmu_{e_{k-1}})^2}+\nmu_{e_{k-1}}\frac{1\pm \mu_{e_{k}}}{2}\frac{\nmue_{e_{k-1}}-\nmu_{e_{k-1}}}{1-(\nmu_{e_{k-1}})^2}.\label{eq:conditional_distribution_on_mus}
\end{align} 
In the noiseless case, the product variables $X^{(\ell)}_{k} X^{(\ell)}_{k+1}$ are independent, leading to a simple expression for this probability (see Lemma \ref{independent_products}, Appendix). 
%However, in the noisy case, deriving a closed-form expression is nontrivial because \emph{the product variables $Y^{(\ell)}_{k} Y^{(\ell)}_{k+1}$, $k\in [d]$ are no longer independent}. On the other hand, the derivation of the corresponding distribution in the noiseless case is trivial because the product variables $X^{(\ell)}_{k} X^{(\ell)}_{k+1}$ are independent (see Lemma \ref{independent_products}, Appendix). 
The closed form expression of \eqref{eq:conditional_distribution_on_mus} is given by Lemma \ref{Lemma Cond_Prob}. Finally, the expectations $\E\left[Z^{(i)}_k |Z^{(i-1)}_k,\ldots, Z^{(1)}_k,\nmue_{e_{k-1}},\ldots,\nmue_{e_1} \right]$ are not zero, however when $n\to  \infty$, they approach zero. As a consequence, a bias exists that affects the sample complexity by introducing an additional term in the bound that that does not appear in the noiseless case, the quantity $e^{4\beta}/\eta^2$ (see Equation \ref{eq:sketch_result1} and Equation \ref{eq:sketch_result2}).

Finally, we continue by bounding the norm $\Lnorm$ between the true and estimated distribution in Appendix E. The proof of Lemma \ref{Lemma end_to_end_error} shows that in the noisy setting as well, the $\Lnorm$ can be bounded by a constant uniformly over the set of all trees and it is not affected by long paths. This property of the hidden model is highly non-trivial and ensures that the estimation error from noisy observations does not increases along paths as someone might expect. Lemma \ref{Lemma F5} follows the corresponding approach of Lemma 6.1 by~\cite{bresler2020learning} and we provide only the required for the noisy setting differences. In Theorem \ref{Main_result_proof_Appendix}, we combine the Lemmata of Appendices \ref{Ecascadesection} and \ref{ProofofthemainresultsSection}, we find the appropriate choice of the parameter $\Delta$ that satisfies the necessary conditions of Lemma \ref{Lemma Ecascn} and we derive the final sample complexity bound. For further details about the proof of the main result see Appendix, Section \ref{Ecascadesection} and Section \ref{ProofofthemainresultsSection}. 

\subsection{Estimating Higher Order Moments}\label{higher_order_moments}

Our results also provide an analogue of Isserlis' Theorem (Theorem  \ref{thm:binary:isserlis}) and the \textit{Matching Pairs} algorithm, which returns the set $\mc{CP}_{\T}(\mc{V'})$ in \eqref{eq:binary:isserlis}. We provide a short proof sketch for the bound on the error of estimation \eqref{eq:accuracy_HOM}.

 \textbf{Proof sketch of Theorem \ref{thm:binary:isserlis}:} We prove that $\mc{C}_{\T}(\mc{V'})$ always exists (when $k$ is even) by induction (see Appendix A, Lemma \ref{disjoint_paths}). We define the set of edges $\mc{CP}_{\T}(\mc{V'})$ as the union of the edge-disjoint paths\footnote{By edge-disjoint paths we refer to paths with no common edges.} $\mc{CP}_{\T}(\mc{V'})=\cup_{w,w'\in \mc{C}_{\T}(\mc{V'})}\tpath(w,w')$. Combining the set $\mc{CP}_{\T}(\mc{V'})$ together with the independent products property (see Lemma \ref{independent_products}), we derive the final expression (see Appendix A, proof of Theorem \ref{thm:binary:isserlis}). Given the tree structure $\T$ and the correlations $\mu_e$ for all $e\in \mc{E}$, we can calculate the higher order expectations. Notice that the collection of edge-disjoint paths $\mc{CP}_{\T}$ depends on the tree structure and as a consequence an algorithm is required to discover those paths. Different matching algorithms can be considered to find the set $\mc{CP}_{\T}$. We propose Algorithm \ref{alg:matching_pairs} which is simple and has low complexity of $\mc{O}(|\mc{E}|)$.

 \textbf{Matching Pairs Algorithm:} Algorithm \ref{alg:matching_pairs} requires as input the tree and the set of nodes $\mc{V'}\equiv\{i_1,\ldots,i_k\}\subset\mc{V}$, and returns the set of edges $\mc{CP}_{\T}(V')$. For each node in the tree, a flag variable is assigned to each node and indicates if the corresponding node is a candidate for the final set $C_{\T}(\mc{V'})$ at the current step of the algorithm. The candidate nodes have to be matched with other nodes of the tree, such that the pairs generate edge-disjoint paths. Initially, the candidate nodes are the nodes of the set $\mc{V}'$. Starting from the nodes which appear in the deepest level of the tree, we ``move'' them to their ancestor. At each step, if two candidate nodes appear at the same point, we match them as pair, we store the pair in the set $\CP (\mc{V'})$ and we remove both of them from the set $\mc{V'}$. We continue until $\mc{V'}\equiv\emptyset$. The complexity of Algorithm \ref{alg:matching_pairs} is $\mc{O}(|\mc{E}|)$. Finally, Theorem \ref{thm:binary:isserlis} can be extended to any forest $\text{F}$ structure by considering the set $\mc{CP}_{\text{F}} (\mc{V'})$ instead of $\CP (\mc{V'})$, where $\mc{CP}_{\text{F}} (\mc{V'})\triangleq \cup_{i} \mc{CP}_{\text{T}_i} (\mc{V'})$ and $T_i$ is the $i^{\text{th}}$ connected tree of the forest.

\textbf{Estimation error of higher order moments:} Inequality \eqref{eq:accuracy_HOM} bounds the error of estimation by the small set Total Variation (ssTV), that is guaranteed to be less than $\eta>0$ by Theorem \ref{thm:Main_result}. Additionally, the bound on the error of the estimation in \eqref{eq:accuracy_HOM} can be found as follows
\begin{align}
 \left|\hat{\E}\left[\prod_{s\in \mc{V}'} X_s \right]-\E\left[\prod_{s\in \mc{V}'} X_s\right] \right|\nonumber&\\
 &\mkern-200mu =\left|\prod_{e\in \mc{CP}_{\TCLn}(i_1,i_2,\ldots,i_k)}\frac{\nmue_e}{(1-2q)^2}-\prod_{e\in \mc{CP}_{\T}(i_1,i_2,\ldots,i_k)}\mu_e \right|\label{eq:error_mom1}\\
 &\mkern-200mu= \left|\prod_{e\in \mc{CP}_{\TCLn}(i_1,i_2,\ldots,i_k)}\frac{\nmue_e}{(1-2q)^2}-\prod_{e\in \mc{CP}_{\T}(i_1,i_2,\ldots,i_k)}\frac{\nmu_e}{(1-2q)^2} \right|\nonumber\\
&\mkern-200mu=  \left|\prod_{e\in \underset{\left\{w,w'\right\}\in \mc{C}_{\TCLn}(\mc{V'})}{\bigcup} \tpath_{\TCLn}(w,w')}\frac{\nmue_e}{(1-2q)^2}-\prod_{e\in \underset{\left\{w,w'\right\}\in \mc{C}_{\T}(\mc{V'})}{\bigcup} \tpath_{\T}(w,w')}\frac{\nmu_e}{(1-2q)^2} \right| \numberthis \label{
eq:expect_error_1}\\
&\mkern-200mu= \left| \prod_{\left\{w,w'\right\}\in \mc{C}_{\TCLn}(\mc{V'})} \prod_{e\in\tpath_{\TCLn}(w,w')}\frac{\nmue_e}{(1-2q)^2}-\prod_{\left\{w,w'\right\}\in \mc{C}_{\T}(\mc{V'})} \prod_{e\in\tpath_{\T}(w,w')}\frac{\nmu_e}{(1-2q)^2} \right| \nonumber\\
&\mkern-200mu\leq  2|\mc{V'}| \Lnorm \lp \p(\cdot),\RIPn (\hat{\p}_{\dagger} ) \rp,\numberthis \label{eq:Exp_error_bounded_L2}
\end{align} where \eqref{eq:error_mom1} holds due to \eqref{eq:binary:isserlis} and \eqref{eq:second_order_monets_estimate}, \eqref{
eq:expect_error_1} comes from \eqref{eq:CPT} and the last inequality \eqref{eq:Exp_error_bounded_L2} is being proved by~\citet[Lemma 1, supplementary material]{bresler2020learning}. Thus, if we can accurately estimate the distribution under the sense $\Lnorm \lp P,\RIPn (\hat{\p}_{\dagger} ) \rp\leq \eta'$, for a sufficiently small positive number $\eta'$, then by using \eqref{eq:second_order_monets_estimate} and choosing $\eta'\leq \eta/(2|\mc{V}'|)$, Theorem \ref{thm:Main_result} guarantees accurate estimates for higher order moments with probability at least $1-\delta$.

\section{Conclusion}

We have considered and analyzed the problem of predictive learning on hidden tree-structures from noisy observations, using the well-known Chow-Liu algorithm. In particular, we derived sample complexity guarantees for exact structure learning and marginal distributions estimation. Our bounds extend prior work (see~\cite{bresler2020learning}) to the hidden model, by introducing the cross-over probability $q$ of the $\BSC$. Our results exactly reduce to the noiseless setting when $q = 0$, and the explicit expressions of the bounds are also continuous functions of $q$. Additionally, by applying a graph property for tree structures and a probabilistic property for Ising models, we derived an equivalent of the well-known Isserlis' theorem for Gaussian distributions, which yields to a consistent high-order moments estimator for Ising models. Further, we considered simulations based on synthetic data to validate our theoretical results. Our theoretical bounds exactly match with the experiment. indicating that our results correctly characterize the dependence on the model parameters.

Our results show that the estimated structure statistic $\TCLn$ is essential for successful statistical inference on the hidden (or observable) layer, while the sample complexity with respect to number of nodes and probability of error remains strictly logarithmic, as in the noiseless case. Our  hidden setting constitutes a first step towards more technically challenging and potentially more realistic statistical models, such as, for instance, structure and distribution learning when the noise is generated by an erasure channel, or when the underlying hidden tree structured distribution has a larger, or even uncountable, support.

\appendix

%%% SEE http://jmlr.org/format/format.html
%%%
%%% for formatting information about appendices

\section{Preliminaries and Outline of Proof}\label{App.A}

The chart in Figure \ref{fig:1-a} shows the various dependencies of the Lemmata and intermediate results either considered or developed in this paper, and the resulting Theorems. The proofs can be found in the corresponding section of the Appendix.

\begin{figure}[!ht]
     \centering
    \includegraphics[width=\textwidth]{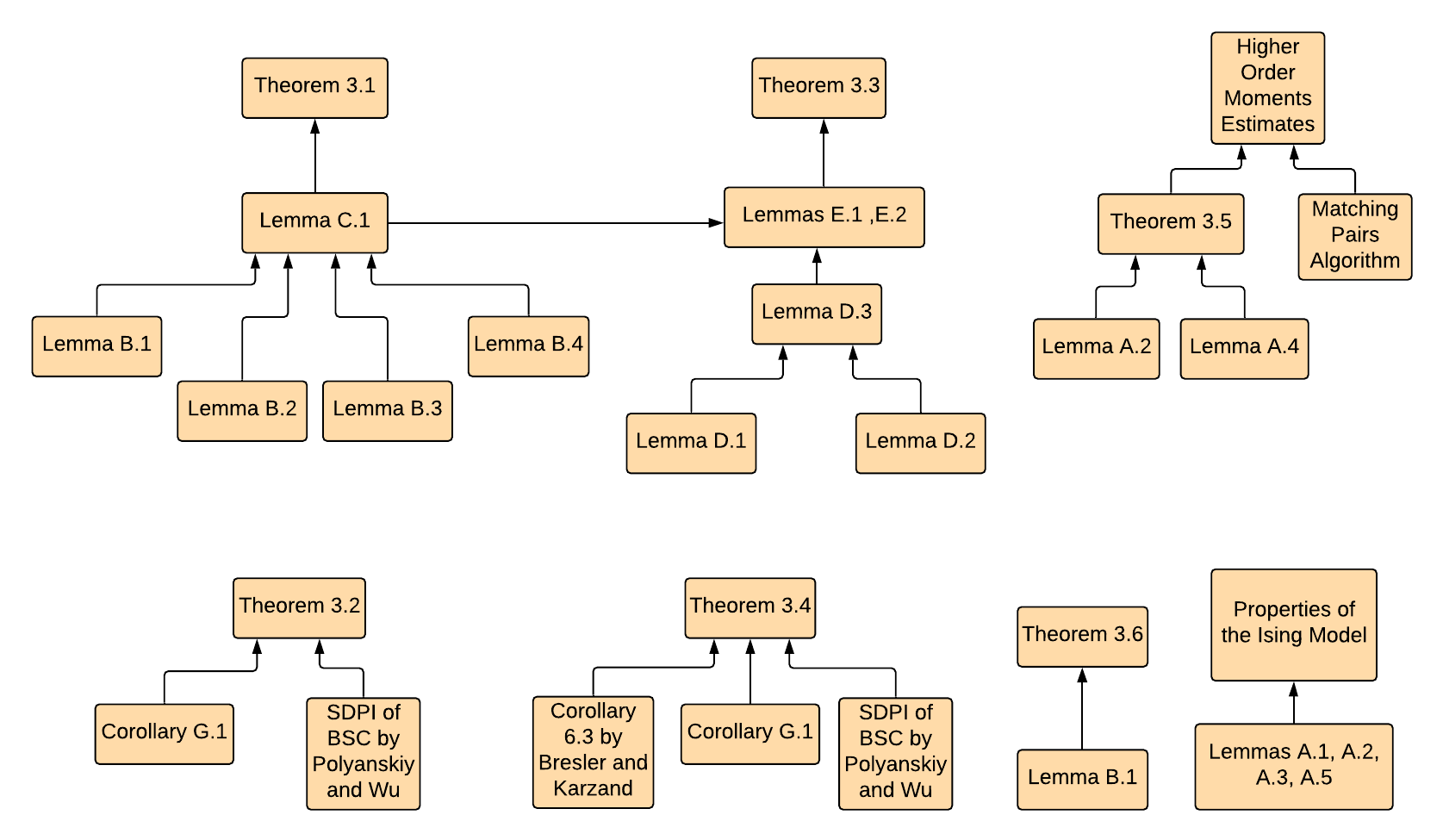}
     \caption{Stream mapping of the results}\label{fig:1-a}
   \end{figure}

For completeness, we start with some properties that hold for any distribution with support $\{-1,+1\}^p$ and tree-structured graphical model~\citep{lauritzen1996graphical}. Later we derive explicit formulas for the Ising model \eqref{eq:Ising_model_zero_external}.

\begin{lemma}\label{Tree_structured_Model_Lemma}
Any distribution $\p(\bx)$ with respect to a forest $F=\left(\mc{V},\mc{E}\right)$, where $\bx\in\{-1,1\}^{p}$ and uniform marginals $\P\left(X_{i}=\pm1\right)=1/2$, for all $i\in\mc{V}$
can be expressed as
\begin{align}
\p(\bx)=\frac{1}{2}\prod_{\left(i,j\right)\in \mc{E}} \frac{1+x_{i}x_{j}\E\left[X_{i}X_{j}\right]}{2}.
\end{align}
\end{lemma}

%\begin{comment}
%\ads{does this require an Ising model with a forest or any joint distribution? You have not defined what it means to have a (general) distribution with respect to a forest.} \kn{Fixed, no it does not require an Ising model assumption}
%\end{comment}

\begin{proof}
We prove the result for an arbitrary tree $\T=(\mc{V},\mc{E})$ and then we extend it to any forest structure by applying cuts to $\mc{E}$. The distribution factorizes according to the tree structure $\T$
and under the assumption of no external field (uniform marginal distributions), we have
\begin{align}
\P(\bX=\bx) 
= \prod_{i\in V} \p\left(x_{i}\right) 
	\prod_{ (i,j) \in \mc{E}} \frac{\p(x_{i},x_{j})}{\p(x_{i}) \p(x_{j})}
%cut & =\frac{1}{2^{p}}  \frac{ \prod_{ (i,j) \in \mc{E} } \p(x_{i},x_{j})}{\left(\frac{1}{2}\right)^{p-1}\left(\frac{1}{2}\right)^{p-1}}\nonumber \\
%cut & =2^{p-2}\prod_{\left(i,j\right)\in \mc{E}} \p\left(x_{i},x_{j}\right)\nonumber \\
 & =2^{p-2}\prod_{\left(i,j\right)\in \mc{E}} \frac{1+x_{i}x_{j}\E\left[X_{i}X_{j}\right]}{4}\label{eq:Tree_distribution_mod}\\
 &=\frac{1}{2}\prod_{\left(i,j\right)\in \mc{E}} \frac{1+x_{i}x_{j}\E\left[X_{i}X_{j}\right]}{2}.\label{eq:Tree_distribution}
\end{align} \eqref{eq:Tree_distribution_mod} holds since the joint distribution of any pair $(X_i,X_j)$ of distinct nodes $i,j\in\mc{V}$ is
	\begin{align}\label{eq:Joint_Pair}
	\p(x_i,x_j)
	=\E\left[\bds{1}_{X_i =x_i}\bds{1}_{X_j =x_j} \right] =\frac{1+x_i x_j \E[X_i X_j]}{4}.
	\end{align}
	By setting $\E[X_i X_j]=0$ for some $(i,j)\in \mc{E}$ we derive the distribution with respect to a forest generated by cutting the edge $(i,j)$ of $\T$.
\end{proof}

In Lemma~\ref{independent_products} we prove two fundamental properties of the model, the independence of the random variables $\{X_i X_j : (i,j)\in\mc{E}\}$ and the \textit{correlation decay property} (CDP). To the best of our knowledge, these properties are known but there is no reference for the corresponding proofs in the literature. 
\begin{lemma}\label{independent_products}
Let $\bX$ be a random binary vector in $\{-1,+1\}^p$ drawn according to a forest-structured distribution $\p(\cdot)$ with uniform marginal distributions on each entry $X_i$ for $i \in [p]$. Then the elements of the collection of $|\mc{E}|$ random variables $\{ X_i X_j : (i,j)\in\mc{E}\}$, are independent. Furthermore, we have 
	\begin{align}
	\E\left[X_{i}X_{j}\right]
		=\prod_{e\in\tpath\left(i,j\right)} \mu_{e},
	\label{eq:CDP_binary_trees}
	\end{align}
so the Correlation Decay Property (CDP) holds since $|\mu_{e}|\leq1$ for all $e\in\mc{E}.$ 
\end{lemma}

\begin{proof} 
 Let $(i_r)^{p}_{r=1}$ be an arbitrary permutation of $\bds{\ell}=\{1,2,\ldots,p\}$. Notice that the singletons $\{i_r \}$, $r=1,\dots,p$ form a partition of $\bds{\ell}$. Then, the set of edges $\mc{E}$ is defined as\begin{align}\label{eq:dependence_structure}
\mc{E}=\lp i_r,j_r \rp^p _{r=2},\quad \text{ and }j_1=\emptyset,\quad j_r \in \{i_1,\ldots,i_{r-1} \}\subset\bds{\ell}.  
\end{align} \eqref{eq:dependence_structure} defines a tree $\T=(\mc{V},\mc{E})$ with root the node $i_1$ (since $j_1=\emptyset$). For the first part, it is sufficient to show that for any $\{ c_r : r = 2, 3, \ldots, p\}\in \{-1,+1\}^{p-1}$, the following holds
	\begin{align}\label{eq:independent_product}
	\P\lp \bigcap^{p}_{r=2} \{x_{i_{r}} x_{j_{r}}=c_r\}\rp=\prod^{p}_{r=2} \P\lp x_{i_{r}} x_{j_{r}}=c_r \rp.
	\end{align} 
We have
%\begin{comment}
%Then it is sufficient to show that:
%	\begin{align*}
%	\P\lp \bigcap^{p}_{r=2} \{x_{i_{r}} x_{j_{r}}=c_r\}\rp=\prod^{p}_{r=2} \P\lp x_{i_{r}} x_{j_{r}}=c_r \rp
%	\end{align*}  
%Notice that 
%	\begin{align*}
%	p(x_i,x_j)
%	&= \E\left[\lp \frac{1+x_i X_i}{2} \rp\lp \frac{1+x_i X_i}{2} \rp \right]\\
%	&=\frac{1+x_i x_j \E[X_i X_j]}{4},
%	\end{align*}
%which holds for all $X_i,X_j\in\{-1,+1\}$ and allows us to find the distribution of the product $X_{i_r} X_{j_r}$: 
%	\begin{align}\label{eq:product_pmf}
%	\P \lp X_{i_r} X_{j_r}=x_{i_r} x_{j_r} \rp
%	&=\frac{1+x_{i_r} x_{j_r}\E[X_{i_r} X_{j_r}]}{2}.
%	\end{align}
%From Lemma \ref{Tree_structured_Model_Lemma} which holds only for tree-structured distributions and \eqref{eq:product_pmf} we have:
%	\begin{align}
%	\P \lp \bX=\bx \rp=\frac{1}{2} \prod^{p}_{r=2} p(x_{i_r} x_{j_r})
%	\end{align}  
%\end{comment}
	\begin{align*}
	\P\lp \bigcap^{p}_{r=2} \{X_{i_{r}} X_{j_{r}}=c_r\}\rp&=\sum_{\bx:x_{i_{r}} x_{j_{r}}=c_r|^p_{r=2}} p\lp \bx \rp\\
%cut &=\sum_{\bx:x_{i_{r}} x_{j_{r}}=c_r|^p_{r=2}} \frac{1}{2}\prod_{\left(i,j\right)\in \mc{E}} \frac{1+x_{i}x_{j}\E\left[X_{i}X_{j}\right]}{2}\numberthis \label{eq:P_I_1}\\
&=\sum_{\bx:x_{i_{r}} x_{j_{r}}=c_r|^p_{r=2}} \frac{1}{2}\prod^p_{r=2} \frac{1+x_{i_r}x_{j_r}\E\left[X_{i_{r}}X_{j_{r}}\right]}{2}\numberthis\label{eq:P_I_2}\\
&=\sum_{\bx:x_{i_{r}} =c_r x_{j_{r}}|^p_{r=2}} \frac{1}{2}\prod^p_{r=2} \frac{1+x_{i_r}x_{j_r}\E\left[X_{i_{r}}X_{j_{r}}\right]}{2}\\
&=\sum_{x_{i_1}\in\{-1,+1\}} \frac{1}{2}\prod^p_{r=2} \frac{1+c_r \E\left[X_{i_{r}}X_{j_{r}}\right]}{2}
%cut &=\prod^p_{r=2} \frac{1+c_r \E\left[X_{i_{r}}X_{j_{r}}\right]}{2}\\
=\prod^{p}_{r=2} \P\lp X_{i_{r}} X_{j_{r}}=c_r \rp,\numberthis
\end{align*} \eqref{eq:P_I_2} comes from \eqref{eq:dependence_structure} and Lemma \ref{Tree_structured_Model_Lemma} and the last from \eqref{eq:Joint_Pair}. %cut For the last equality we use the formula of the probability mass function of the variables $X_i X_j$. Notice that \begin{align*}
%cut \p(x_i,x_j)&=\E\left[\lp \frac{1+x_i X_i}{2} \rp\lp \frac{1+x_j X_j}{2} \rp \right]\\
%cut &=\frac{1+x_i x_j \E[X_i X_j]}{4}, \numberthis
%cut \end{align*} which holds for all $X_i,X_j\in\{-1,+1\}$ and gives the distribution of the product variable $X_{i_r} X_{j_r}$, which is \begin{align*}\label{eq:product_pmf}
%cut \P \lp X_{i_r} X_{j_r}=x_{i_r} x_{j_r} \rp=\frac{1+x_{i_r} x_{j_r}\E[X_{i_r} X_{j_r}]}{2}.\numberthis
%cut \end{align*} 
For the second part of the statement note that for all $i,j \in \mc{V}$ there exists a unique path $\{i, k_1 ,k_2, \ldots, k_\ell, j \}$ from $i$ to $j$. Define the variable $\bds{1}_{(i,j)}\triangleq  (X_{k_1} X_{k_1})  (X_{k_2} X_{k_2}) \ldots (X_{k_\ell} X_{k_\ell})$, which is equal to $1$ almost surely, since $\bX\in\{-1,+1\}^p$.\footnote{$\bds{1}_{(\cdot)}$ should not be confused with $\mathds{1}_{\bds{A}}$, where the last denotes the indicator function of a set $A$.} Then, we have   
	\begin{align*}
	\E [X_i X_j] 
    &= 
		\E [X_i \bds{1}_{(i,j)} X_j ]  \\
	&= 
		\E [X_i (X_{k_1} X_{k_1})( X_{k_2} X_{k_2} ) \ldots (X_{k_\ell} X_{k_\ell}) X_j ]  \\
	&=
		\E[X_i X_{k_1}] \left( \prod^{\ell-1}_{m=1} \E[X_{k_m} X_{k_{m+1}}] \right) \E[X_\ell X_j]= \prod_{e\in\tpath\left(i,j\right)} \mu_{e},\numberthis \label{indicator_step} 
\end{align*} and \eqref{indicator_step} comes from \eqref{eq:independent_product} and completes the proof.
%cut For the first equality we used the fact that the random variable $\bds{1}_{(i,j)}\triangleq  X_{k_1} X_{k_1}  X_{k_2} X_{k_2} \ldots X_{k_\ell} X_{k_\ell} $ is equal to $1$ with probability $1$. Notice that the correlation decay property \eqref{eq:CDP_binary_trees} holds for any binary tree-structured model with support $\{-1,+1\}^p$ due to \eqref{eq:independent_product}.
\end{proof}

%cut By studying a counter example for a graph with cycles, we can see that the products $X_i X_j$ for $(i,j)\in \mc{E}$ are not independent for any two edges $(i,j),(i',j')$ which participate in the same cycle so the correlation decay property does not always hold for graphs with cycles. Examples for the correlations for graphs with cycles under a positive correlation assumption are given by~\citet{daskalakis2018testing}.

\noindent The next lemma relates the pairwise correlations to the parameters of the Ising model.

\begin{lemma}\label{Tree_structured_Model_Lemma2}
 An equivalent expression of \eqref{eq:Ising_model_zero_external} is the following
   \begin{align}
	%\label{Ising_model_second_form}
	\p(\bx) 	= \frac{\prod_{\left(i,j\right)\in \mc{E}} \left[1+x_{i}x_{j}\tanh\left(\theta_{ij}\right)\right]
	}{
	\sum_{\bx} \prod_{\left(i,j\right)\in \mc{E}} \left[1+x_{i}x_{j}\tanh\left(\theta_{ij}\right)\right]
	}
	\hspace{+0.3cm} 
	\bx\in\{-1,1\}^{p}.
\end{align} Further, for a tree-structure Ising model $\E\left[X_{i},X_{j}\right]  =\tanh\left(\theta_{ij}\right),$ for all $\left(i,j\right)\in\mc{E}$.
%cut Any tree-structured Ising model distribution \eqref{eq:Ising_model_zero_external} with zero-external field can be written as
%cut	\begin{align}
%cut	\p(\bx) & =\frac{1}{2}\prod_{\left(i,j\right)\in \mc{E}} \frac{1+x_{i}x_{j}\tanh\left(\theta_{ij}\right)}{2}.
%cut	\end{align}
%cut	As a consequence from Lemma~\ref{Tree_structured_Model_Lemma} the correlations of an Ising model with respect
%cut to a tree $\T=\left(\mc{V},\mc{E}\right)$ are 
%cut \begin{align}
%cut \mu_{i,j}=\E\left[X_{i},X_{j}\right] & =\tanh\left(\theta_{ij}\right),\quad\forall\left(i,j\right)\in\mc{E}\label{eq:Tree_covariates}
%cut \end{align}
%cut and
%cut	\begin{align}
%cut	\E\left[X_{i}X_{j}\right]=\prod_{e\in\tpath\left(i,j\right)} \mu_{e} 
%cut	& =\prod_{e\in\tpath\left(i,j\right)} \tanh\left(\theta_{e}\right),
%cut	\label{eq:Covariates_prod2}
%cut	\end{align}
%cutso the correlation decay property (cdp) holds since $\mu_{e}\leq1$ for all $e\in\mc{E}$. 
\end{lemma}

\begin{proof}
We can write $\exp\left(\theta_{ij}x_{i}x_{j}\right)$ as
\begin{align*}
\exp\left(\theta_{ij}x_{i}x_{j}\right)= & \frac{\exp\left(\theta_{ij}x_{i}x_{j}\right)+\exp\left(-\theta_{ij}x_{i}x_{j}\right)}{2}+\frac{\exp\left(\theta_{ij}x_{i}x_{j}\right)-\exp\left(-\theta_{ij}x_{i}x_{j}\right)}{2}\nonumber \\
= & \frac{\exp\left(\theta_{ij}\right)+\exp\left(-\theta_{ij}\right)}{2}+x_{i}x_{j}\frac{\exp\left(\theta_{ij}\right)-\exp\left(-\theta_{ij}\right)}{2}\numberthis\label{eq:binary_prod}\\
%cut = & \frac{\exp\left(\theta_{ij}\right)+\exp\left(-\theta_{ij}\right)}{2}\left[1+x_{i}x_{j}\frac{\exp\left(\theta_{ij}\right)-\exp\left(-\theta_{ij}\right)}{\exp\left(\theta_{ij}\right)+\exp\left(-\theta_{ij}\right)}\right]\\
= & \cosh\left(\theta_{ij}\right)\left[1+x_{i}x_{j}\tanh\left(\theta_{ij}\right)\right],\numberthis \label{eq:exponent_tanh}
\end{align*} \eqref{eq:binary_prod} holds because $x_ix_j\in \{-1,+1 \}$.
The partition function can be written as 
	\begin{align}
	Z\left(\theta\right) 
 	&= \sum_{\bx}\prod_{\left(i,j\right)\in \mc{E}} \exp\left(\theta_{ij}x_{i}x_{j}\right)\nonumber \\
 	&= \sum_{\bx}\prod_{\left(i,j\right)\in \mc{E}}\cosh\left(\theta_{ij}\right)\left[1+x_{i}x_{j}\tanh\left(\theta_{ij}\right)\right]\nonumber \\
 	&= \prod_{\left(i,j\right)\in \mc{E}} \cosh\left(\theta_{ij}\right)\sum_{\bx} \prod_{\left(i,j\right)\in \mc{E}} \left[1+x_{i}x_{j}\tanh\left(\theta_{ij}\right)\right]= 2^{p}\prod_{\left(i,j\right)\in \mc{E}} \cosh\left(\theta_{ij}\right)\label{eq:partition_tanh}.
	\end{align}
Notice that $\sum_{\bx} \prod_{\left(i,j\right)\in \mc{E}} \left[1+x_{i}x_{j}\tanh\left(\theta_{ij}\right)\right]=2^p$ under the tree-structure assumption. Then %cutThe Ising model distribution with zero external field and with respect to a tree structure is
	\begin{align*}
	\P(\bX  =\bx)=\frac{\prod_{\left(i,j\right)\in \mc{E}} \exp\left(\theta_{ij}x_{i}x_{j}\right)}{Z\left(\theta\right)} &=\frac{\prod_{\left(i,j\right)\in \mc{E}} \cosh\left(\theta_{ij}\right)\left[1+x_{i}x_{j}\tanh\left(\theta_{ij}\right)\right]}{2^{p}\prod_{\left(i,j\right)\in \mc{E}}\numberthis\label{eq:Tree_Ising_Model1} \cosh\left(\theta_{ij}\right)}\\
%cut & =\frac{\prod_{\left(i,j\right)\in \mc{E}} \left[1+x_{i}x_{j}\tanh\left(\theta_{ij}\right)\right]}{2^{p}}\\
 & = \frac{1}{2}\prod_{\left(i,j\right)\in \mc{E}} \frac{1+x_{i}x_{j}\tanh\left(\theta_{ij}\right)}{2},\numberthis \label{eq:Tree_Ising_Model}
\end{align*} \eqref{eq:exponent_tanh} and \eqref{eq:partition_tanh} give \eqref{eq:Tree_Ising_Model1} and $|\mc{E}|=p-1$ gives \eqref{eq:Tree_Ising_Model}.
%Now by combining \eqref{eq:Tree_Ising_Model} and the fact that the
%structure is a tree \eqref{eq:Tree_distribution} we have:  $$\E\left[X_{i}X_{j}\right]=\tanh\left(\theta_{ij}\right) \text{ for all } (i,j)\in\mc{E}.$$  
Finally \begin{align}
\E\left[X_{i}X_{j}\right]\overset{\eqref{eq:Ising_model_zero_external}}{=}\frac{\partial\ln Z\left(\theta\right)}{\partial\theta_{ij}}\overset{\eqref{eq:partition_tanh}}{=}\frac{\partial\ln\left[ 2^{p}\prod_{\left(i,j\right)\in \mc{E}} \cosh\left(\theta_{ij}\right)\right]}{\partial\theta_{ij}}=\tanh(\theta_{ij}),\quad \forall(i,j)\in\mc{E},\label{eq:covariates_of_edges}
\end{align}%cut under the tree assumption we have $Z(\theta)\overset{\eqref{eq:partition_tanh}}{=}2^{p}\prod_{\left(i,j\right)\in \mc{E}} \cosh\left(\theta_{ij}\right)$ and \eqref{eq:covariates_of_edges} gives\begin{align*}
%cut \E\left[X_{i}X_{j}\right]&=\frac{\partial\ln\left[ 2^{p}\prod_{\left(i,j\right)\in \mc{E}} \cosh\left(\theta_{ij}\right)\right]}{\partial\theta_{ij}}=\tanh(\theta_{ij}),\quad  (i,j)\in\mc{E}.\numberthis
%cut\end{align*}
%cut The proof is complete.
and the latter gives the second part of the Lemma.
\end{proof}

\begin{lemma}\label{disjoint_paths}
Let $\mc{V'}$ be a set of nodes such that $\mc{V'}\subset \mc{V}$ and $|\mc{V'}|\in 2\mbb{N}$. Then it exists a set $\mc{C}_{\T}(\mc{V'})$ of $|\mc{V'}|/2$ pairs of nodes of $\mc{V'}$, such that any two distinct pairs $(w,w')$, $(v,v')$ in $\mc{C}_{\T}(\mc{V'})$ are pairwise disjoint (their paths have no commons edge), that is,
	\begin{align}
   \tpath_{\T} (w,w')\cap \tpath_{\T} (v,v')=\emptyset,\quad \forall (w,w'),(v,v')\in\mc{C}_{\T}(\mc{V'}):\text{ }(w,w')\not\equiv (v,v').
    \end{align}
\end{lemma}  

    % \begin{comment}
    %  \kn{ We prove the above by using induction: Assume that $\underset{(w,w')\in \mc{C}_{\T}(\mc{V'})}{\cap} \tpath_{\T} (w,w')\neq\emptyset$ then there are at least two paths with at least one common edge. When there are exactly two edge-disjoint paths $\tpath(w,w')\cap\tpath(z,z')=\tpath(w',z)\neq\emptyset$ (without loss of generality), there is always a choice of $\tpath(w,z)$ and $\tpath(z',w')$ or $\tpath(w,z')$ and $\tpath(z,w')$ where one of the choices gives edge-disjoint paths. The uniqueness of the new choice comes from the fact that there is a unique path between any two nodes in a tree. Now assume that we can always find $k/2-1$ edge-disjoint paths with endpoints $k-2$ different nodes $\mc{V'}_{k-2}$ (if they are not we make them disjoint) and we want to prove that by adding one more path with endpoints $w,w'\notin \mc{V'}_{k-2}$ we can always find a collection of $k/2$ edge-disjoint paths. Thus we consider one more path which may have common edges with exactly $m\leq k/2 -1$ paths. By permuting the endpoints of the $m+1$ paths we derive a collection of $k/2$ paths such that all the paths are disjoint. This is always possible since the path between any two nodes in the tree is unique.}
    % \end{comment}

\begin{proof}
We prove the existence of $\mc{C}_{\T}(\mc{V'})$ by contradiction. Assume that the two distinct paths $\tpath_{\T}(w,u')$, $\tpath_{\T}(u,w')$ share at least one edge. Let their common sub-path be $\tpath_{\T}(z,z')$, Figure \ref{fig:2-a} and note that $z$ and $z'$ do not necessarily differ from $w,w',u,u'$. Notice that the common sub-path is unique (acyclic graph). Then we can always consider the permutation of the endpoints which gives the edge-disjoint paths $\tpath_{\T}(w,u)$ and $\tpath_{\T}(w',u')$. Now the paths $\tpath_{\T}(w,u)$ and $\tpath_{\T}(w',u')$ are disjoint, however it is possible that one of them or both, contain sub-paths with common edges. Then, we similarly proceed by removing the common sub-paths as previously. The set of common edges strictly decreases through the process, which terminates when there are only paths with no common edge.
\end{proof} 
\begin{comment}
\textbf{Uniqueness:} Assume that there exist at least two distinct collections $\mc{C}^{(1)}_{\T}(\mc{V'}), \mc{C}^{(2)}_{\T}(\mc{V'})$
\end{comment}

\vspace{-0.7cm}   

\begin{figure}[!ht]
     \centering
    \includegraphics{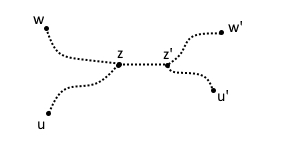}
     \caption{Proof of the existence of $\mc{C}_{\T}(\mc{V'})$, Lemma \ref{disjoint_paths}}\label{fig:2-a}
   \end{figure}
    
%    \begin{comment}
%    The set of the edges: \begin{align}
%    \mc{CP}_{\T}(V')\triangleq\underset{(w,w')\in \mc{C}_{\T}(\mc{V'})}{\cup} \tpath_{\T} (w,w')
%   \end{align} is unique. As a consequence of (I) the sets $ \tpath_{\T} (w,w')$ are a partition of $\mc{CP}_{\T}(V')$. 
%   Again using contradiction we have: Assume that there are two sets $\mc{CP}_{\T}_{1}(V') $ and $\mc{CP}_{\T}_{2}(V') $ such that $\mc{CP}_{\T}_{1}(V') \triangle\mc{CP}_{\T}_{2}(V')\neq \emptyset $, then there is at least one path $\tpath_{\T}(w,w')\subset \mc{CP}_{\T}_{1}(V')$ and one path $\tpath_{\T}(w,z)\subset \mc{CP}_{\T}_{2}(V')$ such that $\tpath_{\T}(w,w')\triangle\tpath_{\T}(w,z) \neq \emptyset$, where $\triangle$ is the symmetric difference. Then, there is a subset (sub-path) of $\tpath_{\T}(w,w')$ which belongs in $\tpath_{\T}(w,z)$ which is a contradiction from (I).
%     \end{comment} 

\begin{theorem}[Theorem \ref{thm:binary:isserlis}]
Assume $\bX\sim \p(\bx)\in \isingtree$, $\{i_1,i_2,\ldots,i_k\}\subset \mc{V}$, then
\begin{align}\label{eq:isserlis_proof}
\E\left[X_{i_1}X_{i_2}\ldots X_{i_k}\right]=\begin{cases}\prod_{e\in \CP (i_1,i_2,\ldots,i_k)}\mu_e, & \forall k\in 2\mbb{N}\\
0, & \forall k\in 2\mbb{N}+1.
\end{cases}
\end{align}
\end{theorem}

Recall that the set of edges $\CP (i_1, \ldots, i_k)$ is a collection of $k/2$ edge-disjoint paths with endpoints pairs of the nodes $i_1,\ldots,i_k$ for each path. Given a tree structure $\T$, $\CP (i_1, \ldots, i_k)$ is found by running Algorithm \ref{alg:matching_pairs} on $\T$.

\begin{proof}%cut We prove each case separately:
\underline{Even $k$.} We proceed by showing that the Algorithm \ref{alg:matching_pairs} returns the unique set $\CP$. When $k=2$ the expression is proved in Lemma \ref{independent_products}. For $k>2$ we proceed by using Lemmas \ref{independent_products} and \ref{disjoint_paths}. For all $i,j \in \mc{V}$ there exists a unique path $\{i, k_1 ,k_2, \ldots, k_\ell, j \}$ from $i$ to $j$. Define as previously the variable $\bds{1}_{(i,j)}\triangleq  (X_{k_1} X_{k_1})  (X_{k_2} X_{k_2}) \ldots (X_{k_\ell} X_{k_\ell})$, which is equal to $1$ almost surely, and define the set of nodes $\mc{V'}\triangleq \{i_1,i_2,\ldots,i_k\}$. %cut If a collection of pairs of nodes $\mc{C}_{\T}(\mc{V'})$ forms a set of edge-disjoint paths, then the random variables $X_i \bds{1}_{(i,j)} X_j $ are independent for all the pairs $(i,j)\in\mc{C}_{\T}(\mc{V'})$.
Without loss of generality we assume that the variables in the product $X_{i_1}X_{i_2}\ldots X_{i_k}$ are ordered such such that the pairs $X_{i_{j}},X_{i_{j+1}}$ for all $j\in\{ 1,3,5,\ldots,k-1 \}\triangleq [k-1]^{\text{odd}}$ form edge-disjoint paths (Lemma \ref{disjoint_paths}), in other words 
	\begin{align}\label{eq:paths_intersection} 
	\tpath(i_{j},i_{j+1})\cap\tpath(i_{j'},i_{j'+1})=\emptyset, \forall j\neq j'\in  [k-1]^{\text{odd}}.
	\end{align}
Then, we have 
	\begin{align}
	\E\left[X_{i_1}X_{i_2}\ldots X_{i_k}\right]
	&=\E\left[X_{i_1} \mathbbm{1}_{(i_1,i_2)} X_{i_2}
		X_{i_3}\mathbbm{1}_{(i_3,i_4)} X_{i_4}
		\cdots X_{i_{k-1}}\mathbbm{1}_{(i_{k-1},i_k)} X_{i_k}
		\right]\label{eq:hom1}\\
	&= 
		\prod^{}_{j\in [k-1]^{\text{odd}}} \E[X_{i_{j}} \mathbbm{1}_{(i_j,i_{j+1})} X_{i_{j+1}}]\label{eq:hom2}\\
	&= 
		\prod^{}_{j\in [k-1]^{\text{odd}}} 
			\prod_{e\in\tpath\left(i_j,i_{j+1}\right)} \mu_{e}\label{eq:hom3} \\
	&= 
		\prod_{e\in \mc{CP}_{\T}(i_1,i_2,\ldots,i_k)} \mu_e,
\end{align}
where \eqref{eq:hom1} and \eqref{eq:hom2} come from \eqref{eq:CDP_binary_trees}, and \eqref{eq:hom3} holds because of \eqref{eq:paths_intersection}.
%\end{proof}

%\begin{comment}
%\begin{proof}
%We will use induction. Assuming that it holds for a distribution $P_{|\mc{E}|=k-1}(\bx^{1:k})$
%	\begin{align*}
%	\E_{P_{|\mc{E}|=k+1}} \left[ X_{i_1}X_{i_2}\ldots X_{i_{k+2}}\right]
%	&= \sum_{\bx } x_{i_1}x_{i_2}\ldots x_{i_{k+2}}\frac{(1+x_l x_{k+1}\tanh(\theta_{\ell, k+1}))(1+x_{l'} x_{k+2}\tanh(\theta_{\ell', k+2}))}{4}P_{|\mc{E}|=k+1}(\bx^{1:k+2})\\
%	&= \E\left[ X_\ell X_{k+1} \right] \E\left[ X_{\ell'} X_{k+2} \right] \E\left[X_{i_1}X_{i_2}\ldots X_{i_k}\right]\\
%	&= \E\left[X_{i_1}X_{i_2}\ldots X_{i_k}\right] 
%		\prod_{e\in\tpath\left(\ell,k+1\right)} \mu_{e}
%		\prod_{e\in\tpath\left(\ell',k+2\right)} \mu_{e}
%\end{align*}
%\end{proof}
%\end{comment}

%\begin{lemma}\label{moment_odd}
%Let $\bX\in\{-1,+1\}^p$ be a binary random vector with distribution $\p(\bx)$ with respect to forest and uniform marginals. Then\begin{align*}
%\E\left[X_{i_1}X_{i_2}\ldots X_{i_k}\right]=0, \text{ when } k \text{ is odd.} 
% \end{align*}
%\end{lemma}

%\begin{proof}

\underline{Odd $k$.} Lemma~\ref{Tree_structured_Model_Lemma} gives $\p(\bx)=2^{-p}\prod_{\left(i,j\right)\in \mc{E}} (1+x_{i}x_{j}\E\left[X_{i}X_{j}\right]).$ Then %\begin{align}\label{eq:moment}
%M_t (\bX)&=\E[e^{t^T \bX}]=\sum_{\bx\in\{-1,+1\}^p}\p(\bx)e^{t^T \bx}\quad %\text{where } t\in\mbb{R}^p
%\end{align} and \begin{align}\label{eq:derivative_moment}
%\frac{\partial^k M_t (\bX)}%{\partial^{k_1}t_1\partial^{k_2}t_2\ldots\partial^{k_p}t_p}=\E\left[X^{k_1}_1X^{k_%2}_2\ldots X^{k_p}_p e^{t^T \bx}\right]
%\end{align} From \eqref{eq:moment}, \eqref{eq:derivative_moment} we have:\begin{align*}
%\E\left[X^{k_1}_1X^{k_2}_2\ldots X^{k_p}_p \right]&=\frac{\partial^k M_t (\bX)}%{\partial^{k_1}t_1\partial^{k_2}t_2\ldots\partial^{k_p}t_p}\Big|_{t=\boldsymbol{0}}\\
%&=\frac{1}{2}\sum_{\bx\in\{-1,+1\}^p}x^{k_1}_1x^{k_2}_2\ldots %x^{k_p}_p\prod_{\left(i,j\right)\in E} \frac{1+x_{i}x_{j}E\left[X_{i}X_{j}\right]}{2}\\
%&=0  
%\end{align*}  
\begin{align*}
\E\left[X_{i_1}X_{i_2}\ldots X_{i_k} \right]
&=\frac{1}{2}\sum_{\bx\in\{-1,+1\}^{i_k}}x_{i_1}x_{i_2}\ldots x_{i_k}\prod_{\left(i,j\right)\in \mc{E}} \frac{1+x_{i}x_{j}\E\left[X_{i}X_{j}\right]}{2}=0,  \numberthis
\end{align*}  
gives the second part of \eqref{eq:isserlis_proof}.
\end{proof}

\begin{lemma}
\label{MI_function_of_corr} The mutual information of $X_i,X_{j}\in\{-1,+1\}$ is symmetric function of the correlation $\E\left[X_{i}X_{j}\right]$ and increasing with respect to $\left|\E\left[X_{i}X_{j}\right]\right|$,
\begin{align}
I\left(X_{i},X_{j}\right)&= \frac{1}{2} \log_{2} \lp \lp 1-\E\left[X_{i}X_{j}\right]\rp^{1-\E\left[X_{i}X_{j}\right]} \lp 1+\E\left[X_{i}X_{j}\right]\rp^{1+\E\left[X_{i}X_{j}\right]}\rp.\label{eq:mutual_info}
\end{align}
%cut $I\left(X_{i},X_{j}\right)$ is a symmetric function of $\E\left[X_{i}X_{j}\right]$ and increasing with respect to $\left|\E\left[X_{i}X_{j}\right]\right|$.
\end{lemma} The proof can be derived through the definition of $I\left(X_{i},X_{j}\right)$ and the expression \eqref{eq:Joint_Pair}, under the assumption of uniform marginal distributions.

%cut \begin{proof} The definition of mutual information gives

%cut \begin{align*}I\left(X_{i},X_{j}\right)= & \sum_{x_{i},x_{j}} \p\left(x_{i},x_{j}\right)\log_{2}\frac{\p\left(x_{i},x_{j}\right)}{\p\left(x_{i}\right)p\left(x_{j}\right)}\nonumber \\
%cut = & \sum_{x_{i},x_{j}} \frac{1+x_{i}x_{j}\E\left[X_{i}X_{j}\right]}{4}\log_{2}\left(1+x_{i}x_{j}\E\left[X_{i}X_{j}\right]\right)\numberthis\label{eq:mutual1} \\
%cut = & \frac{1+\E\left[X_{i}X_{j}\right]}{2}\log_{2}\left(1+\E\left[X_{i}X_{j}\right]\right)+\frac{1-\E\left[X_{i}X_{j}\right]}{2}\log_{2}\left(1-\E\left[X_{i}X_{j}\right]\right)\nonumber \\
%cut  =& \frac{1}{2} \log_{2} \lp \lp 1-\E\left[X_{i}X_{j}\right]\rp^{1-\E\left[X_{i}X_{j}\right]} \lp 1+\E\left[X_{i}X_{j}\right]\rp^{1+\E\left[X_{i}X_{j}\right]}\rp,\numberthis
%cut \end{align*}and \eqref{eq:Joint_Pair} gives \eqref{eq:mutual1}. 
%cut \end{proof}

%%%
%%% APPENDIX ON CORRELATION ESTIMATES
%%%
\section{Bounding the Probability of Mis-Estimating Correlations}

The following lemma bounds the probability that the estimated pairwise correlations in the graph deviate from their true values. This follows from standard concentration of measure arguments.
\begin{lemma}\label{Lem_Ecorrn}
Fix $\delta > 0$. Then for any $\neps>0$, if 
    \begin{align}
    \nn \geq 2\log\left(p^{2}/\delta\right)/\neps^2,
    \end{align}
then the event $\Ecorrn$ defined in \eqref{eq:event_corr_nois-1} holds with high probability:
\begin{align}
\P \lp \Ecorrn \rp \geq 1 - \delta 
= 1 - p^{2}\exp\left(\frac{-\nn \neps ^{2}}{2}\right).
\label{eq:ecorr:whp}
\end{align}
\end{lemma}

\begin{proof}
Let $\nZ^{(i)}$ be
the $i_{\text{th}}$ sample of $\nZ=\nX_{w}\nX_{\bar{w}}=\Nw X_{w}\Nwb X_{\bar{w}}$.
Then $\nmue_{i,j}=\frac{1}{\nn}\sum_{i=1}^{\nn}\nZ^{(i)}=\frac{1}{\nn}\sum_{i=1}^{\nn}\Nw ^{(i)}X_{w}^{(i)}\Nwb ^{(i)}X_{\bar{w}}^{(i)}$ for all $i\neq j\in \mc{V}$. Then Hoeffding's inequality and union bound over all pairs of nodes ${p \choose 2}<p^2/2$ give \eqref{eq:ecorr:whp}.
\end{proof}

For the rest of the paper we consider $\neps=\sqrt{2\log\left(2p^{2}/\delta\right)/\nn}$, which satisfies Lemma \ref{Lem_Ecorrn}. We apply Lemmata \ref{lemma_Estrong3}, \ref{lemma_Estrong1}, \ref{lemma_Estrong2} to Lemma \ref{lemma_Estrong} to bound the required number of samples for exact structure recovery using noisy observations of the hidden model. To analyze the error event we use the ``Two trees lemma'' of~\citet[Appendix F, supplementary material]{bresler2020learning}. Informally, if two maximum spanning trees $\T$, $\T'$ differ in how a pair of nodes are connected then there exists at least one edge in $\mc{E}_{\T}$ which does not exist in $\mc{E}_{\T'}$ and vice versa. Lemma \ref{lemma_Estrong3} characterizes errors in the Chow-Liu in terms of correlations.

\begin{lemma}\label{lemma_Estrong3}
Suppose the error event $\{\T\neq \TCLn\}$ holds and let $f\triangleq \left(w,\bar{w}\right)$ be an edge such that $f\in\T$ and $f\notin\TCLn$. Then there exists an edge $g\triangleq (u,\bar{u})\in\TCLn$ and $g\notin\T$ such that $f\in\tpath_{\T}\left(u,\bar{u}\right)$
and $g\in\tpath_{\TCLn}\left(w,\bar{w}\right)$ and
\begin{align}
\left(\sum_{i=1}^{\nn}\Zfsample\right)\left(\sum_{i=1}^{\nn}\Yfsample\right) & <0,
\end{align}
where $\Zf\triangleq\nX_{w}\nX_{\bar{w}}-\nX_{u}\nX_{\bar{u}}$ and $\Yf\triangleq\nX_{w}\nX_{\bar{w}}+\nX_{u}\nX_{\bar{u}}$.
\end{lemma}

\begin{proof}
Using similar approaches to the procedures as in~\citep[Lemmata 8.2, 8.3]{bresler2020learning} we have that the condition $\left|\nmue_{f}\right|\leq\left|\nmue_{g}\right|$ implies
\begin{align*}
0 & \geq\left|\nmue_{f}\right|^{2}-\left|\nmue_{g}\right|^{2}\\
 & =\left(\nmue_{f}-\nmue_{g}\right)\left(\nmue_{f}+\nmue_{g}\right)\\
 & =\frac{1}{\nn^{2}}\left(\sum_{i=1}^{\nn}\Nw ^{(i)}X_{w}^{(i)}\Nwb ^{(i)}X_{\bar{w}}^{(i)}-\Nu ^{(i)}X_{u}^{(i)}\Nub ^{(i)}X_{\bar{u}}^{(i)}\right)\\&\qquad\qquad\qquad\times\left(\sum_{i=1}^{\nn}\Nw ^{(i)}X_{w}^{(i)}\Nwb ^{(i)}X_{\bar{w}}^{(i)}+\Nw ^{(i)}X_{u}^{(i)}\Nub ^{(i)}X_{\bar{u}}^{(i)}\right)\\
 & =\frac{1}{\nn^{2}}\left(\sum_{i=1}^{\nn}\Zfsample\right)\left(\sum_{i=1}^{\nn}\Yfsample\right).\numberthis
\end{align*} 
%The random variables $\Zfsample$, $\Yfsample$ are functions of observations of the observable variables (noisy observations), while the method of the analysis by~\citet{bresler2018learning} remains the same.
\end{proof}

Setting 
    \begin{align}
    \neps &\triangleq \sqrt{\frac{2\log\left(2p^{2}/\delta\right)}{\nn}} \label{eq:neps:def} \\
    \ntau &\triangleq \frac{ 4 \neps \sqrt{ 1 - (1 - 2q)^4 \tanh \beta} }{1 - \tanh \beta} \label{eq:ntau:def} \\
    \mu_e &\triangleq \E\left[X_{w}X_{\bar{w}}\right],
    \end{align} 
we have that $\mu_{A}$ is defined through the following relationship
    \begin{align}
    \E\left[X_{w}X_{\bar{w}}-X_{u}X_{\bar{u}}\right] = \mu_{e} (1 - \mu_{A}),
    \end{align} and
    \begin{align}
        \nmu_A &\triangleq (1 - 2 q)^4 \mu_A.
    \end{align}
In Lemmata \ref{lemma_Estrong1}, \ref{lemma_Estrong2} we derive two concentration of measure inequalities for the variables $\Zfsample$, $\Yfsample$. In fact, we have that the event 
\begin{align}
\Err_{Z} \triangleq \left\{
    \left| \sum_{i=1}^{\nn}
        \Zsample-\nn \E\left[\Z\right]
    \right|
\leq 
    \nn \max\left\{     
        8\neps^{2},
        4\neps\sqrt{1-\nmu_{A}}\right\}: \forall e \in \mc{E} 
        \text{ and } 
        \forall u,\bar{u}\in \mc{V} 
        \right\}
        \label{eq:Ez} 
\end{align} 
happens with probability at least $1-\frac{\delta'}{2}$ and the event
\begin{align}
&\Err_{M}\triangleq \left\{ 
    \left| \sum_{i=1}^{\nn}
        \Ysample-\nn \E\left[\Y\right]
    \right|
    \leq  
    \nn\max\left\{     
        8\neps^{2},
        4\neps\sqrt{1+\nmu_{A}}\right\}: \forall e \in \mc{E} 
        \text{ and } 
        \forall u,\bar{u}\in \mc{V} \right\}
    \label{eq:Em} 
\end{align} 
happens with probability at least $1-\frac{\delta''}{2}$. The parameters $\neps$ and $\mu_A$, defined below, are decreasing functions of $\nn$. Finally, we apply the union bound to guarantee that the event $E_{Z}\cup E_{M}$ happens with probability at least $1-\delta$, where $\frac{\delta'}{2}+\frac{\delta''}{2} \leq 2\max\{\frac{\delta'}{2},\frac{\delta''}{2}\}\triangleq \delta$. The union bound is first applied over all tuples $(w,\bar{w},u,\bar{u})$ in Lemmata \ref{lemma_Estrong1} and \ref{lemma_Estrong2} and then for the events $E_{Z}$ and $E_{M}$. 

\begin{lemma}
\label{lemma_Estrong1}
Fix $\delta>0$ and let $\neps$ be given by \eqref{eq:neps:def}. For all pairs of vertices $u, \bar{u} \in V$ and edges $e = (w,\bar{w})$ in the path $\tpath_{\T}\left(u,\bar{u}\right)$ from $u$ to $\bar{u}$,
given $\nn$ samples $\Z^{(1)},\Z^{(2)},...,\Z^{(n)}$ of
$\Z= \nX_{w}\nX_{\bar{w}}-\nX_{u}\nX_{\bar{u}}$, it is true that
\begin{align}
\P\lp \left|\sum_{i=1}^{\nn}\Zsample-\nn \E\left[\Z\right]\right| \leq \nn\max\left\{ 8\neps^{2},4\neps\sqrt{1-\nmu_{A}}\right\} \rp \geq 1-\frac{\delta}{2},
\end{align}
%where $\neps=\sqrt{2\log\left(2p^{2}/\delta\right)/\nn}$ and
where $A=\tpath_{\T}\left(u,\bar{u}\right)\backslash\left\{ e\right\}$.
\end{lemma}

\begin{proof}
The proof is an application of Bernstein's inequality. First, it is true that
\begin{align*}
\Z & =X_{w}\Nw X_{\bar{w}}\Nwb -\Nu X_{u}\Nub X_{\bar{u}}\\
 & =\Nw X_{w}\Nwb X_{\bar{w}}\left(1-\Nw X_{w}\Nwb X_{\bar{w}}\Nu X_{u}\Nub X_{\bar{u}}\right).\numberthis
\end{align*}
Then,
\begin{align*}
\E\left[\Z\right] &=\left(1-2q\right)^{2}\E\left[X_{w}X_{\bar{w}}-X_{u}X_{\bar{u}}\right] \\ &=\left(1-2q\right)^{2}\mu_{e}\left(1-\mu_{A}\right) \numberthis \label{eq:mean_Z} \\
\Var\left(\Z\right) & =\E\left[\left(\Z\right)^{2}\right]-\E\left[\left(\Z\right)\right]^{2}\\
 & =\E\left[\left(X_{w}\Nw X_{\bar{w}}\Nwb -\Nub X_{u}\Nub X_{\bar{u}}\right)^{2}\right]-\left[\left(1-2q\right)^{2}\E\left[X_{w}X_{\bar{w}}-X_{u}X_{\bar{u}}\right]\right]^{2}\\
 & =\E\left[1+1-2X_{w}\Nw X_{\bar{w}}\Nwb \Nu X_{u}\Nub X_{\bar{u}}\right]-\left(1-2q\right)^{4}\E\left[X_{w}X_{\bar{w}}-X_{u}X_{\bar{u}}\right]^{2}\\
 & =2-2\E\left[X_{w}\Nw X_{\bar{w}}\Nwb \Nu X_{u}\Nub X_{\bar{u}}\right]-\left(1-2q\right)^{4}\E\left[X_{w}X_{\bar{w}}-X_{u}X_{\bar{u}}\right]^{2}\\
 & =2-2\left(1-2q\right)^{4}\E\left[X_{w}X_{\bar{w}}X_{u}X_{\bar{u}}\right]-\left(1-2q\right)^{4}\E\left[X_{w}X_{\bar{w}}-X_{u}X_{\bar{u}}\right]^{2}\\
 & =2-2\left(1-2q\right)^{4}\mu_{A}-\left(1-2q\right)^{4}\left(\mu_{e}\left(1-\mu_{A}\right)\right)^{2}\\
 & =2-\left(1-2q\right)^{4}\left[2\mu_{A}+\mu_{e}^{2}\left(1-\mu_{A}\right)^{2}\right].\numberthis
\end{align*}
Using the expressions for the mean and the variance, we apply Bernstein's
inequality~\citep{bennett1962probability} for the noisy setting: for all $i\in [\nn]$ we have $\left|\Zsample-\E\left[\Z\right]\right|\leq M$
almost surely. Then, Bernstein's inequality gives, for all $t > 0$
\begin{align*}
%\!\!\!\P\left[\left|\sum_{i=1}^{\nn}\Zsample-\nn\E\left[\Z\right]\right|\geq t\right] & \leq2\exp\left(-\frac{t^{2}}{2\nn\Var\left(\Z\right)+\frac{2}{3}Mt}\right),\quad\forall t>0\implies\\
\P\left[\left|\sum_{i=1}^{\nn}\Zsample-\nn\E\left[\Z\right]\right|\leq t\right] & \geq1-2\exp\left(-\frac{t^{2}}{2\nn\Var\left(\Z\right)+\frac{2}{3}Mt}\right),\quad\forall t>0.\numberthis
\end{align*}
Choose a $\delta>0$ and find $t$ such that
\begin{align*}
\delta/2 &=2\exp\left(-\frac{t^{2}}{2\nn\Var\left(\Z\right)+\frac{2}{3}Mt}\right).
\end{align*}
After some algebra, we have
\begin{align*}
\log\frac{4}{\delta} & =\frac{t^{2}}{2\nn\Var\left(\Z\right)+\frac{2}{3}Mt}
\end{align*}
From this we can solve for $t$:
\begin{align*}
0 &= t^{2}- \frac{2}{3}Mt\log\frac{4}{\delta}-2\nn\Var\left(\Z\right)\log\frac{4}{\delta}
    \\
t_{1,2} &= \frac{\frac{2}{3}M\log\frac{4}{\delta}\pm\sqrt{\left(\frac{2}{3}M\log\frac{4}{\delta}\right)^{2}+8\nn\Var\left(\Z\right)\log\frac{4}{\delta}}}{2}\\
 & =\frac{1}{3}M\log\frac{4}{\delta}\pm\sqrt{\left(\frac{1}{3}M\log\frac{4}{\delta}\right)^{2}+2\nn\Var\left(\Z\right)\log\frac{4}{\delta}}.\numberthis
\end{align*}
Since $t>0$, we have, setting $M = 4$:
\begin{align}
t=& 
%\frac{1}{3}M\log\frac{4}{\delta}+\sqrt{\left(\frac{1}{3}M\log\frac{4}{\delta}\right)^{2}+2\nn\Var\left(\Z\right)\log\frac{4}{\delta}}
%\label{eq:t}\\
%&=
\frac{4}{3}\log\frac{4}{\delta}+\sqrt{\left(\frac{4}{3}\log\frac{4}{\delta}\right)^{2}+2\nn\Var\left(\Z\right)\log\frac{4}{\delta}}. \label{eq:t}
\end{align}
If the probability of the union 
\begin{align*}
\bigcup_{\forall u,\bar{u,}w,\bar{w:}\left(w,\bar{w}\right)\in\tpath_{\T}\left(u,\bar{u}\right)} & \left\{ \left|\sum_{i=1}^{\nn}\Zsample-\nn\E\left[\Z\right]\right|\geq t\right\} 
\end{align*}
is at most $\frac{\delta}{2p^{3}}$, then the union bound gives
probability at most $\frac{\delta}{2}$. Also,
\begin{align}
\Var\left(\Z\right) & =2-\left(1-2q\right)^{4}\left[2\mu_{A}+\mu_{e}^{2}\left(1-\mu_{A}\right)^{2}\right]\nonumber \\
 & =2-\left(1-2q\right)^{4}2\mu_{A}-\left(1-2q\right)^{4}\mu_{e}^{2}\left(1-\mu_{A}\right)^{2}\nonumber \\
 & \leq2-\left(1-2q\right)^{4}2\mu_{A}+0\nonumber \\
 & =2\left(1-\left(1-2q\right)^{4}\mu_{A}\right)\nonumber \\
 & =2\left(1-\nmu_{A}\right)\label{eq:noisy_variance}.
\end{align}
From \eqref{eq:t} and \eqref{eq:noisy_variance}, we have
\begin{align*}
t & =\frac{4}{3}\log\frac{4p^{3}}{\delta}+\sqrt{\left(\frac{4}{3}\log\frac{4p^{3}}{\delta}\right)^{2}+4\nn\left(1-\nmu_{A}\right)\log\frac{4p^{3}}{\delta}}\\&\leq\frac{8}{3}\log\frac{4p^{3}}{\delta}+\sqrt{4\nn\left(1-\nmu_{A}\right)\log\frac{4p^{3}}{\delta}},\numberthis\end{align*}
which implies that
\begin{align*}
t & =\nn\left(\frac{4}{3\nn}\log\frac{4p^{3}}{\delta}+\sqrt{\left(\frac{4}{3\nn}\log\frac{4p^{3}}{\delta}\right)^{2}+\frac{4}{\nn}\left(1-\nmu_{A}\right)\log\frac{4p^{3}}{\delta}}\right)\\ &\leq \nn\left(\frac{8}{3\nn}\log\frac{4p^{3}}{\delta}+\sqrt{\frac{4}{\nn}\left(1-\nmu_{A}\right)\log\frac{4p^{3}}{\delta}}\right).\numberthis
\end{align*}
Define $\neps=\sqrt{\log\left(2p^{2}/\delta\right)2/\nn}$ (as it
is defined in~\citet{bresler2020learning}), then we get
\begin{align}
t\leq & \nn\left(4\neps^{2}+2\neps\sqrt{1-\nmu_{A}}\right)\leq \nn\max\left\{ 8\neps^{2},4\neps\sqrt{1-\nmu_{A}}\right\}.
\end{align}
\end{proof}

Lemma \ref{lemma_Estrong2} gives the concentration of measure bound for the event $E_{M}$ defined in \eqref{eq:Em}.   

%\ads{the $\neps$ above the previous equation looks different than the $\neps$ in the statement of the Lemma below. Should the $\log(2p^2/\delta)$ be in the numerator?}

\begin{lemma}\label{lemma_Estrong2}
Fix $\delta>0$ and let $\neps$ be given by \eqref{eq:neps:def}. For all pairs of vertices $u, \bar{u} \in V$ and edges $e = (w,\bar{w})$ in the path $\tpath_{\T}\left(u,\bar{u}\right)$ from $u$ to $\bar{u}$,
given $\nn$ samples $\Y^{(1)},\Y^{(2)},...,\Y^{(n)}$ of
$\Y= \nX_{w}\nX_{\bar{w}}+\nX_{u}\nX_{\bar{u}}$, it is true that
\begin{align}
\P\lp \left|\sum_{i=1}^{\nn}\Ysample-\nn \E\left[\Y\right]\right| \leq \nn\max\left\{ 8\neps^{2},4\neps\sqrt{1+\nmu_{A}}\right\} \rp \geq 1-\frac{\delta}{2},
\end{align}
$A=\tpath_{\T}\left(u,\bar{u}\right)\backslash\left\{ e\right\}$.
\end{lemma}

\begin{proof}
Similarly to the prior Lemma, we calculate the mean and the variance as
\begin{align*}
\E\left[\Y\right] &=\left(1-2q\right)^{2}\E\left[X_{w}X_{\bar{w}}+X_{u}X_{\bar{u}}\right] =\left(1-2q\right)^{2}\mu_{e}\left(1+\mu_{A}\right) \numberthis \label{eq:mean_Y} \\
\Var\left(\Y\right) & =\E\left[\left(\Y\right)^{2}\right]-\E\left[\left(\Y\right)\right]^{2}\\
 & =\E\left[\left(X_{w}\Nw X_{\bar{w}}\Nwb +\Nu X_{u}\Nub X_{\bar{u}}\right)^{2}\right]-\left[\left(1-2q\right)^{2}\E\left[X_{w}X_{\bar{w}}+X_{u}X_{\bar{u}}\right]\right]^{2}\\
 & =\E\left[1+1+2X_{w}\Nw X_{\bar{w}}\Nwb \Nu X_{u}\Nub X_{\bar{u}}\right]-\left(1-2q\right)^{4}\E\left[X_{w}X_{\bar{w}}+X_{u}X_{\bar{u}}\right]^{2}\\
 & =2+2\E\left[X_{w}\Nw X_{\bar{w}}\Nwb \Nu X_{u}\Nub X_{\bar{u}}\right]-\left(1-2q\right)^{4}\E\left[X_{w}X_{\bar{w}}+X_{u}X_{\bar{u}}\right]^{2}\\
 & =2+2\left(1-2q\right)^{4}\E\left[X_{w}X_{\bar{w}}X_{u}X_{\bar{u}}\right]-\left(1-2q\right)^{4}\E\left[X_{w}X_{\bar{w}}+X_{u}X_{\bar{u}}\right]^{2}\\
 & =2+2\left(1-2q\right)^{4}\mu_{A}-\left(1-2q\right)^{4}\left(\mu_{e}\left(1+\mu_{A}\right)\right)^{2}\\
 & =2+\left(1-2q\right)^{4}\left[2\mu_{A}-\mu_{e}^{2}\left(1+\mu_{A}\right)^{2}\right].\numberthis
\end{align*}
By applying Bernstein's inequality and we get that for any $t>0$ 
\begin{align*}
%\!\!\!\!\!\P\left[\left|\sum_{i=1}^{\nn}\Ysample-\nn\E\left[\Y\right]\right|\geq t\right] & \leq2\exp\left(-\frac{t^{2}}{2\nn\Var\left(\Y\right)+\frac{2}{3}Mt}\right),\\
%\!\!\!
\P\left[\left|\sum_{i=1}^{\nn}\Ysample-\nn\E\left[\Y\right]\right|\leq t\right] & \geq1-2\exp\left(-\frac{t^{2}}{2\nn\Var\left(\Y\right)+\frac{2}{3}Mt}\right).%.\numberthis
\end{align*}
Similarly, we find 
\begin{align}
t\leq & \nn\left(\frac{8}{3\nn}\log\frac{4p^{3}}{\delta}+\sqrt{\frac{2}{\nn}\Var\left(\Y\right)\log\frac{4p^{3}}{\delta}}\right)
\end{align}
and 
\begin{align*}
\Var\left(\Y\right)= & 2+\left(1-2q\right)^{4}\left[2\mu_{A}-\mu_{e}^{2}\left(1+\mu_{A}\right)^{2}\right]\\
\leq & 2+\left(1-2q\right)^{4}2\mu_{A}\\
= & 2\left(1+\nmu_{A}\right).\numberthis
\end{align*}
We define $\neps\triangleq\sqrt{\log\left(2p^{2}/\delta\right)2/\nn}$,
then
\begin{align}
t\leq & \nn\left(4\neps^{2}+2\neps\sqrt{1+\nmu_{A}}\right)\leq \nn\max\left\{ 8\neps^{2},4\neps\sqrt{1+\nmu_{A}}\right\},
\end{align}
which completes the proof.
\end{proof}

% page break for readability
%\newpage

\section{Recovering Strong Edges}

In Lemma~\ref{lemma_Estrong}, we define the set of strong edges for the hidden model and show that the event $\Estrongn$ defined in \eqref{eq:event_strong_nois-1} occurs with high probability. That is, only the strong edges are guaranteed to exist in the estimated structure $\TCLn$
%There is a threshold $\ntau\geq (1-2q)^2 \tau$, which determines the sufficiently strong edges, similarly to the noiseless setting by~\citet{bresler2018learning}. 
We also find a lower bound for the necessary number of samples for exact structure recovery. In fact we have $\nn\geq n$, as expected. Our bounds coincide with the noiseless case~\citep{bresler2020learning} by setting the noise level $q = 0$.
%By setting $q=0$ (then the probability to flip a bit equals to zero) we derive the exact expressions for the threshold $\tau$ and the sufficient number of samples defined in~\citet{bresler2018learning}. 
%Also in Lemma \ref{lemma_Estrong} we show that event $\Estrongn$ defined in \eqref{eq:event_strong_nois-1} happens with probability at least $1-\delta$. Under the event $\Estrongn$ only the strong edges are guaranteed to exist in the estimated structure $\TCLn$.

\begin{lemma}\label{lemma_Estrong}
Fix $\delta\in (0,1)$, and let $\neps=\sqrt{2\log\left(2p^{2}/\delta\right)/\nn}$, for any $\nn>0$. Consider the set of strong edges
    \begin{align}
      \ntau\triangleq\frac{4\neps\sqrt{1-\left(1-2q\right)^{4}\tanh\beta}}{\left(1-\tanh\beta\right)} \text{ and }\left\{ \left(i,j\right)\in\mc{E}_{\T}:\left|\tanh\theta_{ij}\right|\geq\frac{\ntau}{\left(1-2q\right)^{2}}\right\}.
    \end{align}
%% Lemmata \ref{lemma_Estrong3}, \ref{lemma_Estrong1} and \ref{lemma_Estrong2} imply 
Then, the Chow-Liu algorithm recovers the strong edges with probability at least $1-\delta$. In other words, it is true that
\begin{align}
\P \left[ \Estrongn \right] \geq 1-2p^{2}\exp \lp -\frac{\nn\neps^{2}}{2}  \rp.
\end{align}
\end{lemma}

\begin{proof} From Lemma \ref{lemma_Estrong3}, if there is an error then for an edge $f$ not recovered in the tree $\TCLn$, we have 
    \begin{align*}
    \left(\sum_{i=1}^{\nn}\Zfsample\right)\left(\sum_{i=1}^{\nn}\Yfsample\right)<0
    \end{align*}
Therefore one of the sums must be negative. Expanding, one of the two following inequalities must hold:
\begin{align*}
\left|\sum_{i=1}^{\nn}Z_{f,u,\bar{u}}^{(i)}-\nn \E\left[Z_{f,u,\bar{u}}^{(i)}\right]\right| & \geq \nn \E\left[\Zfsample\right]
    \\ 
\left|\sum_{i=1}^{\nn}Y_{f,u,\bar{u}}^{(i)}-\nn \E\left[Y_{f,u,\bar{u}}^{(i)}\right]\right| & \geq \nn \E\left[\Yfsample\right].
\end{align*}
In addition, \eqref{eq:mean_Z}, \eqref{eq:mean_Y}, Lemma \ref{lemma_Estrong1} and Lemma \ref{lemma_Estrong2} give the following pairs of inequalities:
\begin{align*}
\left(1-2q\right)^{2}\mu_{f}\left(1-\mu_{A}\right) & \leq\max\left\{ 8\neps^{2},4\neps\sqrt{1-\nmu_{A}}\right\} 
    \\
\left(1-2q\right)^{2}\mu_{f}\left(1+\mu_{A}\right) & \leq\max\left\{ 8\neps^{2},4\neps\sqrt{1+\nmu_{A}}\right\} 
    \\
\left|\nmu_{f}\right| & \leq\left(1-\mu_{A}\right)^{-1}\max\left\{ 8\neps^{2},4\neps\sqrt{1-\nmu_{A}}\right\}         \\
\left|\nmu_{f}\right| & \leq\left(1+\mu_{A}\right)^{-1}\max\left\{ 8\neps^{2},4\neps\sqrt{1+\nmu_{A}}\right\}.
\end{align*}
Putting these together:
\begin{align*}
\left|\nmu_{f}\right| & \leq\max\left\{ \frac{8\neps^{2}}{\left(1-\mu_{A}\right)},\frac{8\neps^{2}}{\left(1+\mu_{A}\right)},\frac{4\neps\sqrt{1-\nmu_{A}}}{\left(1-\mu_{A}\right)},\frac{4\neps\sqrt{1+\nmu_{A}}}{\left(1+\mu_{A}\right)}\right\}
   \\
&\le \max\left\{ 
    \frac{8\neps^{2}}{\left(1-\mu_{A}\right)},
    \frac{4\neps\sqrt{1-\nmu_{A}}
        }{
        \left(1-\mu_{A}\right)}\right\}
    \\
&\le
    \frac{4\neps\sqrt{1-\nmu_{A}}}{\left(1-\mu_{A}\right)}.\numberthis
\end{align*}
We get the last inequality for non trivial values of the bound $\frac{8\neps^{2}}{\left(1-\nmu_{A}\right)}\leq1$
and by using the following bound
\begin{align}
\frac{8\neps^{2}}{\left(1-\mu_{A}\right)} & \leq\frac{16\neps^{2}}{\left(1-\mu_{A}\right)}\leq\frac{4\neps}{\sqrt{1-\mu_{A}}}=\frac{4\neps\sqrt{1-\mu_{A}}}{\left(1-\mu_{A}\right)}\leq\frac{4\neps\sqrt{1-\nmu_{A}}}{\left(1-\mu_{A}\right)}.
\end{align}
Finally the function $f(\mu_{A})=\frac{4\neps\sqrt{1-\nmu_{A}}}{\left(1-\mu_{A}\right)}=\frac{4\neps\sqrt{1-\left(1-2q\right)^{4}\mu_{A}^{\text{}}}}{\left(1-\mu_{A}\right)}$
is increasing with respect to $\mu_{A}$ (for all $\mu_{A}\leq1$) and $\mu_{A}\leq\tanh\beta<1$, so we have 
\begin{align}
\left|\nmu_{f}\right| & \leq\frac{4\neps\sqrt{1-\nmu_{A}}}{\left(1-\mu_{A}\right)}\leq\frac{4\neps\sqrt{1-\left(1-2q\right)^{4}\tanh\beta}}{\left(1-\tanh\beta\right)}\triangleq\ntau.\label{eq:definitionntau}
\end{align}
The weakest edge should satisfy $\left| \nmu_{f} \right| \geq  \ntau$ to guarantee
the correct recovery of the tree under the event $\Estrongn$. This yields a condition on the edge strengths:
\begin{align*}
\left| \nmu_{f} \right| & \geq  \ntau \implies\\
\left(1-2q\right)^{2}\tanh\alpha & \geq\frac{4\neps\sqrt{1-\left(1-2q\right)^{4}\tanh\beta}}{\left(1-\tanh\beta\right)}
\\
\tanh\alpha & \geq\frac{4\neps\sqrt{1-\left(1-2q\right)^{4}\tanh\beta}}{\left(1-2q\right)^{2}\left(1-\tanh\beta\right)},\quad q\in[0,\frac{1}{2}). \numberthis \label{eq:tau_noise}
\end{align*}
The last inequality gives the definition of the strong edges in the noisy scheme.
\end{proof}

Based on the definition \eqref{eq:definitionntau} we derive the following bound on $\ntau$ \begin{align*}
\ntau & =\frac{4\neps\sqrt{1-\left(1-2q\right)^{4}\tanh\beta}}{\left(1-\tanh\beta\right)}\\
 & =\frac{4\neps\sqrt{1-\left(1-8q+24q^{2}-32q^{3}+16q^{4}\right)\tanh\beta}}{\left(1-\tanh\beta\right)}\\
 & \leq4\neps\frac{\sqrt{1-\tanh\beta}+\sqrt{\left(1-3q+4q^{2}-2q^{3}\right)8q\tanh\beta}}{\left(1-\tanh\beta\right)}\\
 & \leq 4\neps e^{\beta}\left(1+ e^{\beta}\sqrt{\left(1-q\right)\left(2q^{2}-2q+1\right)8q\tanh\beta}\right)\numberthis \label{eq:upperboundontau0} \\
 & <  4\neps e^{\beta}\left(1+ 2e^{\beta}\sqrt{2\left(1-q\right)q\tanh\beta}\right).\numberthis \label{eq:upperboundontau}
\end{align*} \eqref{eq:upperboundontau0} holds because $1-\tanh(\beta)\geq e^{-2\beta}$. We will later use \eqref{eq:upperboundontau} in Lemma \ref{Lemma Ecascn}.

In comparison to the noiseless setting (see ~\cite{bresler2020learning}), we can guarantee exact recover with high probability under the event $\Estrong$ when the weakest edge satisfies the inequality
\begin{align}\label{eq:tau}
\tanh\alpha\geq & \frac{4\eps}{\sqrt{1-\tanh\beta}}.
\end{align}
Notice that \eqref{eq:tau} can be obtained by \eqref{eq:tau_noise} when $q=0$ and $n=\nn$. When $q>0$ and $n=\nn$ it is clear that the set of trees that can be recovered from noisy
observations is a subset of the set of trees that can be recovered
from the original observations. Also, we have 
\begin{align*}
\eps= \sqrt{2\log\left(2p^{2}/\delta\right)/n}
&\implies n=\frac{2}{\eps^{2}}\log\left(2p^{2}/\delta\right) \quad \text{and}\\
\neps=\sqrt{2\log\left(2p^{2}/\delta\right)/\nn}
&\implies \nn=\frac{2}{\neps^{2}}\log\left(2p^{2}/\delta\right) \numberthis \label{eq:n}.
\end{align*}

 By combining \eqref{eq:tau_noise} with \eqref{eq:n} we found the number of samples that we need to recover the
tree with probality at $1-\delta$ (Theorem \ref{thm:sufficient}): 
\begin{align}
\nn> & \frac{32\left[1-\left(1-2q\right)^{4}\tanh\beta\right]}{\left(1-\tanh\beta\right)^{2}\left(1-2q\right)^{4}\tanh^{2}\alpha}\log\frac{2p^{2}}{\delta}.
\end{align}
\begin{comment}
Then by using the lower bounds: $\tanh^{-1}\alpha\geq\frac{1}{\alpha}+1$and
$\frac{1}{1-\tanh\beta}=\frac{1}{2}\left(e^{2\beta}+1\right)$ we
get the simplified form:???
\begin{align*}
\end{align*}
\end{comment}
On the other hand when there is no noise~\citep{bresler2020learning} we need 
\begin{align}
n >\frac{32}{\tanh^{2}\alpha\left(1-\tanh\beta\right)}\log\frac{2p^{2}}{\delta}.
\end{align}

\section{Analysis of the Event $\Ecascn$}\label{Ecascadesection}
\begin{lemma}\label{Lemma Cond_Prob}
Consider a path of length $d\geq 2$ in the original tree $\T$, and without loss of generality assume that path is $X_1-X_2-\cdot\cdot\cdot-X_{d+1}$. Recall that $Y^{(i)}_m$ is the $i^\text{th}$ sample of $Y_m$ and $m\in[d+1]$ and $  \nmue_{k}\triangleq\frac{1}{n}\sum^{n}_{i=1}  Y^{(i)}_k Y^{(i)}_{k+1}$, $ k\in [d]$.
  Then \begin{align}
    &\P\lp  Y^{(\ell)}_{k} Y^{(\ell)}_{k+1} =\pm 1 \Big|\nmue_{k-1},\ldots,\nmue_{1} \rp\nonumber\\
    &=\frac{1\pm(1-2q)^2 \mu_{k}}{2}\frac{1-\nmu_{k-1}\nmue_{k-1}}{1-(\nmu_{k-1})^2}+\nmu_{k-1}\frac{1\pm \mu_{k}}{2}\frac{\nmue_{k-1}-\nmu_{k-1}}{1-(\nmu_{k-1})^2}.
\end{align}
\end{lemma}
\begin{proof}
Note that \begin{align}
    \nmue_{k}=\frac{1}{n}\sum^{n}_{i=1}  Y^{(i)}_k Y^{(i)}_{k+1}=\frac{1}{n}\sum^{n}_{i=1} \lp X_k N_k X_{k+1} N_{k+1}\rp^{(i)},
\end{align} 
where eacch term \begin{align}
    \lp X_k N_k X_{k+1} N_{k+1}\rp^{(\ell)} \perp \nmue_{r}\quad \forall r\in[1,2,\ldots k-2],\forall \ell\in [1,2,\dots,n]
\end{align} 
thus
 \begin{align*}
    &\P\lp  \lp X_k N_k X_{k+1} N_{k+1}\rp^{(\ell)} =\pm 1 \Big|\nmue_{k-1},\ldots,\nmue_{1} \rp\\
    &=\P\lp  \lp X_k N_k X_{k+1} N_{k+1}\rp^{(\ell)} =\pm 1 \Big|\nmue_{k-1} \rp\\
    &=\P\lp  \lp X_k N_k X_{k+1} N_{k+1}\rp^{(\ell)} =\pm 1 \Big| \nmue_{k-1}= \frac{1}{n}\sum^{n}_{i=1} \lp X_{k-1} N_{k-1} X_{k} N_{k}\rp^{(i)} \rp\\
    &=\frac{\P\lp\nmue_{k-1}= \frac{1}{n}\sum^{n}_{i=1} \lp X_{k-1} N_{k-1} X_{k} N_{k}\rp^{(i)}\Big| \lp X_k N_k X_{k+1} N_{k+1}\rp^{(\ell)} =\pm 1\rp}{\P\lp\nmue_{k-1}= \frac{1}{n}\sum^{n}_{i=1} \lp X_{k-1} N_{k-1} X_{k} N_{k}\rp^{(i)}\rp}\nonumber\\
    &\qquad\times \P\lp \lp X_k N_k X_{k+1} N_{k+1}\rp^{(\ell)} =\pm 1 \rp.\numberthis\label{eq:bayes1}
\end{align*}
First we compute the probability $\P\lp\nmue_{k-1}= \frac{1}{n}\sum^{n}_{i=1} \lp X_{k-1} N_{k-1} X_{k} N_{k}\rp^{(i)}\rp$. Define the Bernoulli random variable $Z_{k-1}$ as \begin{align}
    Z_{k-1}\triangleq \frac{X_{k-1} N_{k-1} X_{k} N_{k}+1}{2}=\begin{cases}
    0, \quad \text{w.p. } \frac{1-(1-2q)^2\mu_{k-1}}{2}\\
     1, \quad \text{w.p. }\frac{1+(1-2q)^2\mu_{k-1}}{2}.
    \end{cases}
\end{align}Then \begin{align*}
    &\P\lp\nmue_{k-1}= \frac{1}{n}\sum^{n}_{i=1} \lp X_{k-1} N_{k-1} X_{k} N_{k}\rp^{(i)}\rp\\
    &=\P\lp\nmue_{k-1}= \frac{1}{n}\sum^{n}_{i=1} \lp 2Z_{k-1}-1\rp^{(i)}\rp\\
    &=\P\lp \sum^{n}_{i=1} Z_{k-1}^{(i)}=n\frac{\nmue_{k-1}+1}{2}\rp\\
    &= {n \choose n\frac{\nmue_{k-1}+1}{2}}\lp\frac{1-(1-2q)^2\mu_{k-1}}{2}\rp^{n-n\frac{\nmue_{k-1}+1}{2}} \lp\frac{1+(1-2q)^2\mu_{k-1}}{2}\rp^{n\frac{\nmue_{k-1}+1}{2}}.\numberthis\label{eq:bayes2}
\end{align*} As a second step we compute the probability \begin{align*}\numberthis\label{eq:bayes3}
    &\P\lp\nmue_{k-1}= \frac{1}{n}\sum^{n}_{i=1} \lp X_{k-1} N_{k-1} X_{k} N_{k}\rp^{(i)}\Big| \lp X_k N_k X_{k+1} N_{k+1}\rp^{(\ell)} =\pm 1\rp\\
    &=\P\Bigg(\nmue_{k-1}= \frac{1}{n}\sum^{n}_{i=1,i\neq \ell} \lp X_{k-1} N_{k-1} X_{k} N_{k}\rp^{(i)}\\&\qquad\qquad+\frac{\lp X_{k-1} N_{k-1} X_{k} N_{k}\rp^{(\ell)}}{n}\Big| \lp X_k N_k X_{k+1} N_{k+1}\rp^{(\ell)} =\pm 1\Bigg).
\end{align*} Note that \begin{align}
    \lp X_{k-1} N_{k-1} X_{k} N_{k}\rp^{(i)} \perp \lp X_k N_k X_{k+1} N_{k+1}\rp^{(\ell)},\quad \forall i\neq \ell,
\end{align} and we would like to find the conditional distribution of $\lp X_{k-1} N_{k-1} X_{k} N_{k}\rp^{(\ell)}$ under the event $ \{\lp X_k N_k X_{k+1} N_{k+1}\rp^{(\ell)} =\pm 1\}$. We have
\begin{align*}
    &\P \lp \lp X_{k-1} N_{k-1} X_{k} N_{k}\rp^{(\ell)}=c\Big| \lp X_k N_k X_{k+1} N_{k+1}\rp^{(\ell})=\pm 1\rp  \\
    &=\frac{\P \lp \lp X_{k-1} N_{k-1} X_{k} N_{k}\rp^{(\ell)}=c, \lp X_k N_k X_{k+1} N_{k+1}\rp^{(\ell)}=\pm 1\rp }{\P\lp \lp X_k N_k X_{k+1} N_{k+1}\rp^{(\ell)}=\pm 1\rp}\\
    &=\frac{\frac{1\pm c\E[ X_{k-1} N_{k-1}  X_{k+1} N_{k+1}] +c(1-2q)^2 \mu_{k-1}\pm(1-2q)^2 \mu_{k} }{4}}{\frac{1\pm(1-2q)^2 \mu_{k}}{2}}\\
    &=\frac{1\pm c(1-2q)^2\mu_{k-1}\mu_{k} +c(1-2q)^2 \mu_{k-1}\pm(1-2q)^2 \mu_{k}}{2\lp1\pm(1-2q)^2 \mu_{k} \rp},\quad c\in\{-1,+1\}.\numberthis
\end{align*} Define \begin{align}
    P_1 &\triangleq \P \lp \lp X_{k-1} N_{k-1} X_{k} N_{k}\rp^{(\ell)}=+1\Big| \lp X_k N_k X_{k+1} N_{k+1}\rp^{(\ell})=\pm 1\rp\label{eq:defP1} \\
    P_2 &\triangleq \P \lp \lp X_{k-1} N_{k-1} X_{k} N_{k}\rp^{(\ell)}=-1\Big| \lp X_k N_k X_{k+1} N_{k+1}\rp^{(\ell})=\pm 1\rp,\label{eq:defP2}
\end{align} then 
\begin{align*}
    &\P\Bigg(\nmue_{k-1}= \frac{1}{n}\sum^{n}_{i=1,i\neq \ell} \lp X_{k-1} N_{k-1} X_{k} N_{k}\rp^{(i)}\\&\qquad\qquad+\frac{\lp X_{k-1} N_{k-1} X_{k} N_{k}\rp^{(\ell)}}{n}\Big| \lp X_k N_k X_{k+1} N_{k+1}\rp^{(\ell)} =\pm 1\Bigg)\\
    &=\P\lp\nmue_{k-1}= \frac{1}{n}\sum^{n}_{i=1,i\neq \ell} \lp 2Z_{k-1} -1\rp^{(i)}+\frac{\lp 2Z_{k-1} -1\rp^{(\ell)}}{n}\Big| \lp X_k N_k X_{k+1} N_{k+1}\rp^{(\ell)} =\pm 1\rp\\
    &=\P\lp \sum^{n}_{i=1,i\neq \ell} \lp Z_{k-1} \rp^{(i)}+\lp Z_{k-1} \rp^{(\ell)}=n\frac{\nmue_{k-1}+1}{2}\Big| \lp X_k N_k X_{k+1} N_{k+1}\rp^{(\ell)} =\pm 1\rp\\
    &=\P\lp\lp Z_{k-1} \rp^{(\ell)}=0\Big|\lp X_k N_k X_{k+1} N_{k+1}\rp^{(\ell)} =\pm 1 \rp  \P\lp \sum^{n}_{i=1,i\neq \ell} \lp Z_{k-1} \rp^{(i)}=n\frac{\nmue_{k-1}+1}{2}\rp\\
    &\quad+\P\lp\lp Z_{k-1} \rp^{(\ell)}=1\Big|\lp X_k N_k X_{k+1} N_{k+1}\rp^{(\ell)} =\pm 1 \rp\P\lp \sum^{n}_{i=1,i\neq \ell} \lp Z_{k-1} \rp^{(i)}=n\frac{\nmue_{k-1}+1}{2}-1\rp\\
    &=P_2 {n-1 \choose n\frac{\nmue_{k-1}+1}{2}}\lp\frac{1-(1-2q)^2\mu_{k-1}}{2}\rp^{n-1-n\frac{\nmue_{k-1}+1}{2}} \lp\frac{1+(1-2q)^2\mu_{k-1}}{2}\rp^{n\frac{\nmue_{k-1}+1}{2}}\\
    &\quad + P_1 {n-1 \choose n\frac{\nmue_{k-1}+1}{2}-1}\lp\frac{1-(1-2q)^2\mu_{k-1}}{2}\rp^{n-1-n\frac{\nmue_{k-1}+1}{2}+1} \\ 
    &\quad\quad\quad\quad \lp\frac{1+(1-2q)^2\mu_{k-1}}{2}\rp^{n\frac{\nmue_{k-1}+1}{2}-1}.\numberthis\label{eq:bayes4}
\end{align*}

\noindent Finally \begin{align}
    \P\lp \lp X_k N_k X_{k+1} N_{k+1}\rp^{(\ell)} =\pm 1 \rp=\frac{1\pm(1-2q)^2\mu_k}{2} ,
\end{align} and \eqref{eq:bayes1}, \eqref{eq:bayes2}, \eqref{eq:bayes3}, \eqref{eq:bayes4} give
\begin{align*}
     &\frac{\P\lp  Y^{(\ell)}_{k} Y^{(\ell)}_{k+1} =\pm 1 \Big|\nmue_{k-1},\ldots,\nmue_{1} \rp}{ \P\lp \lp X_k N_k X_{k+1} N_{k+1}\rp^{(\ell)} =\pm 1 \rp}
     \\&=P_2 {n-1 \choose n\frac{\nmue_{k-1}+1}{2}}{n \choose n\frac{\nmue_{k-1}+1}{2}}^{-1}\frac{\lp\frac{1-\nmu_{k-1}}{2}\rp^{n-1-n\frac{\nmue_{k-1}+1}{2}} \lp\frac{1+\nmu_{k-1}}{2}\rp^{n\frac{\nmue_{k-1}+1}{2}}}{\lp\frac{1-\nmu_{k-1}}{2}\rp^{n-n\frac{\nmue_{k-1}+1}{2}} \lp\frac{1+\nmu_{k-1}}{2}\rp^{n\frac{\nmue_{k-1}+1}{2}}}\\
    &\quad + P_1 {n-1 \choose n\frac{\nmue_{k-1}+1}{2}-1}{n \choose n\frac{\nmue_{k-1}+1}{2}}^{-1}\frac{\lp\frac{1-\nmu_{k-1}}{2}\rp^{n-1-n\frac{\nmue_{k-1}+1}{2}+1} \lp\frac{1+\nmu_{k-1}}{2}\rp^{n\frac{\nmue_{k-1}+1}{2}-1}}{\lp\frac{1-\nmu_{k-1}}{2}\rp^{n-n\frac{\nmue_{k-1}+1}{2}} \lp\frac{1+\nmu_{k-1}}{2}\rp^{n\frac{\nmue_{k-1}+1}{2}}}\\
    &=P_2 {n-1 \choose n\frac{\nmue_{k-1}+1}{2}}{n \choose n\frac{\nmue_{k-1}+1}{2}}^{-1}\lp\frac{1-\nmu_{k-1}}{2}\rp^{-1}\\
    &\quad + P_1 {n-1 \choose n\frac{\nmue_{k-1}+1}{2}-1}{n \choose n\frac{\nmue_{k-1}+1}{2}}^{-1}\lp\frac{1+\nmu_{k-1}}{2}\rp^{-1}\\
    &=P_2 \frac{n-n\frac{\nmue_{k-1}+1}{2}}{n}\lp\frac{1-\nmu_{k-1}}{2}\rp^{-1}+ P_1   \frac{n\frac{\nmue_{k-1}+1}{2}}{n} \lp\frac{1+\nmu_{k-1}}{2}\rp^{-1}\\
    &=P_2 \frac{1-\nmue_{k-1}}{1-\nmu_{k-1}}+ P_1\frac{1+\nmue_{k-1}}{1+\nmu_{k-1}}.\numberthis
\end{align*} The latter and the definition of $P_1,P_2$ (see Equations \ref{eq:defP1} and \ref{eq:defP2}) give \begin{align*}
    &\P\lp  Y^{(\ell)}_{k} Y^{(\ell)}_{k+1} \pm 1 \Big|\nmue_{k-1},\ldots,\nmue_{1} \rp\\
    &=\left[ P_2 \frac{1-\nmue_{k-1}}{1-\nmu_{k-1}}+ P_1\frac{1+\nmue_{k-1}}{1+\nmu_{k-1}} \right]\P\lp \lp X_k N_k X_{k+1} N_{k+1}\rp^{(\ell)} =\pm 1 \rp\\
    &= \left[ P_2 \frac{1-\nmue_{k-1}}{1-\nmu_{k-1}}+ P_1\frac{1+\nmue_{k-1}}{1+\nmu_{k-1}} \right] \frac{1\pm(1-2q)^2\mu_k}{2}\\
    &= \Bigg[ \frac{1\mp(1-2q)^2\mu_{k-1}\mu_{k} -(1-2q)^2 \mu_{k-1}\pm(1-2q)^2 \mu_{k}}{2\lp1\pm(1-2q)^2 \mu_{k} \rp} \frac{1-\nmue_{k-1}}{1-\nmu_{k-1}}\\&\qquad+ \frac{1\pm (1-2q)^2\mu_{k-1}\mu_{k} +(1-2q)^2 \mu_{k-1}\pm(1-2q)^2 \mu_{k}}{2\lp1\pm(1-2q)^2 \mu_{k} \rp}\frac{1+\nmue_{k-1}}{1+\nmu_{k-1}} \Bigg] \frac{1\pm(1-2q)^2\mu_k}{2}\\
    &= \frac{1\mp(1-2q)^2\mu_{k-1}\mu_{k} -(1-2q)^2 \mu_{k-1}\pm(1-2q)^2 \mu_{k}}{4} \frac{1-\nmue_{k-1}}{1-\nmu_{k-1}}\\&\qquad+ \frac{1\pm (1-2q)^2\mu_{k-1}\mu_{k} +(1-2q)^2 \mu_{k-1}\pm(1-2q)^2 \mu_{k}}{4}\frac{1+\nmue_{k-1}}{1+\nmu_{k-1}}\\
    &= \frac{1\pm(1-2q)^2 \mu_{k}}{4}\lp  \frac{1-\nmue_{k-1}}{1-\nmu_{k-1}}+\frac{1+\nmue_{k-1}}{1+\nmu_{k-1}}\rp\\&\qquad+ \frac{\pm (1-2q)^2\mu_{k-1}\mu_{k} +(1-2q)^2 \mu_{k-1}}{4}\lp  \frac{1+\nmue_{k-1}}{1+\nmu_{k-1}}-\frac{1-\nmue_{k-1}}{1-\nmu_{k-1}}\rp\\
    &=\frac{1\pm(1-2q)^2 \mu_{k}}{2}\frac{1-\nmu_{k-1}\nmue_{k-1}}{1-(\nmu_{k-1})^2}+(1-2q)^2\mu_{k-1}\frac{1\pm \mu_{k}}{2}\frac{\nmue_{k-1}-\nmu_{k-1}}{1-(\nmu_{k-1})^2}. \numberthis
\end{align*} Note that $(1-2q)^2\mu_{k-1}=\nmu_{k-1}$, and the proof is completed.

\end{proof}

\begin{lemma}\label{lemma:Edge_corr_estimation_error}
Define the function $K(\beta,q)$ as\begin{align}\label{eq:K_def_append}
  K(\beta,q)\triangleq \frac{10(1-\tanh^2 (\beta))}{9+(1-2q)^2-\tanh^2 (\beta) (1-2q)^2(9(1-2q)^2+1)}
\end{align}and the event $\mathrm{E^{edge}_{e,\dagger}}$ as \begin{align}\label{eq:event_Eedge}
    \mathrm{E^{edge}_{e,\dagger}}\triangleq \left\{\left|\nmue_e-\nmu_e\right|\leq \gamma_e  \right\},\quad  e\in \mc{E}_{\T},\quad \gamma_e>0,
\end{align} and $\mathrm{E^{edge}_{\dagger}(\mc{E}_{\T})}\triangleq \cap_{e\in\mc{E}_T} \mathrm{E^{edge}_{e,\dagger}}$. If \begin{align}
    n\geq \frac{108e^{2\beta}\log(2p/\delta)}{(1-2q)^4 K(\beta,q)} \text{ and }\gamma_e=\sqrt{3\frac{1-\mu_e^2}{n K(\beta,q)}\log(2p/\delta)}\label{eq:bernstein_bound_on_n}
\end{align} then $\P \left[ \lp \mathrm{E^{edge}_{\dagger}(\mc{E}_{\T})} \rp^c\right]\leq \delta$.
\end{lemma}

\begin{proof}
 The variance of $\nmue_e$ is $(1-(\nmu_e)^2)/n$ and by applying Bernstein's inequality \begin{align}
    \P \left[ \lp \mathrm{E^{edge}_{e,\dagger}} \rp^c\right]\leq 2 \exp \lp-\frac{n\gamma^2_e}{2\lp 1-(\nmu_e)^2\rp+\frac{4}{3} \gamma_e} \rp,\quad \forall \gamma_e > 0. \label{eq:Bernsteins1}
\end{align}

We choose $\gamma_e=\sqrt{3\frac{1-\mu_e^2}{n K(\beta,q)}\log(2p/\delta)}$ (because the parameter $\gamma_e$ is free, that is, Bernstein's inequality holds for all $\gamma_e>0$). If $n$ satisfies \eqref{eq:bernstein_bound_on_n} then \begin{align}
    \gamma_e\leq \sqrt{3\frac{1-\mu_e^2}{108e^{2\beta}}(1-2q)^4 }\leq \frac{(1-2q)^2}{6}(1-\mu_e^2),\numberthis\label{eq:inequality_gamma}
\end{align} and the last is true because $e^{-2\beta}\leq 1-\tanh(\beta)\leq 1-|\mu_e|\leq 1-\mu^2_e$. By applying \eqref{eq:bernstein_bound_on_n} and \eqref{eq:inequality_gamma} on \eqref{eq:Bernsteins1} we get\begin{align*}
     &\P \left[ \lp \mathrm{E^{edge}_{e,\dagger}} \rp^c\right]\\
     &\leq 2 \exp \lp-\frac{n\gamma^2_e}{2\lp 1-(\nmu_e)^2\rp+\frac{4}{3} \gamma_e} \rp\\
     &\leq 2 \exp \lp-\frac{3\frac{1-\mu_e^2}{ K(\beta,q)}\log(2p/\delta)}{2\lp 1-(1-2q)^4\mu_e^2\rp+\frac{4}{3} \frac{(1-2q)^2}{6}(1-\mu^2_e)} \rp\\
     &=2 \exp \lp-\frac{3}{K(\beta,q)}\frac{1-\mu_e^2}{2+\frac{2}{9}(1-2q)^2-\mu_e^2 (1-2q)^2(2(1-2q)^2+\frac{2}{9})} \log(2p/\delta) \rp\\
     &=2 \exp \lp-\frac{3}{2K(\beta,q)}\frac{1-\mu_e^2}{1+\frac{1}{9}(1-2q)^2-\mu_e^2 (1-2q)^2((1-2q)^2+\frac{1}{9})} \log(2p/\delta) \rp\\
     &\leq 2 \exp \lp-\frac{10}{9K(\beta,q)}\frac{1-\mu_e^2}{1+\frac{1}{9}(1-2q)^2-\mu_e^2 (1-2q)^2((1-2q)^2+\frac{1}{9})} \log(2p/\delta) \rp. \numberthis\label{eq:trick_K}
\end{align*}  The following function \begin{align}
    f(x)=\frac{10}{9}\frac{1-x}{1+\frac{1}{9}(1-2q)^2-x (1-2q)^2((1-2q)^2+\frac{1}{9})},\text{ } x\in [\tanh^2 (\alpha), \tanh^2 (\beta)]
\end{align} is strictly decreasing, thus we have $f(\tanh^2(\beta))\leq f(x)$ for all $x\in [\tanh^2 (\alpha), \tanh^2 (\beta)]$. Also $K(\beta,q)\equiv f(\tanh^2 (\beta))$, the latter together with \eqref{eq:trick_K} give \begin{align*}
     &\P \left[ \lp \mathrm{E^{edge}_{e,\dagger}} \rp^c\right] \\
     &\leq 2 \exp \lp-\frac{1}{K(\beta,q)} f(\mu^2_e)\log(2p/\delta) \rp\\
     &\leq 2 \exp \lp-\frac{1}{K(\beta,q)} f(\tanh^2 (\beta))\log(2p/\delta) \rp\\
     &= 2 \exp \lp-\frac{1}{K(\beta,q)} K(\beta,q) \log(2p/\delta) \rp\\
     &=\frac{\delta}{p}.\numberthis
\end{align*} Finally, by applying union over the $p-1$ edges of the tree we get $\P \left[ \lp \mathrm{E^{edge}_{\dagger}(\mc{E}_{\T})} \rp^c\right]\leq \delta$.
\end{proof}
The next Lemma is the extension of Lemma 8.7 by~\cite{bresler2020learning}. The sample complexity bound exactly recovers the noiseless case and its expression is continuous at $q=0$. Further, the bound is independent of the length of the longest path $d$, similarly to the noiseless setting. Finally, we provide upper bounds on the functions that appear in the bound. The latter give a more tractable version of the result and a clear representation of the required number of samples as a function of the parameters. 
\begin{lemma}[Concentration bound for the event $\Ecascn$]
\label{Lemma Ecascn}For $\beta>0$ and $q\in[0,1/2)$ we define the functions $S(\cdot)$, $G(\cdot)$, $K(\cdot)$, $A(\cdot)$, $\Delta(\cdot)$
\begin{align}
\hspace{-1cm}  \UBxi&\triangleq 2+\frac{(1-2q)^2}{6}(1-(1-2q)^2)\tanh^2(\beta)\leq 3- (1-2q)^2\triangleq S\\
\hspace{-1cm}   A(\beta,q)&\equiv A\triangleq (1-2q)^2 [1 - \tanh (\beta)(1-(1-2q)^2)]\\
\hspace{-1cm}  G(\beta,q)&\triangleq\frac{3}{4(1-2q)^2}\left[ d(1-A) \lp\frac{A+2}{3}\rp^{d} +1\right]\leq \frac{3\lp 3e^{-1}\mathds{1}_{q\neq 0} +1\rp}{4(1-2q)^2}\triangleq G,\label{eq:Def_G()} 
\end{align} and the inequality in \eqref{eq:Def_G()} holds because the function $G(\beta,q)$ is bounded for all $d\in \mc{N}\setminus\{1\}$. 
\begin{align}
\hspace{-1cm} K(\beta,q)&\triangleq \frac{10(1-\tanh^2 (\beta))}{9+(1-2q)^2-\tanh^2 (\beta) (1-2q)^2(9(1-2q)^2+1)}\geq   e^{-2\beta\mathds{1}_{q= 0}} \triangleq K \label{eq:K()_def}\\
\hspace{-1cm}\Delta &\triangleq  \frac{1-(1-2q)^2}{1-(1-2q)^4\tanh^2(\beta)}\sqrt{\frac{3\log(2p^3/\delta)}{n }}\tanh^2(\beta)e^{2\beta}.
\end{align} 
If $\Delta<\ngam\leq S(\beta,q)G(\beta,q)/3+\Delta$ and \begin{align}\label{eq:suff_bound_cascade_red}
    n\geq \max\left\{\frac{S^2(\beta,q)G^2(\beta,q)}{0.3^2\lp \ngam-\Delta \rp^2 }  \log(4p^2/\delta) ,\frac{108e^{2\beta}}{(1-2q)^4 K(\beta,q)} \log(2p^3/\delta)\right\}
\end{align} then for any path $\mc{A}_d = \{e_1,e_2,\dots, e_d \}$ of $\T$ with $d$ edges, it is true that \begin{align}
    \P \lp \left| \prod_{e\in\mc{A}_d}\frac{\hat{\mu}^{\dagger}_{e}}{(1-2q)^2} - \prod_{e\in\mc{A}_d}\frac{\mu^{\dagger}_{e}}{(1-2q)^2} \right|\geq \ngam \rp\leq \frac{2\delta}{p^2},\quad d>2.
\end{align}

\end{lemma}
 
\begin{proof}
For sake of space we proceed by using the notation $\nmu_k$ and $\nmue_k$ instead of $\nmu_{e_k}$ and $\nmue_{e_k}$ for $k\in [d]$. Define the random variable \begin{align}
    M^{\dagger}_i\triangleq \lp\frac{\nmue_i}{(1-2q)^2}-\frac{\nmu_i}{(1-2q)^2}\rp\prod^{i-1}_{j=1}\frac{\nmue_j}{(1-2q)^2}\prod^{d}_{j=i+1}\frac{\nmu_j}{(1-2q)^2}.
\end{align} Then $\sum^d_{i=1} M^{\dagger}_i= \prod^{d}_{i=1}\frac{\hat{\mu}^{\dagger}_{i}}{(1-2q)^2} - \prod^{d}_{i=1}\frac{\mu^{\dagger}_{i}}{(1-2q)^2}$, and define the sequence of paths with length $k$ as $\mc{A}_k\triangleq \left\{ e_1, e_2,\ldots,e_k \right\} \subset \mc{A}_d$, for $2\leq k\leq d$. Although we provided the definition of the event $\mathrm{E}^{\mathrm{edge}}_{\dagger}(\cdot)$ in Lemma \ref{lemma:Edge_corr_estimation_error}, we restate it below for completeness. For some $\gamma_e>0$ the definition follows\begin{align}
    \mathrm{E^{edge}_{\dagger}(\mc{A}_k)}\triangleq \bigcap_{e\in\mc{A}_k}  \left\{\left|\nmue_e-\nmu_e\right|\leq \gamma_e  \right\}.
\end{align} The law of total probability gives
\begin{align}\label{eq:Law_of_total}
  \!\!\!  \P\left[\left|\sum^d_{i=1} M^{\dagger}_i\right|>\gamma  \right]&\leq  \P\left[\left|\sum^d_{i=1} M^{\dagger}_i\right|>\gamma \Bigg|  \mathrm{E}^{\mathrm{edge}}_{\dagger}(\mc{A}_{d-1}) \right]+\P \left[\lp \mathrm{E}^{\mathrm{edge}}_{\dagger}(\mc{A}_{d-1}) \rp^c \right].
\end{align} For second term, Lemma \ref{lemma:Edge_corr_estimation_error} gives that $\P \left[\lp \mathrm{E}^{\mathrm{edge}}_{\dagger}(\mc{A}_{d-1}) \rp^c \right]\leq \delta/p^2$ if \begin{align}
    n\geq \frac{108e^{2\beta}\log(2p^3/\delta)}{(1-2q)^4 K(\beta,q)},
\end{align} we define the function $K(\beta,q)$ in Lemma \ref{lemma:Edge_corr_estimation_error} \eqref{eq:K_def_append}. Here we will find an upper bound for the first term of the right hand-side of \eqref{eq:Law_of_total}.
%\kn{For any $k\leq d$,  we will...}
%\begin{align*}
%    \color{blue}\E\left[\exp\left(\lambda \sum^{k}_{i=1} M^{\dagger}_i \right)\Bigg| \mathrm{E}^{\mathrm{edge}}_{\dagger}(\mc{A}_{k-1}) \right].
%\end{align*} 
Note that $M^{\dagger}_k$ is written as 

\begin{align*}
    M^{\dagger}_k&= \lp\frac{\frac{1}{n}\sum^{n}_{\ell=1} \lp Y_k Y_{k+1}\rp^{(\ell)}}{(1-2q)^2}-\frac{\nmu_k}{(1-2q)^2}\rp\prod^{k-1}_{j=1}\frac{\nmue_j}{(1-2q)^2}\prod^{d}_{j=k+1}\frac{\nmu_j}{(1-2q)^2}\\&= \frac{1}{n}\sum^{n}_{\ell=1}\lp\frac{ \lp X_k N_k X_{k+1} N_{k+1}\rp^{(\ell)}}{(1-2q)^2}-\frac{\nmu_k}{(1-2q)^2}\rp\prod^{k-1}_{j=1}\frac{\nmue_j}{(1-2q)^2}\prod^{d}_{j=k+1}\frac{\nmu_j}{(1-2q)^2}.\numberthis\label{eq:Def_Mk}
\end{align*}
and we define\begin{align}
    Z^{(\ell)}_k\triangleq \lp\frac{ \lp X_k N_k X_{k+1} N_{k+1}\rp^{(\ell)}}{(1-2q)^2}-\frac{\nmu_k}{(1-2q)^2}\rp\prod^{k-1}_{j=1}\frac{\nmue_j}{(1-2q)^2}\prod^{d}_{j=k+1}\frac{\nmu_j}{(1-2q)^2}.\label{eq:Def_Zk}
\end{align} The random variables $ Z^{(\ell)}_k$ for $\ell\in [n]$ and fixed $k\in [d]$ are independent conditioned on the event $\mathrm{E}^{\mathrm{edge}}_{\dagger}(\mc{A}_{k-1})$. However the conditional expectation $\E[Z^{(i)}_k |Z^{(i-1)}_k,\ldots, Z^{(1)}_k,\nmue_{k-1},\ldots,\nmue_{1} ]$ is not zero. To apply a concentration of measure result on $ Z^{(\ell)}_k$ we use the extended Bennet's inequality for supermartingales~\citep{fan2012hoeffding}.

 \textbf{Martingale Differences:} Define $\xi_{k}^{(0)}\triangleq 0$, $\xi_{k}^{(1)}\triangleq Z^{(1)}_k-\E\left[Z^{(1)}_k | \nmue_{k-1},\ldots,\nmue_{1}\right]$, $\xi^{(i)}_{k}\triangleq Z^{(i)}_k -\E\left[Z^{(i)}_k |Z^{(i-1)}_k,\ldots, Z^{(1)}_k,\nmue_{k-1},\ldots,\nmue_{1} \right]$. Also, define as $\mc{F}^{k}_{i-1}$ the $\sigma$-algebra generated by $Z^{(i-1)}_k,\ldots, Z^{(1)}_k, \nmue_{k-1},\ldots,\nmue_{1}$, then $(\xi^{(i)}_{k},\mc{F}^{k}_i)_{i=1.\ldots,n}$ is a Martingale Difference Sequence (MDS). %and $\xi^{k}_1\ldots,\xi^{k}_{i-1}$ is the same, and it is denoted by $\mc{F}^{k}_{i-1}$. 

Additionally, conditioned on $Z^{(i-1)}_k,\ldots, Z^{(1)}_k, \nmue_{k-1},\ldots,\nmue_{1}$ we have
 \begin{align}
     \!\!\!\!Z^{(i)}_k=\begin{cases}
       \frac{1}{(1-2q)^{2d}}\lp 1-\nmu_k \rp\prod^{k-1}_{j=1}\nmue_j \prod^{d}_{j=k+1}  \nmu_j, & \text{ w.p. }\P\lp  Y^{(\ell)}_{k} Y^{(\ell)}_{k+1} =+ 1 \Big|\nmue_{k-1} \rp \\
       -\frac{1}{(1-2q)^{2d}}\lp 1+\nmu_k \rp\prod^{k-1}_{j=1}\nmue_j \prod^{d}_{j=k+1}  \nmu_j, &\text{ w.p. }\P\lp  Y^{(\ell)}_{k} Y^{(\ell)}_{k+1} =- 1 \Big|\nmue_{k-1}\rp ,
     \end{cases}
\end{align}
and we have proved (Lemma \ref{Lemma Cond_Prob}) that
\begin{align}\label{eq:F1}
    &\P\lp  Y^{(\ell)}_{k} Y^{(\ell)}_{k+1} =\pm 1 \Big|\nmue_{k-1},\ldots,\nmue_{1} \rp\nonumber\\
    &=\P\lp  Y^{(\ell)}_{k} Y^{(\ell)}_{k+1} =\pm 1 \Big|\nmue_{k-1} \rp\nonumber\\
    &=\frac{1\pm \nmu_{k}}{2}\frac{1-\nmu_{k-1}\nmue_{k-1}}{1-(\nmu_{k-1})^2}+\nmu_{k-1}\frac{1\pm \mu_{k}}{2}\frac{\nmue_{k-1}-\nmu_{k-1}}{1-(\nmu_{k-1})^2}.
\end{align}
Thus we have \begin{align*}
     \!\!\!&\E\left[Z^{(i)}_k |\mc{F}^{k}_{i-1}\right]=\E\left[Z^{(i)}_k |Z^{(i-1)}_k,\ldots, Z^{(1)}_k,\nmue_{k-1},\ldots,\nmue_{1} \right]\\
     &=\left[\lp 1-\nmu_k \rp \P\lp  Y^{(\ell)}_{k} Y^{(\ell)}_{k+1} =+ 1 \Big|\nmue_{k-1} \rp-(1+\nmu_k)\P\lp  Y^{(\ell)}_{k} Y^{(\ell)}_{k+1} =- 1 \Big|\nmue_{k-1} \rp\right]\nonumber\\\label{eq: cond_mean}\numberthis
     &\qquad\times\frac{\prod^{k-1}_{j=1}\nmue_j \prod^{d}_{j=k+1}  \nmu_j}{(1-2q)^{2d}}.
\end{align*} 
Note that \eqref{eq:F1} gives
\begin{align*}
    &\left[\lp 1-\nmu_k \rp \P\lp  Y^{(\ell)}_{k} Y^{(\ell)}_{k+1} =+ 1 \Big|\nmue_{k-1} \rp-(1+\nmu_k)\P\lp  Y^{(\ell)}_{k} Y^{(\ell)}_{k+1} =- 1 \Big|\nmue_{k-1} \rp\right]\\
    &= \lp(1-\nmu_k)\frac{1+\nmu_k}{2}-(1+\nmu_k)\frac{1-\nmu_k}{2}\rp\nmu_{k-1}\frac{1-\nmu_{k-1}\nmue_{k-1}}{1-(\nmu_{k-1})^2}\nonumber\\
    &\quad + \lp (1-\nmu_k)\frac{1+\mu_k}{2}-(1+\nmu_k)\frac{1-\mu_k}{2} \rp\nmu_{k-1}\frac{\nmue_{k-1}-\nmu_{k-1}}{1-(\nmu_{k-1})^2}\\
    &=\lp (1-\nmu_k)\frac{1+\mu_k}{2}-(1+\nmu_k)\frac{1-\mu_k}{2} \rp\nmu_{k-1}\frac{\nmue_{k-1}-\nmu_{k-1}}{1-(\nmu_{k-1})^2}\\
    &=\frac{1}{2}\lp 1-\nmu_k+\mu_k -\nmu_k\mu_k - 1 +\mu_k -\nmu_k+\nmu_k\mu_k  \rp\nmu_{k-1}\frac{\nmue_{k-1}-\nmu_{k-1}}{1-(\nmu_{k-1})^2}\\
    &=(\mu_{k}-\nmu_{k})\nmu_{k-1}\frac{\nmue_{k-1}-\nmu_{k-1}}{1-(\nmu_{k-1})^2}.\numberthis
\end{align*} Combine the latter with \eqref{eq: cond_mean} to get \begin{align}\label{eq:cond_mean1}
    \E\left[Z^{(i)}_k |\mc{F}^{k}_{i-1}\right]= \nmu_{k-1}(\mu_{k}-\nmu_{k})\frac{\nmue_{k-1}-\nmu_{k-1}}{1-(\nmu_{k-1})^2}   \frac{\prod^{k-1}_{j=1}\nmue_j \prod^{d}_{j=k+1}  \nmu_j}{(1-2q)^{2d}},\text{ } i\in [n].
\end{align} If $q=0$ then $\E\left[Z^{(i)}_k |\mc{F}^{k}_{i-1}\right]=0$. Also $\lim_{n\to\infty} \E\left[Z^{(i)}_k |\mc{F}^{k}_{i-1}\right]\to 0$ for all $q\in [0,1/2)$ because $\lim_{n\to\infty}\nmue_{k-1}\to\nmu_{k-1}$. Note that \begin{align*}
     \E\left[\lp \xi^{(i)}_{k}\rp^2\Bigg| \mc{F}^{k}_{i-1}   \right]&=\E\left[\lp Z^{(i)}_k- \E\left[Z^{(i)}_k |\mc{F}^{k}_{i-1}\right]\rp^2\Bigg| \mc{F}^{k}_{i-1}   \right]\\
     &= \E\left[\lp Z^{(i)}_{k}\rp^2\Bigg| \mc{F}^{k}_{i-1}   \right]-\E^2\left[Z^{(i)}_k |\mc{F}^{k}_{i-1}\right].\numberthis\label{eq:cond_var1}
     %&=\color{blue}\frac{1-\lp \nmu_k+f_{q,n,k}\lp \omega \rp \rp^2 }{(1-2q)^{4d}}\left[\prod^{k-1}_{j=1}\nmue_j \prod^{d}_{j=k+1}  \nmu_j\right]^2
\end{align*} We compute $\E\Big[\lp Z^{(i)}_{k}\rp^2\Big| \mc{F}^{k}_{i-1}   \Big]$:
\begin{align*}
     \!\!\!&\E\left[\lp Z^{(i)}_k\rp^2 |\mc{F}^{k}_{i-1}\right]=\E\left[\lp Z^{(i)}_k\rp^2 |Z^{(i-1)}_k,\ldots, Z^{(1)}_k,\nmue_{k-1},\ldots,\nmue_{1} \right]\\
     &=\left[\lp 1-\nmu_k \rp^2 \P\lp  Y^{(\ell)}_{k} Y^{(\ell)}_{k+1} =+ 1 \Big|\nmue_{k-1} \rp+(1+\nmu_k)^2\P\lp  Y^{(\ell)}_{k} Y^{(\ell)}_{k+1} =- 1 \Big|\nmue_{k-1} \rp\right]\nonumber\\\label{eq:cond_var2} \numberthis
     &\qquad\times\lp\frac{\prod^{k-1}_{j=1}\nmue_j \prod^{d}_{j=k+1}  \nmu_j}{(1-2q)^{2d}}\rp^2.
\end{align*} We use \eqref{eq:F1} to find \begin{align*}
    &\left[\lp 1-\nmu_k \rp^2 \P\lp  Y^{(\ell)}_{k} Y^{(\ell)}_{k+1} =+ 1 \Big|\nmue_{k-1} \rp+(1+\nmu_k)^2\P\lp  Y^{(\ell)}_{k} Y^{(\ell)}_{k+1} =- 1 \Big|\nmue_{k-1} \rp\right]\\
    &=\lp\lp 1-\nmu_k \rp^2 \frac{1+\nmu_{k}}{2}+\lp 1+\nmu_k \rp^2 \frac{1-\nmu_{k}}{2}\rp\frac{1-\nmu_{k-1}\nmue_{k-1}}{1-(\nmu_{k-1})^2}\nonumber\\
    &\qquad+\lp\lp 1-\nmu_k \rp^2\frac{1+ \mu_{k}}{2}   +\lp 1+\nmu_k \rp^2\frac{1- \mu_{k}}{2}   \rp\nmu_{k-1}\frac{\nmue_{k-1}-\nmu_{k-1}}{1-(\nmu_{k-1})^2}\\
    &=(1-(\nmu_{k})^2)\frac{1-\nmu_{k-1}\nmue_{k-1}}{1-(\nmu_{k-1})^2}+ \lp 1+(\nmu_k)^2 -2\mu_k \nmu_k  \rp\nmu_{k-1}\frac{\nmue_{k-1}-\nmu_{k-1}}{1-(\nmu_{k-1})^2}\\
    &=(1-(\nmu_{k})^2)\frac{1-\nmu_{k-1}\nmue_{k-1}}{1-(\nmu_{k-1})^2}+ \lp 1+(\nmu_k)^2-2(\nmu_k)^2+2(\nmu_k)^2 -2\mu_k \nmu_k  \rp\nmu_{k-1}\frac{\nmue_{k-1}-\nmu_{k-1}}{1-(\nmu_{k-1})^2}\\
    &=(1-(\nmu_{k})^2)\frac{1-\nmu_{k-1}\nmue_{k-1}}{1-(\nmu_{k-1})^2}+ \lp 1-(\nmu_k)^2\rp\nmu_{k-1}\frac{\nmue_{k-1}-\nmu_{k-1}}{1-(\nmu_{k-1})^2} \\&\qquad+2\nmu_k\lp\nmu_k -\mu_k  \rp\nmu_{k-1}\frac{\nmue_{k-1}-\nmu_{k-1}}{1-(\nmu_{k-1})^2}\\
    &=\lp 1-(\nmu_k)^2\rp\left[\frac{1-\nmu_{k-1}\nmue_{k-1}}{1-(\nmu_{k-1})^2}+\nmu_{k-1}\frac{\nmue_{k-1}-\nmu_{k-1}}{1-(\nmu_{k-1})^2}\right] +2\nmu_k\lp\nmu_k -\mu_k  \rp\nmu_{k-1}\frac{\nmue_{k-1}-\nmu_{k-1}}{1-(\nmu_{k-1})^2}\\
    &=\lp 1-(\nmu_k)^2\rp+2\nmu_k\lp\nmu_k -\mu_k  \rp\nmu_{k-1}\frac{\nmue_{k-1}-\nmu_{k-1}}{1-(\nmu_{k-1})^2}.\numberthis\label{eq:cond_var3}
\end{align*} Now we combine \eqref{eq:cond_mean1}, \eqref{eq:cond_var1}, \eqref{eq:cond_var2} and \eqref{eq:cond_var3} to get
\begin{align*}
     &\E\left[\lp \xi^{(i)}_{k}\rp^2\Bigg| \mc{F}^{k}_{i-1}   \right]\\&= \E\left[\lp Z^{(i)}_{k}\rp^2\Bigg| \mc{F}^{k}_{i-1}   \right]-\E^2\left[Z^{(i)}_k |\mc{F}^{k}_{i-1}\right]\\
     &=\Bigg[\lp 1-(\nmu_k)^2\rp+2\nmu_k\lp\nmu_k -\mu_k  \rp\nmu_{k-1}\frac{\nmue_{k-1}-\nmu_{k-1}}{1-(\nmu_{k-1})^2} - \lp\nmu_{k-1}(\mu_{k}-\nmu_{k})\frac{\nmue_{k-1}-\nmu_{k-1}}{1-(\nmu_{k-1})^2}\rp^2\Bigg]  \\
     &\qquad \times\lp\frac{\prod^{k-1}_{j=1}\nmue_j \prod^{d}_{j=k+1}  \nmu_j}{(1-2q)^{2d}}\rp^2\\
     &=\left[1- \lp \nmu_k + \nmu_{k-1}(\mu_{k}-\nmu_{k})\frac{\nmue_{k-1}-\nmu_{k-1}}{1-(\nmu_{k-1})^2}   \rp^2 \right]\lp\frac{\prod^{k-1}_{j=1}\nmue_j \prod^{d}_{j=k+1}  \nmu_j}{(1-2q)^{2d}}\rp^2. \numberthis\label{eq:cond_var4}
\end{align*} For sake of space we define the function \begin{align}
    \efq\equiv f(\nmu_k, \nmu_{k-1},\nmue_{k-1},q )\triangleq \nmu_{k-1}(\mu_{k}-\nmu_{k})\frac{\nmue_{k-1}-\nmu_{k-1}}{1-(\nmu_{k-1})^2},\label{eq:f_q}
\end{align} and then \eqref{eq:cond_mean1} and \eqref{eq:cond_var4} can be written as \begin{align}
    \E\left[Z^{(i)}_k |\mc{F}^{k}_{i-1}\right]&= \efq \frac{\prod^{k-1}_{j=1}\nmue_j \prod^{d}_{j=k+1}  \nmu_j}{(1-2q)^{2d}},\text{ } i\in [n],\\
    \E\left[\lp \xi^{(i)}_{k}\rp^2\Bigg| \mc{F}^{k}_{i-1}   \right]&=\left[1- \lp \nmu_k + \efq   \rp^2 \right]\lp\frac{\prod^{k-1}_{j=1}\nmue_j \prod^{d}_{j=k+1}  \nmu_j}{(1-2q)^{2d}}\rp^2,\text{ }i\in[n]. \label{eq:Conditional_var}
\end{align} 

We would like to find an upper bound on the summation $\sum^d_{k=1}\E\Big[\Big( \xi^{(i)}_{k}\Big)^2\Big| \mc{F}^{k}_{i-1}   \Big]$. Define $A\equiv A(\beta,q)\triangleq (1-2q)^2 [1 - \tanh (\beta)(1-(1-2q)^2)]$, for all $\beta>0$ and $q\in [0,1)$. Then
 \begin{align*}
     &\left[1- \lp \nmu_k + \efq   \rp^2 \right]
     \\&= \left[1- \lp \nmu_k + \nmu_{k-1}(\mu_{k}-\nmu_{k})\frac{\nmue_{k-1}-\nmu_{k-1}}{1-(\nmu_{k-1})^2}   \rp^2 \right]
     \\&= \left[1- \mu^2_k \lp (1-2q)^2 + \nmu_{k-1}(1-(1-2q)^2)\frac{\nmue_{k-1}-\nmu_{k-1}}{1-(\nmu_{k-1})^2}   \rp^2 \right]
     \\& \leq \left[1- \mu^2_k \lp (1-2q)^2 - |\nmu_{k-1}|(1-(1-2q)^2)\frac{|\nmue_{k-1}-\nmu_{k-1}|}{1-(\nmu_{k-1})^2}   \rp^2 \right]
     \\&\leq \left[1- \mu^2_k \lp (1-2q)^2 - (1-2q)^2\tanh (\beta)(1-(1-2q)^2)   \rp^2 \right]\numberthis\label{eq:step1_boundsumvariance}
     \\&  \leq [1-\mu^2_k A(\beta,q) ],\numberthis \label{eq:step2_boundsumvariance}
      \end{align*} and \eqref{eq:step1_boundsumvariance} holds because $|\nmue_{k-1}-\nmu_{k-1}|\leq \gamma^{\dagger}_j\leq (1-2q)^2 (1-\mu^2_j)/6\leq (1-2q)^2 (1-(\nmu_j)^2)/6$. Then \eqref{eq:Conditional_var} and \eqref{eq:step2_boundsumvariance} give
 \begin{align*} 
     & \sum^d_{k=1}\E\left[\lp \xi^{(i)}_{k}\rp^2\Bigg| \mc{F}^{k}_{i-1}   \right]
     \\&\leq  \sum^d_{k=1}[1-\mu^2_k A(\beta,q) ] \lp\frac{\prod^{k-1}_{j=1}\nmue_j \prod^{d}_{j=k+1}  \nmu_j}{(1-2q)^{2d}}\rp^2
     \\&\leq\frac{1}{(1-2q)^2} \sum^d_{k=1}[1-\mu^2_k A(\beta,q) ] \prod^{d}_{j=1,j\neq k}\left[   \mu_j^2 +2\frac{\gamma^{\dagger}_j}{(1-2q)^2}\right].\numberthis\label{eq:accuracu_on_edges}
     \end{align*} The inequality \eqref{eq:accuracu_on_edges} holds under the event $\mathrm{E^{edge}_{\dagger}(\mc{E}_{\T})}$ defined in \eqref{eq:event_Eedge} (see Lemma \ref{lemma:Edge_corr_estimation_error}) because \begin{align}
         \lp\frac{\nmue_j}{(1-2q)^2}\rp^2\leq \lp   \frac{\mu^{\dagger}_j}{(1-2q)^2}   \rp^2 +2\frac{\gamma^{\dagger}_j}{(1-2q)^2},
     \end{align} since $|\nmu_j|\leq (1-2q)^2$, $|\nmue_j|\leq (1-2q)^2$ under the assumption of known $q$. Next we define $x_j\triangleq \mu_j^2 +2\frac{\gamma^{\dagger}_j}{(1-2q)^2}$, then $3(1-A(\beta,q) x_j)/2 +(A(\beta,q)-1)/2) \geq 1-\mu^2_k A(\beta,q)$ and \eqref{eq:accuracu_on_edges} gives     
     \begin{align*}
      &\sum^d_{k=1}\E\left[\lp \xi^{(i)}_{k}\rp^2\Bigg| \mc{F}^{k}_{i-1}   \right]
     \\& \leq \frac{1}{(1-2q)^2}\sum^d_{k=1}\left[\frac{3}{2}(1-A x_j) +\frac{A-1}{2} \right] \prod^{d}_{j=1,j\neq k}x_j
     \\& \leq \frac{d}{(1-2q)^2}\left[\frac{3}{2}(1-A x) +\frac{A-1}{2} \right] x^{d-1}.\numberthis 
     \end{align*} The latter is maximized at $x^* =(A+2) \lp 1-\frac{1}{d} \rp$/3 and $A\in (0,1]$, thus we have
     \begin{align*}
     &\frac{d}{(1-2q)^2}\left[\frac{3}{2}(1-A x) +\frac{A-1}{2} \right] x^{d-1}
     \\& \leq \frac{d}{(1-2q)^2} \left[\frac{3}{2}(1- A\frac{A+2}{3} +A\frac{A+2}{3d}  ) +\frac{A-1}{2} \right] \lp\frac{A+2}{3}\rp^{d-1} \lp 1- \frac{1}{d}\rp^{d-1}
     \\&  = \frac{d}{(1-2q)^2}\left[  \frac{3}{2}-\frac{A(A+2)}{2}+\frac{A-1}{2} \right] (x^*)^{d-1} +  \frac{A(A+2)}{2(1-2q)^2}   (x^*)^{d-1}
     \\&  = d\frac{2-A^2-A}{2(1-2q)^2} \lp\frac{A+2}{3}\rp^{d-1} \lp 1-\frac{1}{d} \rp^{d-1} +  \frac{A(A+2)}{2(1-2q)^2}   \lp\frac{A+2}{3}\rp^{d-1} \lp 1-\frac{1}{d} \rp^{d-1}
     \\& \leq d\frac{2-A^2-A}{2(1-2q)^2} \lp\frac{A+2}{3}\rp^{d-1} \lp 1-\frac{1}{2} \rp +  \frac{A(A+2)}{2(1-2q)^2}   \lp\frac{A+2}{3}\rp \lp 1-\frac{1}{2} \rp
     \\& = d\frac{(A+2)(1-A)}{4(1-2q)^2} \lp\frac{A+2}{3}\rp^{d-1}  +  \frac{A(A+2)^2}{12(1-2q)^2} 
     \\&\leq d\frac{3(1-A)}{4(1-2q)^2} \lp\frac{A+2}{3}\rp^{d} +\frac{3}{4(1-2q)^2}\triangleq G(\beta,q) \numberthis\label{eq:G_definition}
     \\& \leq \frac{(1-A)}{4(1-2q)^2}\frac{3}{e\log \frac{3}{A+2}}+  \frac{3}{4(1-2q)^2}
     \\& = \frac{3(1-A)}{e4(1-2q)^2\log \frac{3}{A+2}}+  \frac{3}{4(1-2q)^2}
     \\&\leq \frac{3\lp 3e^{-1} +1\rp}{4(1-2q)^2}.
 \end{align*} In \eqref{eq:G_definition} we define the function $G(\beta,q)$ and we proved that it has an upper bound independent of $d\in [p-1]$, \begin{align}
     G(\beta,q)\triangleq d\frac{3(1-A(\beta,q))}{4(1-2q)^2} \lp\frac{A(\beta,q)+2}{3}\rp^{d} +\frac{3}{4(1-2q)^2}\leq \frac{3\lp 3e^{-1} \mathds{1}_{q\neq 0} +1\rp}{4(1-2q)^2}\equiv G.\label{eq:G_bound}
 \end{align} 
 For the rest of the proof and the final result $G(\beta,q)$ can be replaced by its upper bound in \eqref{eq:G_bound}, however the definition of $G(\beta,q)$ shows the continuity of the result for $q\to 0$.

The following inequality holds with probability $1$ for all $i\in [n]$ and $k\in[d]$ under the event $\mathrm{E^{edge}_{\dagger}(\mc{E}_{\T})}$ (Lemma \ref{lemma:Edge_corr_estimation_error}), 
\begin{align*}
    |\xi^{(i)}_{k}|&\leq 2 + \left| \E\left[Z^{(i)}_k |\mc{F}^{k}_{i-1}\right] \right|\\&=2+\frac{|\efq|}{(1-2q)^2} \frac{\prod^{k-1}_{j=1}|\nmue_j| \prod^{d}_{j=k+1}  |\nmu_j|}{(1-2q)^{2d-2}}\\
    &\leq 2+\frac{|\efq|}{(1-2q)^2}\\
    &= 2+\frac{1}{(1-2q)^2} |\nmu_{k-1}||(\mu_{k}-\nmu_{k})|\frac{|\nmue_{k-1}-\nmu_{k-1}|}{1-(\nmu_{k-1})^2} \\
    &=2+\tanh^2(\beta)(1-(1-2q)^2)\frac{\gamma^{\dagger}_{k-1}}{1-(\nmu_{k-1})^2}\\
    &\leq 2+\frac{(1-2q)^2}{6}(1-(1-2q)^2)\tanh^2(\beta)\triangleq \UBxi.\numberthis\label{eq:upper_bound_ksi}
\end{align*} the last step comes form the inequality $\gamma^{\dagger}_e\leq \frac{(1-2q)^2}{6} (1-(\mu_e)^2)$ \eqref{eq:inequality_gamma}, which holds if the inequality $n>\frac{108e^{2\beta}}{K(\beta,q)(1-2q)^4}\log(4p)$ holds (see \ref{eq:bernstein_bound_on_n}). Recall that Lemma \ref{lemma:Edge_corr_estimation_error} gives \begin{align}
    \gamma_e&=\sqrt{3\frac{1-\mu_e^2}{n K(\beta,q)}\log(2p^3/\delta)}\leq  \sqrt{3\frac{1-\tanh^2(\beta)}{n K(\beta,q)}\log(2p^3/\delta)}\leq \sqrt{\frac{3\log(2p^3/\delta)}{n }}\equiv \tilde{\gamma}.\label{eq:gammatilde}
\end{align} Also, for all $i\in [n]$ we have \begin{align*}
    &\left|\sum^{d}_{k=1} \E\left[Z^{(i)}_k |\mc{F}^{k}_{i-1}\right] \right|
    \\&=\left|\sum^{d}_{k=2} \E\left[Z^{(i)}_k |\mc{F}^{k}_{i-1}\right] \right|
    \\&\leq\sum^{d}_{k=2}\left| \E\left[Z^{(i)}_k |\mc{F}^{k}_{i-1}\right] \right|
    \\&= \sum^{d}_{k=2}\frac{|\efq|}{(1-2q)^2} \frac{\prod^{k-1}_{j=1}|\nmue_j| \prod^{d}_{j=k+1}  |\nmu_j|}{(1-2q)^{2d-2}}
    \\&\leq\frac{1}{(1-2q)^2}\sum^{d}_{k=2}|\nmu_{k-1}|(\mu_{k}-\nmu_{k})\frac{|\nmue_{k-1}-\nmu_{k-1}|}{1-(\nmu_{k-1})^2}\prod^{k-1}_{j=1}\frac{|\nmue_j|}{(1-2q)^{2}} \prod^{d}_{j=k+1}  \frac{|\nmu_j|}{(1-2q)^2}
    \\&\leq\tanh^2(\beta)(1-(1-2q)^2)\frac{\gamma_e}{1-(1-2q)^4\tanh^2(\beta)}\sum^{d}_{k=2}\prod^{k-1}_{j=1}\frac{|\nmue_j|}{(1-2q)^{2}} \prod^{d}_{j=k+1}  \frac{|\nmu_j|}{(1-2q)^2} \numberthis \label{eq:Eedge_step}
    \\&\leq\tanh^2(\beta)\frac{(1-(1-2q)^2)\tilde{\gamma}}{1-(1-2q)^4\tanh^2(\beta)}\sum^{d}_{k=2}\prod^{k-1}_{j=1}\frac{|\nmue_j|}{(1-2q)^{2}} \prod^{d}_{j=k+1}  \frac{|\nmu_j|}{(1-2q)^2} \numberthis\label{eq:boundongammae}
    \\& \leq\tanh^2(\beta)\frac{(1-(1-2q)^2)\tilde{\gamma}}{1-(1-2q)^4\tanh^2(\beta)}\sum^{d}_{k=2}\prod^{k-1}_{j=1}\frac{|\nmu_j|+\gamma_{j}}{(1-2q)^{2}} \prod^{d}_{j=k+1}  \frac{|\nmu_j|}{(1-2q)^2} \numberthis \label{eq:Eedge_step2}
    \\& =\tanh^2(\beta)\frac{(1-(1-2q)^2)\tilde{\gamma}}{1-(1-2q)^4\tanh^2(\beta)}\sum^{d}_{k=2}\prod^{k-1}_{j=1}\tanh(\beta)+\frac{\gamma_{j}}{(1-2q)^{2}} \prod^{d}_{j=k+1}  \tanh(\beta) 
    \\& \leq \tanh^2(\beta)\frac{(1-(1-2q)^2)\tilde{\gamma}}{1-(1-2q)^4\tanh^2(\beta)}\sum^{d}_{k=2}\prod^{d}_{j=1,j\neq k}\tanh(\beta)+\frac{\gamma_{j}}{(1-2q)^{2}}  
    \\& \leq \tanh^2(\beta)\frac{(1-(1-2q)^2)\tilde{\gamma}}{1-(1-2q)^4\tanh^2(\beta)}\sum^{d}_{k=2}\prod^{d}_{j=1,j\neq k}\lp \tanh (\beta) +\frac{1}{6}(1-\tanh^2(\beta))  \rp
    \\& = \tanh^2(\beta)\frac{(1-(1-2q)^2)\tilde{\gamma}}{1-(1-2q)^4\tanh^2(\beta)}(d-1)\lp \frac{5}{3}-\frac{1}{6}(\tanh(\beta)-3)^2  \rp^{d-1}
    \\& \leq  \tanh^2(\beta)\frac{(1-(1-2q)^2)\tilde{\gamma}}{1-(1-2q)^4\tanh^2(\beta)}\frac{1}{-e\log\lp \frac{5}{3}-\frac{1}{6}(\tanh(\beta)-3)^2  \rp}
    \\&\leq  \frac{(1-(1-2q)^2)\tanh^2(\beta)\tilde{\gamma}}{1-(1-2q)^4\tanh^2(\beta)}e^{2\beta-1}.
    \\&\leq  \frac{(1-(1-2q)^2)\tanh^2(\beta)e^{2\beta}}{1-(1-2q)^4\tanh^2(\beta)}\sqrt{\frac{3\log(2p^3/\delta)}{n }}\triangleq \Delta,\numberthis\label{eq:def_Delta}
    %\\&\leq  (1-(1-2q)^2)e^{4\beta}\sqrt{\frac{3\log(2p^3/\delta)}{n }}\triangleq \Delta. 
\end{align*}  
where \eqref{eq:Eedge_step}, \eqref{eq:boundongammae} , \eqref{eq:Eedge_step2} come from Lemma \ref{lemma:Edge_corr_estimation_error} and \eqref{eq:gammatilde}. Finally, $0<\frac{5}{3}-\frac{1}{6}(\tanh(\beta)-3)^2 <1$ for all $\beta>0$ and $-1/\log\lp \frac{5}{3}-\frac{1}{6}(\tanh(\beta)-3)^2  \rp\leq e^{2\beta}$.

We use the symbol $\E_{\mc{A}_{k-1}}[\cdot]$ to denote the conditional expectation given the the event $\mathrm{E}^{\mathrm{edge}}_{\dagger}(\mc{A}_{k-1})$, for instance \begin{align}
    \E\left[\exp\left(\lambda \sum^{k}_{i=1} M^{\dagger}_i \right)\Bigg|  \mathrm{E}^{\mathrm{edge}}_{\dagger}(\mc{A}_{k-1}) \right]\equiv \E_{\mc{A}_{k-1}}\left[\exp\left(\lambda \sum^{k}_{i=1} M^{\dagger}_i \right)\right].
\end{align} Further we define the function $F(\cdot,\cdot)$ as \begin{align}
    F(t,\lambda)=\log\lp\frac{1}{1+t}e^{-\lambda t} + \frac{t}{1+t}e^\lambda \rp.
\end{align} 
For any $k\leq d$ we have
\begin{align*}
    &\E\left[\exp\left(\lambda \sum^{k}_{i=1} M^{\dagger}_i \right)\Bigg|  \mathrm{E}^{\mathrm{edge}}_{\dagger}(\mc{A}_{k-1}) \right]\\&
    = \E_{\mc{A}_{k-1}}\left[\exp\left(\lambda \sum^{k-1}_{i=1} M^{\dagger}_i \right) \E\left[\exp\lp\lambda M^{\dagger}_k \rp\Bigg|\nmue_1,\ldots,\nmue_{k-1} \right]  \right]\numberthis \label{eq:HS1}
    \\&
    = \E_{\mc{A}_{k-1}}\left[\exp\left(\lambda \sum^{k-1}_{i=1} M^{\dagger}_i \right) \E\left[\exp \lp \lambda \frac{1}{n}\sum^{n}_{i=1} Z^{(i)}_k \rp\Bigg|\nmue_1,\ldots,\nmue_{k-1} \right]  \right]\numberthis \label{eq:HS2}
    \\&  = \E_{\mc{A}_{k-1}}\Bigg[\exp\left(\lambda \sum^{k-1}_{i=1} M^{\dagger}_i \right) \\&\qquad\times\E\left[ \exp \lp \lambda \frac{1}{n}\sum^{n}_{i=1} \xi^{(i)}_{k}\rp \exp \lp\lambda\frac{1}{n}\sum^{n}_{i=1} \E\left[Z^{(i)}_k |\mc{F}^{k}_{i-1}\right]  \rp\Bigg|\nmue_1,\ldots,\nmue_{k-1} \right]  \Bigg]\numberthis \label{eq:HS3}
    \\&\leq \exp\lp\lambda\E\left[Z^{(i)}_k |\mc{F}^{k}_{i-1}\right] \rp\nonumber  \\&\qquad\times\E_{\mc{A}_{k-1}}\left[\exp\left(\lambda \sum^{k-1}_{i=1} M^{\dagger}_i \right) \E\left[ \exp \lp \lambda \frac{1}{n}\sum^{n}_{i=1} \xi^{(i)}_{k}\rp \Bigg|\nmue_1,\ldots,\nmue_{k-1} \right]  \right]\numberthis \label{eq:HS4}
    \\& \leq\exp\lp\lambda\E\left[Z^{(i)}_k |\mc{F}^{k}_{i-1}\right] \rp\nonumber \\& \qquad \times \E_{\mc{A}_{k-1}}\left[\exp\left(\lambda \sum^{k-1}_{i=1} M^{\dagger}_i \right) \exp\left\{n\; F\lp \frac{\E\left[\lp \xi^{(i)}_{k}\rp^2\Bigg| \mc{F}^{k}_{i-1}\numberthis \label{eq:HS5}  \right]}{\UBxi ^2},|\lambda|\frac{\UBxi}{n}  \rp  \right\}  \right]
    %\\&\color{blue}\leq \exp\lp\lambda\frac{F_{\eps^{*}_q,p}(q,n)}{(1-2q)^{2d}}\Bigg| \prod^{k-1}_{j=1}\nmue_j \prod^{d}_{j=k+1}  \nmu_j\Bigg| \rp\nonumber \\&\color{blue}\qquad\times\E_{\mc{A}_{k-1}}\left[\exp\left(\lambda \sum^{k-1}_{i=1} M^{\dagger}_i \right) \exp\left\{n\; F\lp \frac{g_d (q,n)}{4},|\lambda|\frac{y(q,n)}{n}  \rp  \right\}  \right]\numberthis \label{eq:HS6}
    %\\&\color{blue}=\exp\lp\lambda\frac{F_{\eps^{*}_q,p}(q,n)}{(1-2q)^{2d}}\Bigg| \prod^{k-1}_{j=1}\nmue_j \prod^{d}_{j=k+1}  \nmu_j\Bigg|\nonumber \rp\\&\color{blue}\qquad\times\exp\left\{n\; F\lp \frac{g_d (q,n)}{4},|\lambda|\frac{y(q,n)}{n}  \rp  \right\} \E_{\mc{A}_{k-1}}\left[\exp\left(\lambda \sum^{k-1}_{i=1} M^{\dagger}_i \right)\right]
    %\\&\color{blue}=\exp\lp\lambda\frac{F_{\eps^{*}_q,p}(q,n)}{(1-2q)^{2d}}\Bigg| \prod^{k-1}_{j=1}\nmue_j \prod^{d}_{j=k+1}  \nmu_j\Bigg| \rp\exp\left\{n\; F\lp \frac{g_d (q,n)}{4},|\lambda|\frac{y(q,n)}{n}  \rp  \right\}\\&\color{blue}\qquad \times\E\left[\exp\lp\lambda \sum^{k-1}_{i=1} M^{\dagger}_i \rp\Bigg|  \mathrm{E}^{\mathrm{edge}}_{\dagger}(\mc{A}_{k-1}) \right]\numberthis \label{eq:HS9}\\
    \\& \leq\exp\lp\lambda\E\left[Z^{(i)}_k |\mc{F}^{k}_{i-1}\right] \rp\nonumber \\& \qquad \times\exp\left\{n\; F\lp \frac{\E\left[\lp \xi^{(i)}_{k}\rp^2\Bigg| \mc{F}^{k}_{i-1}  \right]}{\UBxi ^2},|\lambda|\frac{\UBxi}{n}  \rp  \right\} \E_{\mc{A}_{k-1}}\left[\exp\left(\lambda \sum^{k-1}_{i=1} M^{\dagger}_i \right)   \right].\numberthis \label{eq:HSL}
\end{align*} 
The equation \eqref{eq:HS1} comes from change of measure and tower property, the definitions \eqref{eq:Def_Mk} and \eqref{eq:Def_Zk} of $M^{\dagger}_{k}$ and $\xi^{(i)}_k$ respectively give \eqref{eq:HS2} and \eqref{eq:HS3}. The \eqref{eq:HS4} is derived by upper bounding the quantity $\left| \E\left[Z^{(i)}_k |\mc{F}^{k}_{i-1}\right] \right|$ similarly to \eqref{eq:upper_bound_ksi}, \eqref{eq:HS5} is the upper bound on the moment generating function of the supermartingale~\cite{fan2012hoeffding}. To get a recurrence we proceed as follows:
\begin{align*}
    &\E\left[\exp\lp\lambda \sum^{k-1}_{i=1} M^{\dagger}_i \rp\Bigg|  \mathrm{E}^{\mathrm{edge}}_{\dagger}(\mc{A}_{k-1}) \right] \P\lp  \mathrm{E}^{\mathrm{edge}}_{\dagger}(\mc{A}_{k-1}) \rp\\
    &= \E\left[\exp\lp\lambda \sum^{k-1}_{i=1} M^{\dagger}_i \rp\Bigg| \mathrm{E}^{\mathrm{edge}}_{\dagger}(\mc{A}_{k-2})\cap \mathrm{E^{edge}_{e_{k-1},\dagger}} \right] \P\lp  \mathrm{E}^{\mathrm{edge}}_{\dagger}(\mc{A}_{k-1}) \rp
    \\& =\E\left[\exp\lp\lambda \sum^{k-1}_{i=1} M^{\dagger}_i \rp\Bigg|  \mathrm{E}^{\mathrm{edge}}_{\dagger}(\mc{A}_{k-2})\cap \mathrm{E^{edge}_{e_{k-1},\dagger}} \right]\nonumber\\&\qquad\times \P\lp \mathrm{E}^{\mathrm{edge}}_{\dagger}(\mc{A}_{k-1})   \Bigg|   \mathrm{E}^{\mathrm{edge}}_{\dagger}(\mc{A}_{k-2}) \rp\P\lp  \mathrm{E}^{\mathrm{edge}}_{\dagger}(\mc{A}_{k-2}) \rp
    \\&=\E\left[\exp\lp\lambda \sum^{k-1}_{i=1} M^{\dagger}_i \rp\Bigg|  \mathrm{E}^{\mathrm{edge}}_{\dagger}(\mc{A}_{k-2})\cap \mathrm{E^{edge}_{e_{k-1},\dagger}} \right]\nonumber\\&\qquad\times \P\lp \mathrm{E^{edge}_{e_{k-1},\dagger}}   \Bigg|   \mathrm{E}^{\mathrm{edge}}_{\dagger}(\mc{A}_{k-2}) \rp\P\lp  \mathrm{E}^{\mathrm{edge}}_{\dagger}(\mc{A}_{k-2}) \rp
    \\&=\E\left[\exp\lp\lambda \sum^{k-1}_{i=1} M^{\dagger}_i \rp \ind{\mathrm{E^{edge}_{e_{k-1},\dagger}}}\Bigg| \mathrm{E}^{\mathrm{edge}}_{\dagger}(\mc{A}_{k-2}) \right] \P\lp  \mathrm{E}^{\mathrm{edge}}_{\dagger}(\mc{A}_{k-2}) \rp
    \\&\leq\E\left[\exp\lp\lambda \sum^{k-1}_{i=1} M^{\dagger}_i \rp \Bigg| \mathrm{E}^{\mathrm{edge}}_{\dagger}(\mc{A}_{k-2}) \right] \P\lp  \mathrm{E}^{\mathrm{edge}}_{\dagger}(\mc{A}_{k-2}) \rp.\numberthis
\end{align*} 
By applying the recurrence $d$ times, we derive the following bound
\begin{align*}
    &\E\left[\exp\left(\lambda \sum^{k}_{i=1} M^{\dagger}_i \right)\Bigg|  \mathrm{E}^{\mathrm{edge}}_{\dagger}(\mc{A}_{k-1}) \right]\\
    &\leq \exp\lp\lambda \sum^d_{k=1}\E\left[Z^{(i)}_k |\mc{F}^{k}_{i-1}\right] \rp\nonumber \\& \qquad \times\exp\left\{n\; \sum^d_{k=1} F\lp \frac{\E\left[\lp \xi^{(i)}_{k}\rp^2\Bigg| \mc{F}^{k}_{i-1}  \right]}{\UBxi ^2},|\lambda|\frac{\UBxi}{n}  \rp  \right\}\\
    &\leq \exp\lp\lambda \sum^{d}_{k=1}\E\left[Z^{(i)}_k |\mc{F}^{k}_{i-1}\right] \rp\nonumber \\& \qquad \times\exp\left\{n\; d F\lp \frac{1}{d}\sum^d_{k=1} \frac{\E\left[\lp \xi^{(i)}_{k}\rp^2\Bigg| \mc{F}^{k}_{i-1}  \right]}{\UBxi ^2},|\lambda|\frac{\UBxi}{n}  \rp  \right\}.\numberthis\label{eq:moment_bound1}
\end{align*} 
Further \eqref{eq:G_bound}, \eqref{eq:upper_bound_ksi}, \eqref{eq:def_Delta} and \eqref{eq:moment_bound1} give
\begin{align*}
    &\E\left[\exp\left(\lambda \sum^{k}_{i=1} M^{\dagger}_i \right)\Bigg|  \mathrm{E}^{\mathrm{edge}}_{\dagger}(\mc{A}_{k-1}) \right]
    \\& \leq \exp\lp\lambda \Delta (\beta,q) \rp\exp\left\{nd\; F\lp \frac{1}{d}\frac{G(\beta,q)}{S^2_{\xi}(\beta,q)},|\lambda|\frac{\UBxi}{n}  \rp  \right\}
    \\& \leq\exp\left\{nd\; F\lp \frac{G(\beta,q)}{d},|\lambda|\frac{\UBxi}{n}  \rp + \lambda \Delta (\beta,q) \right\}.\numberthis
    %\\&  =\exp\left\{nd\; F\lp \frac{G}{4d},|\lambda|\frac{y}{n}  \rp + \lambda \Delta  \right\}\numberthis
\end{align*} 
For sake of space, we denote the functions $G(\beta,q),$ $\UBxi$, and $\Delta (\beta,q)$ as $G$, $S$, and $\Delta$ respectively. It is true that
\begin{align*}
    &\E\left[\exp\lp\lambda \sum^{d}_{k=1} M^{\dagger}_k \rp\Bigg| \mathrm{E}^{\mathrm{edge}}_{\dagger}(\mc{A}_{d-1}) \right]\P \left[ \mathrm{E}^{\mathrm{edge}}_{\dagger}(\mc{A}_{d-1}) \cap \mathrm{E}_q \right]\\&\leq \exp\left\{nd\; F\lp \frac{G}{d},|\lambda|\frac{S}{n}  \rp + \lambda \Delta  \right\}.\numberthis
\end{align*}
 Under the assumption $n>\frac{108e^{2\beta}}{(1-2q)^4 K(\beta,q)} \log(4p)$, we have \begin{align}
     \P \left[ \lp\mathrm{E}^{\mathrm{edge}}_{\dagger}(\mc{A}_{d-1})\rp^c \right]  \leq \frac{1}{2}.
 \end{align} The latter gives 
\begin{align}
    \E\left[\exp\lp\lambda \sum^{d}_{k=1} M^{\dagger}_k \rp\Bigg| \mathrm{E}^{\mathrm{edge}}_{\dagger}(\mc{A}_{d-1})  \right]\leq 2 \exp\left\{nd\; F\lp \frac{G}{d},|\lambda|\frac{S}{n}  \rp + \lambda \Delta  \right\},
\end{align} which implies that
 \begin{align*}
    \P\left[\sum^{d}_{k=1} M^{\dagger}_k \geq \gamma \Bigg| \mathrm{E}^{\mathrm{edge}}_{\dagger}(\mc{A}_{d-1})  \right]&\leq 2 \min_{\lambda>0} \exp\left\{nd\; F\lp \frac{G}{d},\lambda\frac{S}{n}  \rp + \lambda \Delta -\lambda \gamma  \right\}
    \\&= 2 \min_{\lambda>0} \exp\left\{nd\; F\lp \frac{G}{d},\lambda\frac{S}{n}  \rp + \lambda \lp\Delta - \gamma\rp  \right\},\numberthis \label{eq:Bennet_1}
\end{align*} 
and we define $\gamma'\triangleq \gamma-\Delta$. The minimum value is attained at \begin{align}
    \lambda^{*}&= \frac{n/S}{1+G/d}\log \frac{1+\frac{\gamma'}{GS}  }{1-\frac{\gamma'}{dS}}
\end{align} 
and by substituting the optimal value we get
\begin{align*}
     &\exp\left\{nd\; F\lp \frac{G}{d},\lambda^*\frac{S}{n}  \rp - \lambda^*  \gamma' \right\}
     \\&=\left[ \frac{1}{1+\frac{G}{d}}\lp \frac{1+\frac{\gamma'}{GS}  }{1-\frac{\gamma'}{dS}} \rp^{-\frac{G/d}{1+G/d}-\frac{ \gamma'/(dS)}{(1+G/d)}}+\frac{\frac{G}{d}}{1+\frac{G}{d}}\lp \frac{1+\frac{\gamma'}{Gy}  }{1-\frac{\gamma'}{dS}} \rp^{\frac{1}{1+G/d}\lp 1-\frac{\gamma'}{Sd} \rp}\right]^{nd}
     \\& =\Bigg[ \frac{1}{1+\frac{G}{d}}\lp 1+\frac{\gamma'}{GS}\rp^{-\frac{G/d}{1+G/d}-\frac{ \gamma'/(dS)}{(1+G/d)}}\lp 1-\frac{\gamma'}{dS} \rp^{\frac{G/d}{1+G/d}+\frac{ \gamma'/(dS)}{(1+G/d)}}\\&\qquad+\frac{\frac{G}{d}}{1+\frac{G}{d}}\lp 1+\frac{\gamma'}{GS}   \rp^{\frac{1}{1+G/d}\lp 1-\frac{\gamma'}{Sd} \rp}\lp 1-\frac{\gamma'}{dS} \rp^{-\frac{1}{1+G/d}\lp 1-\frac{\gamma'}{Sd} \rp}\Bigg]^{nd}
     \\&=\Bigg[ \frac{1}{1+\frac{G}{d}}\lp 1+\frac{\gamma'}{GS}\rp^{-\frac{G/d}{1+G/d}-\frac{ \gamma'/(dS)}{(1+G/d)}}\lp 1-\frac{\gamma'}{dS} \rp^{\frac{G/d}{1+G/d}+\frac{ \gamma'/(dS)}{(1+G/d)}-1}\lp1-\frac{\gamma'}{dS}  \rp\\&\qquad+\frac{\frac{G}{d}}{1+\frac{G}{d}}\lp 1+\frac{\gamma'}{GS}   \rp^{\frac{1}{1+G/d}\lp 1-\frac{\gamma'}{Sd} \rp-1} \lp 1+\frac{\gamma'}{GS}\rp  \lp 1-\frac{\gamma'}{dS} \rp^{-\frac{1}{1+G/d}\lp 1-\frac{\gamma'}{Sd} \rp}\Bigg]^{nd}
     \\&=\Bigg[ \frac{1}{1+\frac{G}{d}}\lp1-\frac{\gamma'}{dS}  \rp+\frac{\frac{G}{d}}{1+\frac{G}{d}} \lp 1+\frac{\gamma'}{GS}\rp\nonumber \Bigg]^{nd}\\&\qquad\times\left[\lp 1+\frac{\gamma'}{GS}\rp^{-\frac{G/d}{1+G/d}-\frac{ \gamma'/(dS)}{(1+G/d)}}\lp 1-\frac{\gamma'}{dS} \rp^{\frac{G/d}{1+G/d}+\frac{ \gamma'/(dS)}{(1+G/d)}-1}\right]^{nd}
     \\& =1^{nd}\times\left[\lp 1+\frac{\gamma'}{GS}\rp^{-\frac{G/d}{1+G/d}-\frac{ \gamma'/(dS)}{(1+G/d)}}\lp 1-\frac{\gamma'}{dS} \rp^{\frac{G/d}{1+G/d}+\frac{ \gamma'/(dS)}{(1+G/d)}-1}\right]^{nd}
     \\&=\left[\lp 1+\frac{\gamma'}{GS}\rp^{-\frac{G/d}{1+G/d}-\frac{ \gamma'/(dS)}{(1+G/d)}}\lp 1-\frac{\gamma'}{dS} \rp^{\frac{-1}{1+G/d}+\frac{ \gamma'/(dS)}{(1+G/d)}}\right]^{nd}.\numberthis \label{eq:Bennet_2}
\end{align*} Then \eqref{eq:Bennet_1} and \eqref{eq:Bennet_2} give
\begin{align*}
     &\P\left[\sum^{d}_{k=1} M^{\dagger}_k \geq \gamma \Bigg| \mathrm{E}^{\mathrm{edge}}_{\dagger}(\mc{A}_{d-1})  \right]\\&\qquad\leq 2\left[\lp 1+\frac{\gamma'}{GS}\rp^{-\frac{G/d}{1+G/d}-\frac{ \gamma'/(dS)}{(1+G/d)}}\lp 1-\frac{\gamma'}{dS} \rp^{\frac{-1}{1+G/d}+\frac{ \gamma'/(dS)}{(1+G/d)}}\right]^{nd}.\label{eq:exp_bound1}\numberthis
\end{align*} As a final step we want to express the upper bound as an exponential function of $\gamma$, we define $\zeta\triangleq \gamma'/(SG)$ and we proceed as follows:
\begin{align*}
    &d\left[\lp\frac{G/d}{1+G/d}+\frac{ \gamma'/(dS)}{(1+G/d)}\rp \log \lp 1+\frac{\gamma'}{GS}\rp +\lp\frac{1}{1+G/4d}-\frac{ \gamma'/(dy)}{(1+G/4d)} \rp \log \lp 1-\frac{\gamma'}{dy} \rp \right]
    \\& \geq d\Bigg[ \lp\frac{G/d}{1+G/d}+\frac{ \gamma'/(dS)}{(1+G/d)}\rp \lp \frac{\gamma'}{GS} - \frac{1}{2}\left[ \frac{\gamma'}{GS}\right]^2 \rp\nonumber \\&\qquad -\lp\frac{1}{1+G/d}-\frac{ \gamma'/(dS)}{(1+G/d)} \rp \log \lp   1+\frac{\gamma'/dS}{1-\gamma'/dS} \rp \Bigg] 
    \\& \geq \Bigg[ \lp\frac{dG}{d+G}+\frac{ d\gamma'/S}{(d+G)}\rp \lp \frac{\gamma'}{GS} - \frac{1}{2}\left[ \frac{\gamma'}{GS}\right]^2 \rp\nonumber \\&\qquad-\lp\frac{d^2}{d+G}-\frac{ d\gamma'/S}{(d+G)} \rp  \lp   \frac{\gamma'/dS}{1-\gamma'/dS}-\frac{1}{2}\lp\frac{\gamma'/dS}{1-\gamma'/dS} \rp^2 \rp  \Bigg]
    \\&=\frac{d}{d+G}  \Bigg[\lp G+ \gamma'/S\rp \lp \frac{\gamma'}{GS} - \frac{1}{2}\left[ \frac{\gamma'}{GS}\right]^2 \rp\nonumber \\&\qquad-\lp d- \gamma'/S \rp  \lp   \frac{\gamma'/dS}{1-\gamma'/dS}-\frac{1}{2}\lp\frac{\gamma'/dS}{1-\gamma'/dS} \rp^2 \rp  \Bigg]
    \\& \geq\frac{d}{d+G}  \left[\lp G+ \gamma'/S\rp \lp \frac{\gamma'}{GS} - \frac{1}{2}\left[ \frac{\gamma'}{GS}\right]^2 \rp -\gamma'/S \right]
    \\&=\frac{dG}{d+G}  \left[\lp 1+ \gamma'/(SG)\rp \lp \frac{\gamma'}{GS} - \frac{1}{2}\left[ \frac{\gamma'}{GS}\right]^2 \rp -\gamma'/(SG) \right]
    \\&= \frac{dG}{d+G}  \left[\lp 1+ \zeta\rp \lp \zeta - \frac{1}{2} \zeta^2 \rp -\zeta \right]
    \\&\geq \frac{2G}{2+G}  \left[\lp 1+ \zeta\rp \lp \zeta - \frac{1}{2} \zeta^2 \rp -\zeta \right]
    %\\ &= G\frac{\zeta^2 (1-\zeta)}{2+G}
    \\ &\geq \lp\frac{3\gamma'}{10GS}\rp^2,\quad \forall \gamma'\in (0,\frac{SG}{3}).\numberthis \label{eq:gamma'ineq}
\end{align*}
Recall that $\zeta\triangleq \gamma'/(SG)$, $\gamma'=\gamma-\Delta$. If $\Delta<\gamma\leq S(\beta,q)G(\beta,q)/3+\Delta$ then \eqref{eq:exp_bound1} and \eqref{eq:gamma'ineq} give \begin{align*}
 \P&\left[\sum^{d}_{k=1} M^{\dagger}_k \geq \gamma \Bigg| \mathrm{E}^{\mathrm{edge}}_{\dagger}(\mc{A}_{d-1})  \right]\leq 2\exp\lp-0.3^2n \frac{\lp \gamma-\Delta \rp^2}{S^2(\beta,q)G^2(\beta,q)} \rp. \numberthis \label{eq:first_part_of_law_of_total}
\end{align*} In a similar way we derive the bound \begin{align}\label{eq:second_part_of_law_of_total}
    \P\left[\sum^{d}_{k=1} M^{\dagger}_k \leq - \gamma \big| \mathrm{E}^{\mathrm{edge}}_{\dagger}(\mc{A}_{d-1})  \right]\leq 2\exp\lp-0.3^2n \frac{\lp \gamma-\Delta \rp^2}{S^2(\beta,q)G^2(\beta,q)} \rp.
\end{align}
Finally, we combine \eqref{eq:Law_of_total}, Lemma \ref{lemma:Edge_corr_estimation_error}, \eqref{eq:first_part_of_law_of_total} and \eqref{eq:second_part_of_law_of_total} to derive the bound \eqref{eq:suff_bound_cascade_red} which guarantees that
\begin{align}
    \P\left[\left|\prod^{d}_{i=1}\frac{\hat{\mu}^{\dagger}_{i}}{(1-2q)^2} - \prod^{d}_{i=1}\frac{\mu^{\dagger}_{i}}{(1-2q)^2} \right|>\gamma \right]\leq 2\delta/p^2,\quad \forall d\geq2.
\end{align} To summarize we proved that the event $\Ecascn$ happens with probability at least $1-2\delta/p^2$ by combining Bresler's and Karzand's technique, the Corollary 2.3 by~\cite{fan2012hoeffding} and Lemma \ref{Lemma Cond_Prob}. 
\end{proof}
\section{Predictive Learning, Proof of Theorem \ref{thm:Main_result:short} and Theorem \ref{thm:Main_result}}\label{Section:Predictive_Learning_Proof}\label{ProofofthemainresultsSection}
Recall that our goal is to guarantee that the quantity $\mc{L}^{(2)}(\p(\cdot),\RIPn (\hat{\p}_{\dagger} ))$ is smaller than a number $\eta>0$ with probability at least $1-\delta$. To do this, we use the triangle inequality as 
\begin{align}
\mc{L}^{(2)}\left(\p(\cdot),\RIPn (\hat{\p}_{\dagger} )\right)\leq\mc{L}^{(2)}\left(\p (\cdot),\RIPn\left(\p(\cdot)\right)\right)+\mc{L}^{(2)}\left(\RIPn\left(\p(\cdot)\right),\RIPn (\hat{\p}_{\dagger} )\right)
\label{eq:triangle11}
\end{align} 
and we find the required number of samples such that each of the terms $\mc{L}^{(2)}(\p (\cdot),\RIPn\left(\p(\cdot)\right))$ and $\mc{L}^{(2)}(\RIPn\left(\p(\cdot)\right),\RIPn (\hat{\p}_{\dagger} ))$ in \eqref{eq:triangle11}  is less than $\eta/2$ with probability at least $1-\delta$. The next Lemma provides the necessary bounds on $\ngam$ and $\neps$ that guarantee
    \begin{align*}
    \mc{L}^{(2)}(\RIPn\left(\p(\cdot)\right),\RIPn (\hat{\p}_{\dagger} ))\leq \eta/2.    
    \end{align*}

\begin{lemma}\label{Lemma end_to_end_error}
If $\ngam\leq \frac{\eta}{3}$ and
\begin{align*}
\neps\leq (1-2q)^2 e^{-\beta}\left[20\left(1+ 2e^{\beta}\sqrt{2\left(1-q\right)q\tanh\beta}\right)\right]^{-1},
\end{align*}
then $\mc{L}^{(2)}(\RIPn\left(\p(\cdot)\right),\RIPn (\hat{\p}_{\dagger} ))\leq \eta/2$ under the event $\Ecorrn\cap\Ecascn \cap \Estrongn$.
\end{lemma}

\begin{proof} 
The derivation of the bound is similar to the approach by~\citet[Section 6.1]{bresler2020learning} but with different calculations. In the hidden model, we consider the path between two nodes $i,j$ in the estimated structure $\TCLn$, namely $\text{path}_{\TCLn} (i,j)$, to be $(\mc{F}_{0},e_{1},\mc{F}_{1},e_{1}...,\mc{F}_{t-1},e_{t},\mc{F}_{t})$,
and $\mc{F}_{i}$ are segments with all strong edges and $e_{i}$
are all weak edges. We consider the case of at least one weak edge to exist in the path. If there is no weak edge the bound reduces to the case of Lemma \ref{Lemma Ecascn}. The length of each sub-path $\mc{F}_i$ is denoted as $d_i$, for all $i\in\{0,1,\ldots,t\}$. Each segment (sub-path) $\mc{F}_{i}$ has exactly $d_{i}$ edges, and the total number of edges in the path are $d$; thus $d=\sum^{t}_{i=0} d_{i}+t$. Note that $t\geq 1$ and $d_i\geq0$ for all $i\in\{0,1,\ldots,t\}$. Recall that \begin{align}
\RIPn\left(\p(\cdot)\right) & =\frac{1}{2}\prod_{\left(i,j\right)\in\mc{E}_{\TCLn}} \frac{1+x_{i}x_{j}\E\left[X_{i}X_{j}\right]}{2}=\frac{1}{2}\prod_{{\left(i,j\right)\in\mc{E}_{\TCLn}} }\frac{1+x_{i}x_{j}\frac{\E\left[\nX_{i}\nX_{j}\right]}{\left(1-2q\right)^{2}}}{2}\label{eq:P_=00007Bnoisy=00007D}
\end{align} (the latter comes from \eqref{eq:BSC_corr}), and 
\begin{align}\label{eq:probability_estimator}
\RIPn (\hat{\p}_{\dagger} ) & \triangleq\frac{1}{2}\prod_{\left(i,j\right)\in\mc{E}_{\TCLn}} \frac{1+x_{i}x_{j}\frac{\mbb{\hat{E}}\left[\nX_{i}\nX_{j}\right]}{\left(1-2q\right)^{2}}}{2}.
\end{align} Further, for any tree-structured Ising model distributions $P,\tilde{P}$ with structures $\T=(\mc{V},\mc{E})$ and $\tilde{\T}=(\mc{V},\tilde{\mc{E}})$ respectively, we have \begin{align*}
\Lnorm\left(P,\tilde{P} \right) & \triangleq \sup_{i,j\in\mc{V}}\quad \frac{1}{2}\sum_{x_{i},x_{j}\in\{-1,+1\}^{2}} \left| P(x_{i},x_{j})-\tilde{P}(x_{i},x_{j}) \right| \numberthis \label{eq:L2norm1} \\ &= \sup_{i,j\in\mc{V}} \quad \frac{1}{2}\left|\prod_{e\in\tpath_{\T}\left(i,j\right)} \mu_{e}-\prod_{e'\in\tpath_{\tilde{\T}}\left(i,j\right)} \tilde{\mu}_{e'}\right|.\numberthis \label{eq:L2_norm}
\end{align*} To upper bound the quantity $\mc{L}^{(2)}(\RIPn\left(\p(\cdot)\right),\RIPn (\hat{\p}_{\dagger} ))$ we have 
\begin{align}
&2\mc{L}^{(2)}\left(\RIPn\left(\p(\cdot)\right),\RIPn (\hat{\p}_{\dagger} )\right)\nonumber\\
 &= \left|\prod_{e\in\tpath_{\text{\ensuremath{\TCLn}}}\left(i,j\right)} \frac{\nmu_{e}}{\left(1-2q\right)^{2}}-\prod_{e\in\tpath_{\TCLn}\left(i,j\right)} \frac{\nmue_{e}}{\left(1-2q\right)^{2}}\right|\nonumber\\
 & =\frac{1}{(1-2q)^{2d}}\left|\nmue_{\mc{F}_{0}}\prod_{i=1}^{t}\nmue_{\mc{F}_{i}}\nmue_{e_{i}}-\nmu_{\mc{F}_{0}}\prod_{i=1}^{t}\nmu_{\mc{F}_{i}}\nmu_{e_{i}}\right|\label{eq:telescoping}\\
 & \leq\frac{1}{(1-2q)^{2d}} \Bigg[ \left|\nmue_{\mc{F}_{0}}-\nmu_{\mc{F}_{0}}\right|\prod_{j=1}^{t} \left|\nmu_{\mc{F}_{i}}\nmu_{e_{i}}\right|\nonumber\\&\qquad +
 \sum_{i=1}^{t}\left|\nmue_{\mc{F}_{i}}\nmue_{e_{i}}-\nmu_{\mc{F}_{i}}\nmu_{e_{i}}\right|\left|\nmue_{\mc{F}_{0}}\right|\prod_{j=1}^{i-1} \left|\nmue_{\mc{F}_{j}}\nmue_{e_{j}}\right|\prod_{k=i+1}^{t}\left|\nmu_{\mc{F}_{k}}\nmu_{e_{k}}\right|\Bigg]\label{eq:triangle_inequality}\\
 & =  \frac{\left|\nmue_{\mc{F}_{0}}-\nmu_{\mc{F}_{0}}\right|}{(1-2q)^{2d_0}}\prod_{j=1}^{t} \left|\frac{\nmu_{\mc{F}_{i}}}{(1-2q)^{2d_i}}\frac{\nmu_{e_{i}}}{(1-2q)^{2t}}\right|\nonumber\\&\qquad +
 \sum_{i=1}^{t}\frac{\left|\nmue_{\mc{F}_{i}}\nmue_{e_{i}}-\nmu_{\mc{F}_{i}}\nmu_{e_{i}}\right|}{(1-2q)^{2(d_i+1)}}\frac{|\nmue_{\mc{F}_{0}}|}{(1-2q)^{2d_0}}\frac{\prod_{j=1}^{i-1} \left|\nmue_{\mc{F}_{j}}\nmue_{e_{j}}\right|\prod_{k=i+1}^{t}\left|\nmu_{\mc{F}_{k}}\nmu_{e_{k}}\right|}{(1-2q)^{2(d-d_i-d_0-1)}}\label{eq:d=...}\\
 & \leq\ngam  \lp\frac{\ntau}{(1-2q)^2}\rp^{t-1}+\lp\frac{\ntau+\neps}{(1-2q)^2}\rp^{t-1}\sum_{i=1}^{t}\frac{\left|\nmue_{\mc{F}_{i}}\nmue_{e_{i}}-\nmu_{\mc{F}_{i}}\nmu_{e_{i}}\right|}{(1-2q)^{2(d_i+1)}}\label{eq:event_subst1}\\
 & \leq\ngam  \lp\frac{\ntau}{(1-2q)^2}\rp^{t-1} +\lp\frac{\ntau+\neps}{(1-2q)^2}\rp^{t-1}\sum_{i=1}^{t}  \lp \frac{\left|(\nmue_{\mc{F}_{i}}-\nmu_{\mc{F}_{i}})\nmue_{e_{i}}\right|}{(1-2q)^{2(d_i+1)}}+\frac{\left|\nmu_{\mc{F}_{i}}(\nmue_{e_{i}}-\nmu_{e_{i}})\right|}{(1-2q)^{2(d_i+1)}}\rp\nonumber\\
 &\leq\ngam  \lp\frac{\ntau}{(1-2q)^2}\rp^{t-1} +\lp\frac{\ntau+\neps}{(1-2q)^2}\rp^{t-1}\sum_{i=1}^{t}  \lp \gamma_{\dagger}+\frac{\eps_{\dagger}}{(1-2q)^{2}}\rp\label{eq:event_subst2}\\
 &\leq \lp\frac{\ntau+\neps}{(1-2q)^2}\rp^{t-1} (2t+1) \max \left\{\gamma_{\dagger},\frac{\neps}{(1-2q)^2} \right\} \nonumber\\
  &\leq \lp\frac{ 4\neps e^{\beta}\left(1+ 2e^{\beta}\sqrt{2\left(1-q\right)q\tanh\beta}\right) +\neps}{(1-2q)^2} \rp^{t-1} (2t+1) \max \left\{\gamma_{\dagger},\frac{\neps}{(1-2q)^2} \right\}\label{eq:tau_subst} \\
  &\leq \lp \frac{5\neps e^{\beta}}{(1-2q)^2}\rp^{t-1} \left(1+ 2e^{\beta}\sqrt{2\left(1-q\right)q\tanh\beta}\right)^{t-1} (2t+1) \max \left\{\gamma_{\dagger},\frac{\neps}{(1-2q)^2} \right\}\nonumber \\
 &\leq\frac{2t+1}{4^{t-1}}\frac{\eta}{3}\label{eq:assumptionsepsgam}\\
 &\leq \eta\label{eq:final_bound_on_eta}.
\end{align} Telescoping summation and triangle inequality give \eqref{eq:telescoping} and \eqref{eq:triangle_inequality}. We use the definition of $d$, $d=\sum^{t}_{i=0} d_{i}+t$ to get \eqref{eq:d=...}. The inequalities $\left|\nmu_{\mc{F}_{i}}\right|\leq (1-2q)^{2d_i},$
$\left|\nmue_{\mc{F}_{i}}\right|\leq (1-2q)^{2d_i},$ $\left|\nmu_{e_{i}}\right|\leq\ntau$,
$\left|\nmue_{e_{i}}\right|\leq\ntau+\neps$ hold under
$\Ecorrn$, $\Estrongn$. Further, under event $\Ecascn$ (Lemma \ref{Lemma Ecascn}) it is true that $\left|\nmue_{\mc{F}_{i}}-\nmu_{\mc{F}_{i}}\right|\leq\ngam$, the latter give \eqref{eq:event_subst1} and \eqref{eq:event_subst2}. 
The bound $\ntau\leq 4\neps e^{\beta}(1+ 2e^{\beta}\sqrt{2\left(1-q\right)q\tanh\beta})$ gives \eqref{eq:tau_subst} (see inequality \ref{eq:upperboundontau}). Inequality \eqref{eq:assumptionsepsgam} requires 
    \begin{align}
    \max\left\{\frac{\neps}{(1-2q)^2} ,\ngam\right\}\leq  \frac{\eta}{3}
    \end{align}
    and
    \begin{align}
    \neps\leq (1-2q)^2 e^{-\beta}\left[20\left(1+ 2e^{\beta}\sqrt{2\left(1-q\right)q\tanh\beta}\right)\right]^{-1}.
    \end{align} 
Finally \eqref{eq:final_bound_on_eta} holds for all $t\in\mbb{N}$. The latter completes the proof.
\end{proof} 

The next Lemma provides the set of values of $\neps$ that guarantee $\mc{L}^{(2)}(\p (\cdot),\RIPn (\p(\cdot)))\leq \frac{\eta}{2}$ with high probability.

\begin{lemma}\label{Lemma F5}If
\begin{align}\label{assumption_on_eps_2}
\neps\leq \min \left\{ \frac{\eta}{16}(1-2q)^2,\frac{(1-2q)^2 e^{-\beta}}{24\left(1+ 2e^{\beta}\sqrt{2\left(1-q\right)q\tanh\beta}\right)} \right\}.
\end{align}
then $\mc{L}^{(2)}\left(\p (\cdot),\RIPn\left(\p(\cdot)\right)\right)\leq \frac{\eta}{2}$ under the event $\Ecorrn\cap \Estrongn$.
\end{lemma}
\begin{proof} 
Recall that \begin{align}
\mc{L}^{(2)}\left(\p (\cdot),\RIPn\left(\p(\cdot)\right)\right)&=\frac{1}{2}\left| \prod_{e\in\text{path}_{\T}(w,\tilde{w}) }\frac{\nmu_e}{(1-2q)^2}-  \prod_{e\in\text{path}_{\TCLn}(w,\tilde{w}) }\frac{\nmu_e}{(1-2q)^2}  \right|\nonumber\\
&=\frac{1}{2}\left| \prod_{e\in\text{path}_{\T}(w,\tilde{w}) }\mu_e-  \prod_{e\in\text{path}_{\TCLn}(w,\tilde{w}) }\mu_e  \right|.
\end{align} 
We follow Bresler's and Karzand's technique ``Loss due to graph estimation''~\cite[Section 6.2]{bresler2020learning} and highlight the difference that appears in our setting.  For the noisy case/hidden model, the argument changes slightly in the following manner: 
\begin{align*}
2\mc{L}^{(2)}\left(\p (\cdot),\RIPn\left(\p(\cdot)\right)\right)&\leq|\mu_f \mu_A \mu_{\tilde{A}} \mu_B \mu_{\tilde{B}}|| \mu^2_C \mu^2_{\tilde{C}}-1|+|\mu_f|\lp \Delta (k)+\Delta (\tilde{k})+\Delta (\tilde{k})\Delta (k) \rp\\
&=\left|\frac{\nmu_f}{(1-2q)^2} \mu_A \mu_{\tilde{A}} \mu_B \mu_{\tilde{B}}\right|\left| \mu^2_C \mu^2_{\tilde{C}}-1\right|
\\&\qquad\qquad\qquad\qquad\qquad +\left|\frac{\nmu_f}{(1-2q)^2}\right|\lp \Delta (k)+\Delta (\tilde{k})+\Delta (\tilde{k})\Delta (k) \rp\\
&\leq 8\frac{\neps}{(1-2q)^2}+\frac{\ntau}{(1-2q)^2} (2\eta+\eta^2 )\numberthis\label{eq:lossduetograph1}\\&\leq\eta.\numberthis\label{eq:lossduetograph2}
\end{align*}
 \eqref{eq:lossduetograph1} holds since $|\nmu_f|-|\nmu_g|\leq 4\neps $, $|\mu_f|\lp 1-\mu^2_{C}\mu^2_{\tilde{C}}\rp\leq 2|\mu_f|-2|\mu_g|$, $|\nmu_f|\leq \ntau$, and \eqref{eq:lossduetograph2} holds for all the values of $\neps $ that satisfy \begin{align}\label{assumption_on_eps_22}
\neps\leq \min \left\{ \frac{\eta}{16}(1-2q)^2,\frac{(1-2q)^2 e^{-\beta}}{24\left(1+ 2e^{\beta}\sqrt{2\left(1-q\right)q\tanh\beta}\right)} \right\}.
\end{align} The latter provides the statement of the Lemma.%The modification above is necessary, since for our problem we should be able to work with and derive results based only the parameters $\nn$, $\neps$, $\ntau$, $\ngam$, instead of $n$, $\eps$, $\tau$, $\gamma$.  
\end{proof}

The next Theorem provides the sufficient number of samples for predictive learning that recovers exactly the noiseless setting for $q=0$. Note that the dependence on $\beta$ changes from $e^{2\beta}$ to $e^{4\beta}$ when the data are noisy. A key component of the bound is the following function \begin{align}
    \Gamma(\beta,q)\triangleq \lp \frac{1-(1-2q)^2}{1-(1-2q)^4\tanh^2 (\beta)}\rp^2,\quad  \beta>0 \text{ and } q\in [0,1/2).
\end{align} Note that $\Gamma(\beta,q)\in [0,1]$ for all $\beta>0$ and $q\in [0,1/2)$, and $\Gamma(\beta,0)=0$ for all $\beta>0$. Further we define \begin{align}
    B (\beta,q)\triangleq \max\left\{\frac{1}{K(\beta,q)},\left(1+ 2e^{\beta}\sqrt{2\left(1-q\right)q\tanh\beta}\right)^2\right\},\label{eq:B_def1}
\end{align} and the expression of $K(\beta,q)$ is given by \eqref{eq:K()_def}.

\begin{theorem}\label{Main_result_proof_Appendix}  
Fix $\delta\in (0,1)$. Choose $\eta>0$ (independent of $\delta$). If
\begin{align}
    n\geq& \max \Bigg\{\frac{512}{ \eta^2(1-2q)^4}, \frac{1152 e^{2\beta}B(\beta,q)}{(1-2q)^4}, \frac{48 e^{4\beta}}{\eta^2}\Gamma(\beta,q) \Bigg\}\log\lp\frac{6p^3}{\delta}\rp,\label{eq:main_result_continuous1}
\end{align}
 then 
 \begin{align}
	\P \lp \Lnorm \lp \p (\cdot),\RIPn (\hat{p}_{\dagger} ) \rp \leq \eta \rp \geq 1-\delta.
	\end{align} Additionally, as a consequence of \eqref{eq:main_result_continuous1} , if \begin{align}
    n\geq& \max \Bigg\{\frac{512}{ \eta^2(1-2q)^4}, \frac{1152 \lp1+3\sqrt{q}\rp^2e^{2\beta(1+\mathds{1}_{q\neq 0})}}{(1-2q)^4}, \frac{48 e^{4\beta}}{\eta^2}\mathds{1}_{q\neq 0}  \Bigg\}\log\lp\frac{6p^3}{\delta}\rp,
\end{align} 
then 
\begin{align}
	\P \lp \Lnorm \lp \p (\cdot),\RIPn (\hat{p}_{\dagger} ) \rp \leq \eta \rp \geq 1-\delta.
	\end{align}
\end{theorem}

\begin{proof}Recall that  
\begin{align}
\mc{L}^{(2)}\left(\p (\cdot),\RIPn\left(\p(\cdot)\right)\right)=\frac{1}{2}\left| \prod_{e\in\text{path}_{\T}(w,\tilde{w}) }\mu_e-  \prod_{e\in\text{path}_{\TCLn}(w,\tilde{w}) }\mu_e  \right|.
\end{align} 
We combine the triangle inequality \begin{align}
\mc{L}^{(2)}\left(\p(\cdot),\RIPn (\hat{\p}_{\dagger} )\right)\leq\mc{L}^{(2)}\left(\p (\cdot),\RIPn\left(\p(\cdot)\right)\right)+\mc{L}^{(2)}\left(\RIPn\left(\p(\cdot)\right),\RIPn (\hat{\p}_{\dagger} )\right),
\end{align} 
Lemma \ref{Lemma end_to_end_error}, and Lemma \ref{Lemma F5} to get that $\mc{L}^{(2)}(\p(\cdot),\RIPn (\hat{\p}_{\dagger} ))\leq \eta$ with probability at least $1-\delta$ if
\begin{align}
    \ngam\leq \frac{\eta}{3} \text{ and } \neps\leq \min \left\{ \frac{\eta}{16}(1-2q)^2,\frac{(1-2q)^2 e^{-\beta}}{24\left(1+ 2e^{\beta}\sqrt{2\left(1-q\right)q\tanh\beta}\right)} \right\}.\label{eq:assumptionsgammaeps}
\end{align} 
First, we find the necessary number of samples such that for $\ngam\leq \eta/3$ the probability of the complement of $\Ecascn$ is not greater than $\delta/3$. Recall that 
\begin{align}
    G&\triangleq\frac{3\lp 3e^{-1}\mathds{1}_{q\neq 0} +1\rp}{4(1-2q)^2},\\
    S&\triangleq 3-(1-2q)^2.
\end{align} Recall that \begin{align}
    \Gamma(\beta,q)\triangleq \lp \frac{1-(1-2q)^2}{1-(1-2q)^4\tanh^2 (\beta)}\rp^2,\quad  \beta>0 \text{ and } q\in [0,1/2).\label{eq:Gamma_def}
\end{align} Lemma \ref{Lemma Ecascn} gives that for any $\Delta>0$ and $\eta>\Delta$ if 
\begin{align}
   \hspace{-0.32cm} n\geq \max \left\{\frac{0.3^{-2}S^2 G^2}{ (\eta-\Delta)^2}, \frac{108 e^{2\beta}}{(1-2q)^4 K(\beta,q)}, \frac{3 e^{4\beta}}{\Delta^2}\Gamma (\beta,q) \right\}\log\lp\frac{6p^3}{\delta}\rp,\label{eq:boundonn1}
\end{align}
then the probability of the complement of $\Ecascn$ is upper bounded by $\delta/3$ and we write 
    \begin{align}
    \P \lp  \lp \Ecascn\rp ^c \rp \leq \frac{\delta}{3}.
    \end{align} 
Second, we find the necessary number of samples such that the complements of the events $\Estrongn$ and $\Ecorrn$ occur with probability not greater than $\delta/3$ each. In fact the upper bound on $\neps$ \eqref{eq:assumptionsgammaeps} and Lemma \ref{Lem_Ecorrn} gives that if 
\begin{align}
   n\geq \max\left\{ \frac{512}{ \eta^2 (1-2q)^4 },
   \frac{1152 e^{2\beta}}{(1-2q)^4} \left(1+ 2e^{\beta}\sqrt{2\left(1-q\right)q\tanh\beta}\right)^2 \right\}\log \lp \frac{6p^3}{\delta}\rp,\label{eq:boundonn2}
\end{align} 
then $\neps $ satisfies the inequality in \eqref{eq:assumptionsgammaeps} with probability at least $1-\delta/3$. Note that \eqref{eq:boundonn1} holds for any $\Delta\in (0,\eta)$ and we will choose $\Delta=\eta/4$. Under the choice $\Delta=\eta/4$ 
\begin{align}
    \frac{0.3^{-2}S^2 G^2}{ (\eta-\Delta)^2}=\frac{0.3^{-2}S^2 G^2}{ (\eta-\eta/4)^2}<\frac{512}{ \eta^2 },\quad \forall \eta>0, q\in[0,1/2).\label{eq:approx1}
\end{align} 
Recall that \begin{align}
    B (\beta,q)\triangleq \max\left\{\frac{1}{K(\beta,q)},\left(1+ 2e^{\beta}\sqrt{2\left(1-q\right)q\tanh\beta}\right)^2\right\}.\label{eq:B_def}
\end{align} 
Combining \eqref{eq:boundonn1}, \eqref{eq:boundonn2}, \eqref{eq:approx1} and \eqref{eq:B_def} yields 
\begin{align}
    n\geq& \max \Bigg\{\frac{512}{ \eta^2(1-2q)^4}, \frac{1152 e^{2\beta}B(\beta,q)}{(1-2q)^4}, \frac{48 e^{4\beta}}{\eta^2}\Gamma(\beta,q) \Bigg\}\log\lp\frac{6p^3}{\delta}\rp.\label{eq:main_result_continuous}
\end{align} The latter gives the sample complexity for accurate predictive learning, it reduces exactly to the noiseless setting of prior work by~\cite{bresler2020learning} and it is continuous because
\begin{align}
    \lim_{q\to 0^+}\Gamma (\beta,q)= \Gamma (\beta,0)=0,\quad \forall \beta>0
\end{align}
and
\begin{align}
    \lim_{q\to 0^+} K(\beta,q)=K(\beta,0)=1, \quad \forall \beta>0\\
    \lim_{q\to 0^+} \left(1+ 2e^{\beta}\sqrt{2\left(1-q\right)q\tanh\beta}\right)^2=1
\end{align} thus \begin{align}
    \lim_{q\to 0^+}B(\beta,q)= \Gamma (\beta,0)=1,\quad \forall \beta>0.
\end{align}
To derive a simplified version of \eqref{eq:main_result_continuous} note that \begin{align}
    \frac{1}{K(\beta,q)}\leq e^{2\beta\mathds{1}_{q\neq 0}}\label{eq:bound_on_K}
\end{align} by the definition \eqref{eq:K()_def} of $K(\beta,q)$ and \begin{align}
    \left(1+ 2e^{\beta}\sqrt{2\left(1-q\right)q\tanh\beta}\right)^2&\leq \left(e^{\beta\mathds{1}_{q\neq 0}}+ 2e^{\beta\mathds{1}_{q\neq 0}}\sqrt{2\left(1-q\right)q\tanh\beta}\right)^2\nonumber\\
    &\leq \lp1+3\sqrt{q}\rp^2e^{2\beta\mathds{1}_{q\neq 0}}.\label{eq:approx2}
\end{align} Then \eqref{eq:B_def}, \eqref{eq:bound_on_K} and \eqref{eq:approx2} give \begin{align}
    B(\beta,q)\leq \lp1+3\sqrt{q}\rp^2e^{2\beta\mathds{1}_{q\neq 0}} \label{eq:bound_on_Beta}
\end{align} and by the definition \eqref{eq:Gamma_def} $\Gamma (\beta,q)\in [0,1)$ and $\Gamma (\beta,0)=0$ for all $\beta>0$, thus\begin{align}
    \Gamma (\beta,q)\leq \mathds{1}_{q\neq 0}. \label{eq:bound_on_Gamma}
\end{align}
Finally, we combine \eqref{eq:main_result_continuous}, \eqref{eq:bound_on_Beta}, \eqref{eq:bound_on_Gamma}  to get
\begin{align}
    n\geq& \max \Bigg\{\frac{512}{ \eta^2(1-2q)^4}, \frac{1152 \lp1+3\sqrt{q}\rp^2e^{2\beta(1+\mathds{1}_{q\neq 0})}}{(1-2q)^4}, \frac{48 e^{4\beta}}{\eta^2}\mathds{1}_{q\neq 0}  \Bigg\}\log\lp\frac{6p^3}{\delta}\rp.
\end{align} This completes the proof. 
\end{proof}

\section{Theorem \ref{thm:SKLTheorem}: KL-Divergence Loss}

Assume the Ising model tree distributions $\text{P}_{\boldsymbol{\theta}}$ according
to a tree $\T_{\boldsymbol{\theta}}=\left(\mc{V},\mc{E}_{\boldsymbol{\theta}}\right)$
and the estimate $\text{P}_{\boldsymbol{\theta'}}$ according a tree
$\T_{\boldsymbol{\theta'}}=\left(\mc{V},\mc{E}_{\boldsymbol{\theta}'}\right)$
The goal is to upper bound the symmetric KL divergence 
\begin{align*}
\SKL\left(\boldsymbol{\theta}||\boldsymbol{\theta}'\right) & =\sum_{s,t\in\mc{E}} \left(\theta_{st}-\theta'_{st}\right)\left(\mu_{st}-\mu'_{st}\right)
\end{align*}
with high probability. 
Under the event $\Ecorr$ we can upper
bound the quantity $\left|\mu_{st}-\mu'_{st}\right|$ for all $(s,t)\in\mc{E}$
with high probability. 

%\begin{comment}
%\kn{The challenging part is to find a tighter upper
%bound than $\left|\theta_{st}-\theta'_{st}\right|\leq\left|\beta-\alpha\right|$
%$\forall s,t\in\mc{V}$. Notice that $\theta_{st}=\tanh^{-1}\mu_{st}$
%and $\theta'_{st}=\tanh^{-1}\mu'_{st}$. However, the function $\tanh^{-1}$
%is not Lipschitz and the upper bounds on the differences $\left|\mu_{st}-\mu'_{st}\right|$
%are not enough to guarantee an upper bound for $\left|\theta_{st}-\theta'_{st}\right|$.
%\\}
%\end{comment}

By using bounds $\left|\theta_{st}-\theta'_{st}\right|\leq 2\beta$
and $\left|\mu_{st}-\mu'_{st}\right|\leq\eps$ for all $(s,t)\in\mc{E}$
under the event $\Ecorr$, we have 
\begin{align*}
\SKL\left(\boldsymbol{\theta}||\boldsymbol{\theta}'\right) & =\left|\SKL\left(\boldsymbol{\theta}||\boldsymbol{\theta}'\right)\right|\\
 & =\left|\sum_{s,t\in\mc{E}}\left(\theta_{st}-\theta'_{st}\right)\left(\mu_{st}-\mu'_{st}\right)\right|\\
 & \leq\sum_{s,t\in\mc{E}}\left|\theta_{st}-\theta'_{st}\right|\left|\mu_{st}-\mu'_{st}\right|\\
 & \leq(p-1)\left|\beta-(-\beta)\right|\eps\\
 & \leq\eta_{S},\numberthis
\end{align*}
by assuming $\epsilon\leq \frac{\eta_{s}}{2\beta(p-1)}$. The sufficient number of samples satisfies the inequality
\begin{align}
n\geq 4\log\left(p^{2}/\delta\right)\frac{\beta^{2}(p-1)^2}{\eta_{s}^{2}}.
\end{align}
 Now assume that $\nn$ samples of $\bnX$ are given, by using the estimate $\text{P}_{\boldsymbol{\theta'}}=\RIPn (\hat{\p}_{\dagger} )$ defined in \eqref{eq:probability_estimator} under the event $\Ecorrn$ we have $\left| \mu_{st}-\frac{\nmue _{st}}{(1-2q)^2} \right|\leq \frac{\neps}{(1-2q)^2}$ from Lemma \ref{Lem_Ecorrn}. In the same way by assuming $\neps\leq \frac{\eta_{s}(1-2q)^2}{2\beta(p-1)}$, we get
\begin{align}
\nn\geq  4\log\left(p^{2}/\delta\right)\frac{\beta^{2}(p-1)^2}{(1-2q)^4\eta_{s}^{2}}.
\end{align}

%%%
%%% UPPER BOUNDS AND CONVERSE ARGUMENTS
%%%
\section{Theorem \ref{thm:necessary} and Theorem \ref{thm:fano's inequality theorem}: Proofs}

We combine Fano's inequality and a Strong Data Processing Inequality to prove the necessary number of samples in the hidden model setting, first for structure learning (Theorem \ref{thm:necessary}) and then for inference (Theorem \ref{thm:fano's inequality theorem}). We use the following variation of Fano's inequality. \begin{corollary} \label{fanostsybakov}\cite[Corollary 2.6]{tsybakov2009introduction}:
Assume that $\Theta$ is a family of $M+1$ distributions $\theta_0,\theta_1,\ldots,\theta_M$ such that $M\geq 2$. Let $P_{\theta_i}$ be the distribution of the variable $X$ under the model $\theta_i$, if \begin{align}
\frac{1}{M+1} \sum^{M}_{i=1} \KL\lp P_{\theta_i}||P_{\theta_0} \rp\leq \gamma\log M,\quad \text{for } \gamma\in(0,1) 
\end{align} then for the probability of error $p_{e}$ the following inequality holds: $p_{e}\geq \frac{\log(M+1)-\log(2)}{\log(M)}-\gamma$. \end{corollary}

The construction from the noiseless case, with 
Corollary \ref{fanostsybakov} and the Strong Data Processing Inequality for the BSC yield the bound of Theorem \ref{thm:necessary}. We start by presenting Bresler's and Karzand's construction, which gives a sufficiently tight upper bound on symmetric KL divergence.

\noindent\textbf{Proof of Theorem \ref{thm:necessary}:} 
Consider a family of $M+1$ different Ising model distributions $\{P_{\theta^{i}}:i\in\{0,\ldots,M\}\}$. This family of the structured distributions is chosen such that the structure recovery task (through The Chow-Liu algorithm) is sufficiently hard. First, we define $P_{\theta^0}$ to be an Ising model distribution with underlying structure a chain with $p$ nodes and parameters $\theta^{0}_{j,j+1}=\alpha$, when $j$ is odd and $\theta^{0}_{j,j+1}=\beta$ when $j$ is even. The rest of family is constructed as follows: the elements of each $\theta^{i}$, $i\in[M]$ are equal to the elements of $\theta^{0}$ apart from two elements $\theta^{i}_{i,i+1}=0$ and $\theta^{i}_{i,i+2}=\arctanh(\tanh(\alpha)\tanh(\beta))$, for each odd value of $j$. There are $(p+1)/2$ distinct distributions in the constructed family. Through the expression \eqref{eq:SKL}, we derive the following upper bound on the $\SKL (P_{\theta^{0}}||P_{\theta^{i}})$, for all $i\in[M]$,~\cite[Section 7.1]{bresler2020learning},
\begin{align}\label{eq:SKLupperbound}
&\SKL (P_{\theta^{0}}||P_{\theta^{i}})=\alpha \lp\tanh (\alpha) - \tanh (\alpha)\tanh^2 (\beta)  \rp\leq 4\alpha \tanh (\alpha) e^{-2\beta}.
\end{align}

 \noindent\textit{\underline{Strong Data Processing Inequality}:} For each distribution $P_{\theta^{i}}$ and $i\in\{0,\ldots,M\}$ we consider the distribution of the noisy variable in the hidden model $\nP_{\theta^{i}}\triangleq P_{\bnX|\bX}\circ P_{\theta^{i}}$. We would like to find an upper bound for the quantities $\SKL (\nP_{\theta^{0}}||\nP_{\theta^{i}})$, $i\in \{0,\ldots,M\}$. For that purpose, we a use a strong data processing inequality result for the BSC by~\citet{polyanskiy2017strong}. The input random variable $\bX$ is considered to have correlated binary elements, while the noise variables $N_i$ are i.i.d $\Rad (q)$. This scheme is equivalent to the hidden model that we consider in this paper. In fact we have the following bound \begin{align}\label{eq:boundetaSDPI}
\eta_{\text{KL}}\leq 1-(4q(1-q))^p,
\end{align} that is proved by Polyanski~\cite[``Evaluation for the BSC'', equation (39)]{polyanskiy2017strong}, where the quantity $\eta_{\text{KL}}$ is defined as \begin{align}\label{eq:SDPI_sup_sup}
\eta_{\text{KL}}\triangleq \sup_{Q} \sup_{P:0<\KL(P||Q)<\infty} \frac{  \KL\lp P_{\bnX|\bX}\circ P || P_{\bnX|\bX}\circ Q\rp }{\KL\lp P||Q \rp},
\end{align} $P_{\bnX|\bX}$ is the distribution of the BSC and $P,Q$ are any distributions of the input variable $\bX$.

Since the supremum in \eqref{eq:SDPI_sup_sup} is with respect to all possible distributions, it covers any pair of distributions in the desired family  $\{P_{\theta^{j}}:j\in\{0,\ldots,M\}\}$. Thus, for all $k,\ell\in\{0,1,\ldots,M\}$ and $k\neq\ell$, it is true that\begin{align}
\frac{\KL(\nP_{\theta^{k}}||\nP_{\theta^{\ell}})}{\KL(P_{\theta^{k}}||P_{\theta^{\ell}})}&\leq 1-(4q(1-q))^p,\end{align} which comes from \eqref{eq:boundetaSDPI},\eqref{eq:SDPI_sup_sup} and implies the following
\begin{align}
\SKL(\nP_{\theta^{k}}||\nP_{\theta^{\ell}})&\leq[1-(4q(1-q))^p]\SKL(P_{\theta^{k}}||P_{\theta^{\ell}}),\quad \forall k\neq\ell\in\{0,1,\ldots,M\}\label{eq:SKL_SDPI}.
\end{align} We combine \eqref{eq:SKLupperbound} and \eqref{eq:SKL_SDPI} to get \begin{align}\label{eq:SKLboundnoisy}
\SKL(\nP_{\theta^{k}}||\nP_{\theta^{i}})&\leq [1-(4q(1-q))^p]4\alpha\tanh(\alpha) e^{-2\beta}\leq [1-(4q(1-q))^p]4\alpha^2 e^{-2\beta}.
\end{align} Finally, from \eqref{eq:SKLboundnoisy} and Corollary \ref{fanostsybakov} we derive the first part of Theorem \ref{thm:necessary}.

 \noindent\textbf{Proof of Theorem \ref{thm:fano's inequality theorem}}:  Theorem \ref{thm:fano's inequality theorem} is the extended version of Theorem 3.4 by \citet{bresler2020learning} to the hidden model. Following a similar technique, we consider chain structured Ising models with parameters $\theta^{j}$ for $j\in[M]$ such that $\theta^{j}_{j,j+1}=\alpha$ and $\theta^{j}_{i,i+1}=\arctanh(\tanh(\alpha)+2\eta)$, for all $i\neq j$. Then\begin{align}
\Lnorm \lp P_{\theta^{j}}, P_{\theta^{j'}} \rp= \max_{s,t} \left| \E_{\theta^{j}}[X_s X_t] -\E_{\theta^{j'}}[X_s X_t] \right| \geq 2\eta 
\end{align} and \begin{align}\label{eq:SKLLnorm}
\SKL\lp P_{\theta^{j}}, P_{\theta^{j'}} \rp\leq 2\eta \left[\arctanh(\tanh(\alpha)+2\eta)-\alpha  \right]\leq 2\eta \frac{2\eta}{1-\left[ \tanh(\alpha) +2\eta\right]^2},
\end{align} where the last inequality is a consequence of Mean Value Theorem (see~\cite[Section 6.3]{bresler2020learning} for the original statement). We derive the bound of Theorem \ref{thm:fano's inequality theorem} by combining the strong data processing inequality \eqref{eq:SDPI_sup_sup} with \eqref{eq:boundetaSDPI}, \eqref{eq:SKLLnorm}, and Corollary \ref{fanostsybakov}.

\section{Supplementary Discussion}
In this section we provide supplementary material that supports the discussion in Sections \ref{Hidden_structure_estimation} and \ref{MLE}. First, we present one marginal case for which perfect denoising is possible before applying the Chow-Liu algorithm. Then we show a structure-preserving case.

\subsection{The Gap between the Upper and Lower Bounds}\label{GAP_discussion}
We continue by analyzing the gap that appears between the upper and lower bounds for an example where perfect denoising can be applied on a specific class of tree models in the high-dimensional regime. This shows why the effect of noise vanishes in Theorems 1.2 and 1.4 for $p\to \infty$. Further, while it seems counter-intuitive that when $p\to\infty$ the problem becomes easier, we show below one example that this is the case. Our lower bound is directly affected by marginal cases like this, for instance see Proposition 1.4.
   
The gap is introduced by the terms $(1-2q)^4$ and $1-(4q(1-2q))^p$ in the denominator of the lower and upper bounds respectively. Specifically, for $p\to\infty$ there exists a special case for which perfect denoising before running the Chow-Liu algorithm is possible, while in other cases that is not possible. Thus the minimax bound ought to be identical to noiseless case when $p\to \infty$ and $1-(4q(1-2q))^p \to 1$ in the large dimensional regime. We continue by providing the marginal case of a trivial tree structure and showing that perfect denoising is possible in this case before running the Chow-Liu algorithm.
    
First notice that if $p \to \infty$, then the sample size $n \to \infty$, even in the noiseless regime. Consider the case of $\E[X_i X_j]\to 1$ for all $(i,j)\in\mc{E}$. Because an infinite number of samples are available, we can estimate perfectly the correlations of the observables and we find $\hat{\E}[Y_i Y_j]=\E[Y_i Y_j] = (1-2q)^2$ for all $(i,j)\in\mc{V}$. The latter as information is sufficient to find that $\E[X_i X_j]\to 1$ for all $(i,j)\in\mc{E}$. The hidden layer $\bX$ take two values, $(+1,+1,\ldots)\triangleq +1^p$ ($p$ values $+1$) or $(-1.-1,\ldots)\triangleq -1^p$ ($p$ values $-1$), because $\E[X_i X_j]\to 1$ for all $(i,j)\in\mc{V}$ and the later allows us to denoise each sample. Define as $\hd (\bX,\bnX)$ the Hamming distance between $\bX$ and $\bnX$. At this point we can perform perfect denoising for each sample $\bnx_s$ of infinite length $p$ and find the hidden sample $\bx_s$ with probability $1$ because \begin{align*}
        \P \lp \bX =\bx_s | \bnX =\bnx_s  \rp &= \frac{\P \lp \bnX=\bnx_s | \bX_s =\bx_s  \rp \P (\bX_s =\bx_s)}{\sum_{\bx} \P \lp \bnX =\bnx_s | \bX =\bx  \rp \P (\bX =\bx)}\\
        &= \frac{ q^{\hd (\bx_s,\bnx_s)} (1-q)^{p-\hd (\bx_s,\bnx_s)}  }{q^{\hd (\bx_s,\bnx_s)} (1-q)^{p-\hd (\bx_s,\bnx_s)}+q^{\hd (-\bx_s,\bnx_s)} (1-q)^{p-\hd (-\bx_s,\bnx_s)}}\numberthis \label{eq:cond_prob}
    \end{align*} and the last holds for both of the cases $\bx_s=+1^p$ or $\bx_s=-1^p$ because of symmetry. Further for any observation $\bnx_s$ for any $q\in (0,1/2)$ we have \begin{align}\label{eq:hamming_lim}
       % \lim_{p\to \infty  } \hd (\bx_s,\bnx_s) \overset{a.s.}{\to} qp\quad \text{  and  }\quad 
       \lim_{p\to \infty  } \frac{\hd (-\bx_s,\bnx_s)-\hd (\bx_s,\bnx_s)}{p} \overset{a.s.}{=}1-q -q=1-2q.
    \end{align} We combine \eqref{eq:cond_prob} and \eqref{eq:hamming_lim} to find \begin{align}
         \lim_{p\to \infty  } \P \lp \bX =\bx_s | \bnX =\bnx_s  \rp= \lim_{p\to \infty  } \frac{1}{1+\lp\frac{q}{1-q}\rp^{\frac{\hd (-\bx_s,\bnx_s)-\hd (\bx_s,\bnx_s)}{p}p}} =1
    \end{align} 
    and 
    \begin{align}
        \lim_{p\to \infty  } \P \lp \bX =-\bx_s | \bnX =\bnx_s  \rp= \lim_{p\to \infty  } \frac{1}{1+\lp\frac{1-q}{q}\rp^{\frac{\hd (-\bx_s,\bnx_s)-\hd (\bx_s,\bnx_s)}{p}p}} =0.
    \end{align} As a consequence there exists one case for which perfect denoising is possible before running the Chow-Liu algorithm. Because we want Theorems 1.2 and 1.4 to reduce to the noiseless case for $p\to \infty$,  the above best case scenario must be covered. However, perfect denoising is not possible in general (for instance $\E [X_i X_j]<1$ and finite $p$.).

\subsection{A Structure-Preserving Case}
 Lemma \ref{G_graph} considers a special case of tree structures for the hidden variables, the set of edges is a set with disconnected edges, no edge is connected to any other. Then we show that the same structure is preserved for the observable variables. 
\begin{lemma}\label{G_graph}
Let $\F=(\mc{V},\mc{E})$ be a forest with $|\mc{V}|=p$ and $|\mc{E}|=p/2\in \mbb{N}$ such that no edge is connected to any other edge. Assume that $X_i \in \{-1,+1\}$ and $\E\left[ X_{i}\right]=0$ for all $i\in[1,\ldots,p]$. If $\bnX$ is the output of the BSC channel (in the hidden model) with distribution  $\np(\bnx)$, then $\np(\bnx)$ also factorizes with respect to $\F$.
\end{lemma}
\begin{proof}
The pair variables $(Y_i,Y_j)$ for $(i,j)\in\mc{E}$ are independent because of the disconnected edges of the hidden layer. The latter directly gives the factorization as\begin{align}\label{eq:str_pres_distr}
   \np(\bnx)= \prod_{(i,j)\in \mc{E}} \np(y_i,y_j) = \prod_{i\in V} \np\left(y_{i}\right) 
	\prod_{ (i,j) \in \mc{E}} \frac{\np(y_{i},y_{j})}{\np(y_{i}) \np(y_{j})},
\end{align} because $|\mc{V}|=p$, $|\mc{E}|=p/2$ and the marginal distributions are uniform.
\end{proof}

\acks{This work was supported in part by DARPA and SSC Pacific under contract N66001-15-C-4070 and the United States National Science Foundation under award CCF-1453432, and the United States National Institutes of Health under award 1R01DA040487.}

\bibliography{Structure_learning}
\end{document}